\documentclass[letterpaper]{article}
\usepackage{uai2020}
\usepackage[margin=1in]{geometry}

\usepackage{wrapfig}

\usepackage{url}            % simple URL typesetting
\usepackage{booktabs}       % professional-quality tables
\usepackage{amsfonts}       % blackboard math symbols
\usepackage{multirow}
\usepackage{rotating}
\usepackage{amsmath}
\usepackage{amssymb}
\usepackage{amsthm}
\usepackage{enumitem}
\usepackage{tikz}
\usepackage{pgfplots}
\usetikzlibrary{pgfplots.statistics}
\usepackage{microtype}
\usepackage{graphicx}
\usepackage{subfigure}
\usepackage{booktabs} % for professional tables

\usepackage{mathptmx}
\usepackage{color}
  \usepackage{bbm}
  \usepackage{multirow}
  \usepackage{mathtools}
  \usepackage{soul}
    \usepackage{booktabs}
\usepackage[square,sort,comma,numbers]{natbib}
        \usepackage{subfigure}

\def\BdCor{\text{BdCor}}
\def\dCor{\text{dCor}}

\usepackage{mathrsfs}

\theoremstyle{plain}  % gives non italic without it, Latex gives the default which is \theoremstyle{plain}
\newtheorem{definition}{Definition}
\newtheorem{lemma}{Lemma}
\newtheorem{proposition}{Proposition}
\newtheorem{corollary}{Corollary}

\newtheorem{theorem}{Theorem}
\newcommand{\M}{\mathcal{M}}
\newcommand{\pspace}{\mathcal{Z}}

\newcommand{\hilb}{\mathcal{H}}

\makeatletter
\newcommand*{\indep}{%
  \mathbin{%
    \mathpalette{\@indep}{}%
  }%
}
\newcommand*{\nindep}{%
  \mathbin{%                   % The final symbol is a binary math operator
    \mathpalette{\@indep}{\not}% \mathpalette helps for the adaptation
  }%
}
\newcommand*{\@indep}[2]{%
  \sbox0{$#1\perp\m@th$}%        box 0 contains \perp symbol
  \sbox2{$#1=$}%                 box 2 for the height of =
  \sbox4{$#1\vcenter{}$}%        box 4 for the height of the math axis
  \rlap{\copy0}%                 first \perp
  \dimen@=\dimexpr\ht2-\ht4-.2pt\relax
  \kern\dimen@
  {#2}%
  \kern\dimen@
  \copy0 %                       second \perp
} 
\makeatother
\usepackage{hyperref}

\usepackage[font=small]{caption}
\usepackage{times}

\title{Bayesian Kernelised Test of (In)dependence with Mixed-type Variables}

\author{ {\bf Alessio Benavoli} \\
School of Computer Science and Statistics\\
Trinity College Dublin\\
Ireland\\
\And
{\bf Cassio de Campos}  \\
Mathematics and Computer Science      \\
Eindhoven University of Technology \\
Netherlands\\
}
\pgfplotsset{compat=1.16}

\begin{document}

\maketitle

\begin{abstract}
  A fundamental task in AI is to assess (in)dependence between mixed-type variables (text, image, sound). We propose a Bayes\-ian kernelised correlation test of (in)dependence using a Dirichlet process model. The new measure of (in)dependence allows us to answer some fundamental questions: Based on data, are (mixed-type) variables independent? How likely is dependence/independence to hold? How high is the probability that two mixed-type variables are more than just weakly dependent? We theoretically show the properties of the approach, as well as algorithms for fast computation with it. We empirically demonstrate the effectiveness of the proposed method by analysing its performance and by comparing it with other frequentist and Bayesian approaches on a range of datasets and tasks with mixed-type variables.
   
\end{abstract}

\section{Introduction}\label{sec:intro}
Most traditional data analysis approaches are hindered by handling either continuous or categorical variables but not both.
The digital era has led to a rapid increase in data diversity and volume. This requires techniques that can effectively deal with mixed-type data
(images, text, sound), and that are able to extract useful information from data and to communicate meaningful insights.
% 
% in domains where explanability is essential (medicine, manufacturing).
%In this paper, w
We focus on a particular aspect of data analysis through the following Question:
%that focuses on the following question
%\begin{description}
% \item[Q:] 
\textit{Given mixed-type variables, are they dependent or are they independent?}
% \end{description}
%is to assess dependence/independence
%  between mixed-type variables.
  
Distance correlation ($\dCor$) \citep{szekely2009brownian,szekely2007measuring} is a measure of dependence that generalises Pearson correlation defined for paired variables of arbitrary (and not necessarily equal) type. $\dCor$ takes values in $[ 0 , 1 ]$ and equals zero if and only if independence holds. It detects both linear and nonlinear associations. Moreover, it has been shown~\citep{sejdinovic2013equivalence} that its statistics can be defined via the \textit{squared Hilbert-Schmidt norm of the cross-covariance operator} (HSIC) of the kernel embedding of the distribution  into \emph{Reproducing Kernel Hilbert Spaces}
(RKHSs)~\citep{doran2014permutation,eric2008testing,fukumizu2008kernel,gretton2012kernel,gretton2008kernel}. Therefore, $\dCor$ can be easily generalised to any type of data by using kernel embedding. Using $\dCor$ statistics, a Null-Hypothesis Significance Test (NHST) of independence can be derived.
$\dCor$ is potentially a very powerful and interpretative tool to answer our Question but it has two major drawbacks: (i) It has a bias towards dependence that increases with the dimension of the variables being compared (to remove this bias, only an ad-hoc solution has been proposed~\citep{SZEKELY2013193}); % introducing a modified  statistics.
(ii) The p-value-based NHST derived from $\dCor$ cannot really answer our Question about independence.
In NHST, dependence is ``declared'' whenever the p-value is below a certain significance threshold. % (often $0.05$). 
%Such p-value is commonly interpreted as meaning that observed data only had a 5\% chance of occurring if the null hypothesis of independence were true but often also as meaning a 95\% chance of occurring if the alternative hypothesis of dependence were true.
%This interpretation is \textbf{wrong}: 
P-value is the probability that any $\dCor$ generated from the null hypothesis, according to the intended sampling process, has magnitude greater than or equal to that of the observed $\dCor_{obs}$, that is, $\text{p-value}=p(\text{any } \dCor_{null}  \geq   \dCor_{obs})$.
Since a p-value is evaluated only under the null hypothesis, it cannot say anything directly about the comparison between the null and alternative hypotheses. Using a NHST we can \textbf{never declare that the two variables are independent}. Thus, such test is not able to distinguish between statistical and practical significance (that is, the difference between weak and strong dependence).

\begin{figure*}
\centering
  \begin{tabular}{c  c  c }
    \includegraphics[width=.16\linewidth,trim={1cm 0.4cm 1cm 0},clip]{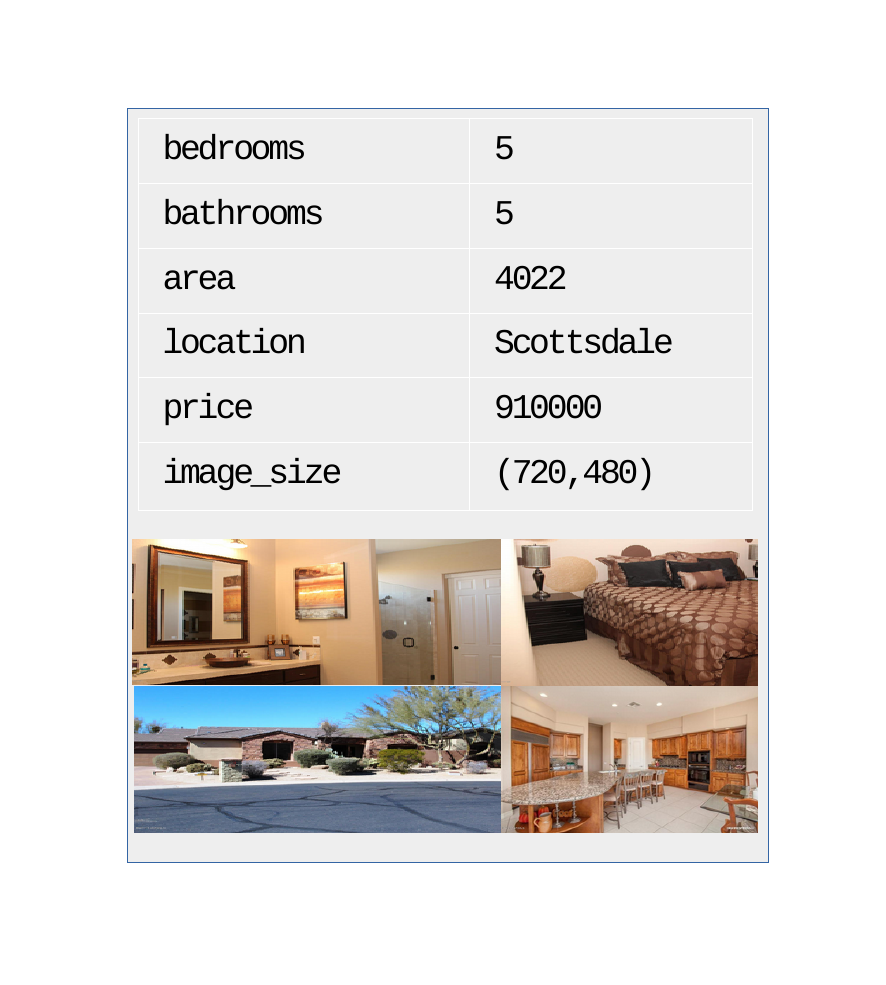} &
    \includegraphics[width=.26\linewidth,trim={2cm 0.4cm 2cm 0},clip]{{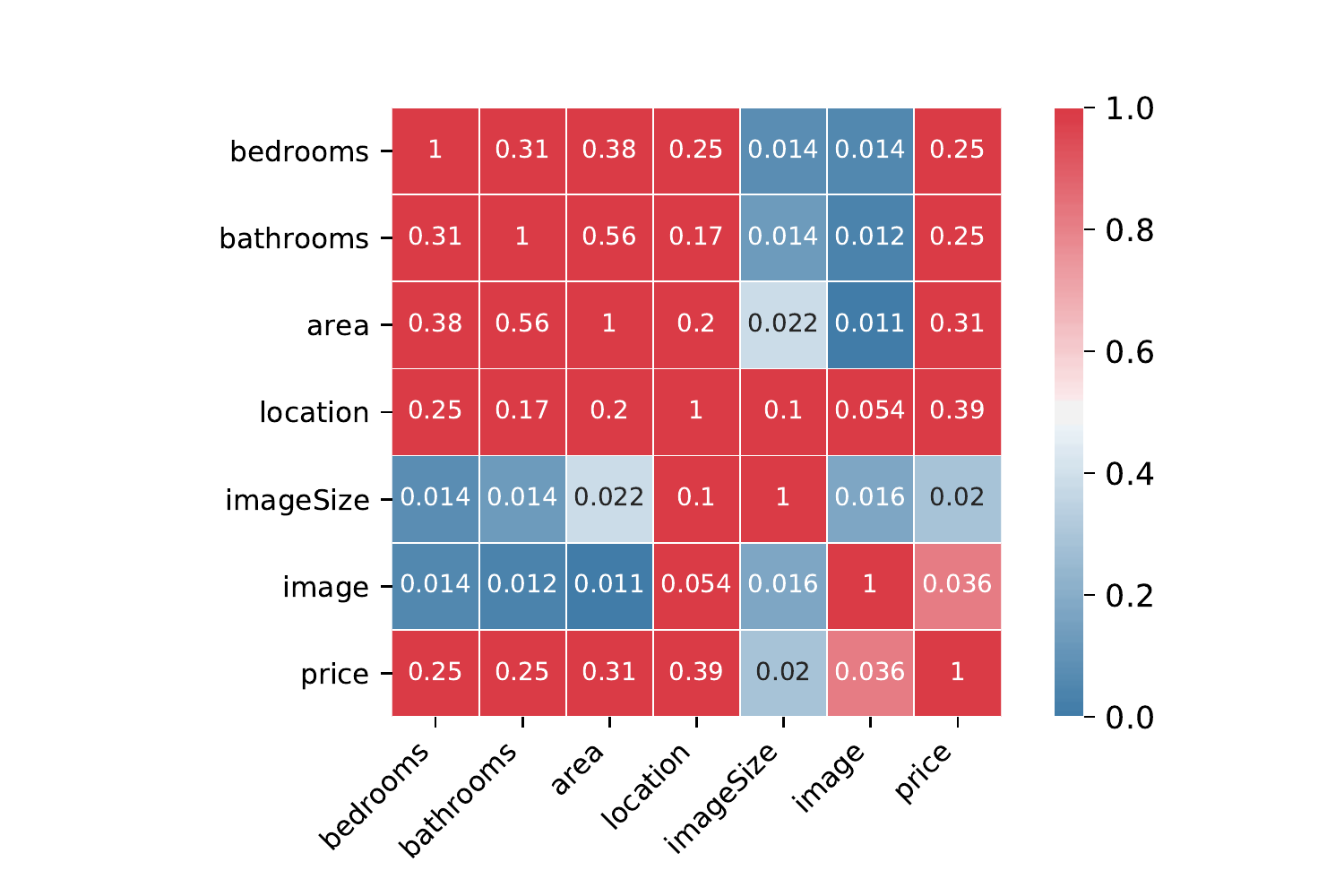}} &
    \includegraphics[width=.50\linewidth,trim={3cm 0.0cm 2cm 0},clip]{{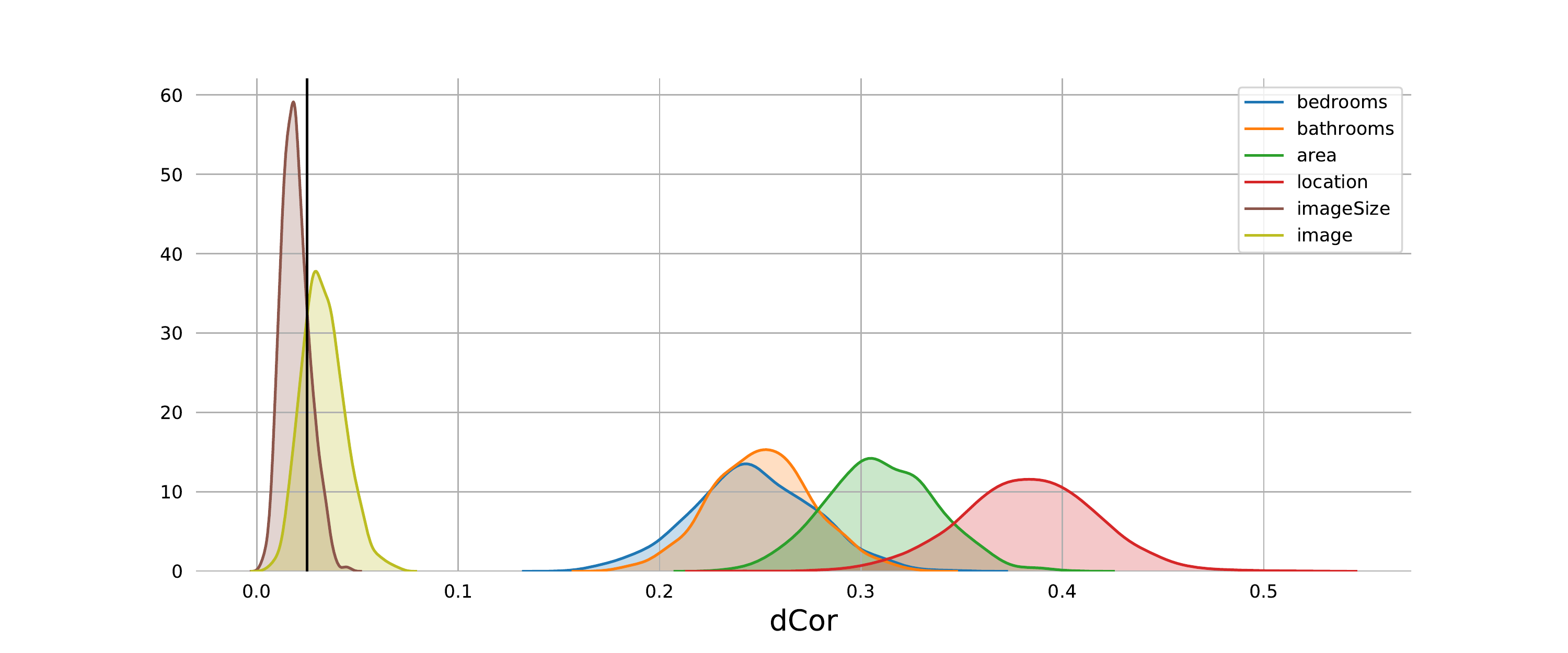}}\\
    \small (a) & \small (b) & \small (c)\\
  \end{tabular}
\caption{BKR applied to the House dataset. (a) An example of entry; (b) heatmap of $\BdCor$ among attributes; (c) Posterior of attributes' relation with price (ROPI is the region at the left of the black vertical line).}
\label{Fig:1}
\end{figure*}

Questions such as those in the abstract are questions about posterior probabilities, which can be naturally provided by the Bayesian methods. We propose a Bayesian Kernelised Correlation (BKR) test of (in)dependence based on $\dCor$. First, we
 derive a nonparametric posterior distribution over HSIC by employing a Dirichlet process (DP) prior, and then we derive $\BdCor$, the Bayesian $\dCor$. Second, we propose a new approach to remove the bias of $\BdCor$ consisting on estimating its distribution under independence by means of an \textit{invariant DP}. In particular, we consider invariance to exchangeability. Third, to make the approach scalable to high dimensions, 
%we combine it with  Gaussian Process Latent Variable Model (Kernel PCA) for dimensionality reduction and 
we provide a Bayesian extension of the low-rank kernel approximation developed for HSIC~\cite{zhang2018large}. 

In hypothesis testing, we often need to perform multiple comparisons between variables. The NHST analysis controls the family-wise Type I error (FWER), namely the probability of finding 
at least one Type I error (the error we incur by rejecting the null hypothesis when it is true) among the null hypotheses which are rejected when performing the multiple comparisons.
Bonferroni-like corrections for multiple comparisons can be used to control FWER~\cite{demvsar2006statistical}.
Such corrections simplistically treat the multiple comparisons as independent from each other. However, when comparing variables $\{A,B,C\}$, the outcome of the comparisons (A,B), (A,C), (B,C) are usually \textit{not} independent. 
Our contribution in this regard is to extend the approach proposed in~\cite{benavoli2015b} to develop a \textit{joint} procedure for the analysis of the multiple comparisons which accounts for their dependences. We analyse the posterior probability computed through DP, identifying statements of \textit{joint} association which have high posterior probability. This new test can answer our Question in an interpretative manner, and can solve the several drawbacks of NHST. It allows us to quantify the strength of a relationship and so to introduce the concept of \textit{Region Of Practical Independence} (ROPI).
ROPI indicates a range of $\BdCor$ values that are considered to be practically equivalent to the null value representing independence, and can be designed with a particular application in mind. This gives us the possibility of statistically distinguishing between weak and strong dependence, which is important for greater interpretability. 

We compare our approach with other methods on 
both synthetic and real data. For instance, Figure~\ref{Fig:1} illustrates the application of BKR to the \textit{Houses} dataset~\cite{ahmed2016house} that combines visual (4 images tiled together), numeric (number of bedrooms and bathrooms, area) and textual (location) attributes to be used for housing price estimation.
As further attribute, we use the size (in pixels) of the house frontal image as originally uploaded by the real estate agent. An example of one of the 535 entries in this dataset is showed in Figure~\ref{Fig:1}(a). Figure~\ref{Fig:1}(b)
reports a summary of the results for all pairwise (potentially mixed-type) comparisons performed using BKR (further details are provided in the Supp. Material); each number in a heatmap cell is the  posterior mean of $\BdCor$, and the colour of the cell is the posterior probability of practical dependence $p(\BdCor>0.025|Data)$.
A dark red means high probability of dependence, while a dark blue means high probability of independence (that is
$p(\BdCor<0.025|Data)\gg 0.5$). Figure \ref{Fig:1}(c) shows the posterior distribution over $\BdCor$ for all the pairwise comparisons of price vs.\ each of the other 6 variables. We immediately see that number of bedrooms and bathrooms, area and location have strong association with price (the mass of the posterior is concentrated around $\BdCor=0.3$).
The association between image and price is weaker, but the mass of the posterior is significantly outside ROPI, meaning that we expect the image to have some predicting power for Price (using both images and the above 4 variables yields a slightly better estimation accuracy compared to not using the images~\cite[Fig.7]{ahmed2016house}). As expected, Imagesize is instead practically independent to Price (mass largely inside ROPI) and so we could ignore it. On the other hand, the NHST based on HSIC~\citep{gretton2008kernel} reports p-value 0.0008 for Imagesize vs.\ Price. This shows that NHSTs may detect many spurious relationships due to their inability of assessing independence. BKR can explain this weak dependence from the heatmap: image size and location are dependent, which is understandable since real estate agents work by zone.

%\section{Methods}
\section{Kernel embeddings of measures}

Let $\pspace$ be a Borel measurable space. By  $\M^+ (\pspace)$ we denote the set of all probability measures  (Borel)  on $\pspace$.

\begin{definition} \label{def: Reproducing}
Let $\hilb$ be a real Hilbert space (HS) on $\pspace$. A function $k:\pspace \times \pspace \rightarrow \mathbb{R}$ is called a reproducing kernel of $\hilb$ if: 
\begin{enumerate}
\item $\forall z \in \pspace, k(\cdot,z) \in \hilb$;
\item $\forall z \in \pspace, \forall f \in \hilb, \langle f,k(\cdot,z) \rangle = f(z)$.
\end{enumerate}
If $\hilb$ has a reproducing kernel, then it is called a Reproducing Kernel HS~\citep{aronszajn1950theory}. 
\end{definition}
Therefore, $k(\cdot,z)$ is the evaluation map in $\hilb$. For any $x,y \in \pspace$, 
%\begin{equation}
%\label{eq:feature_map}
$k(x,y) = \langle k(\cdot,x), k(\cdot,y) \rangle$
%\end{equation}
is a symmetric positive definite function. The Moore-Aronszajn theorem~\citep{aronszajn1950theory} states that for any symmetric positive
definite function $k:\pspace \times \pspace \rightarrow \mathbb{R}$, there exists a unique HS of functions $\hilb$ defined on $\pspace$
such that $k$ is the reproducing kernel of $\hilb$. Since $k$ defines uniquely the RKHS $\hilb$, the latter is %usually
denoted as $\hilb_k$.
For $x,y \in \mathbb{R}^p$, an example of reproducing kernel is the square exponential kernel $k(x,y) = \exp(-\tfrac{\| x- y\|^2}{2\ell^2})$,
$\ell > 0$.

In non-parametric testing, we may represent probability measures $\nu\in\M^+(\pspace)$ with
elements of a RKHS~\citep{berlinet2011reproducing,smola2007hilbert}.
\begin{definition}
  Let $k$ be a kernel on $\pspace$, and $\nu \in \M(\pspace)$. The embedding of measure $\nu$ into RKHS $\hilb_k$ is
  $\mu_k (\nu) \in \hilb_k$ such that $\int f(z) d\nu(z) = \ \langle f,\mu_k(\nu) \rangle _{\hilb_k},~ \forall f \in \hilb_k.$\footnote{We assume 
    that the integral of any RKHS function $f$ under measure $\nu$ can be computed as the inner product between $f$ and the kernel embedding
    $\mu_k (\nu)$ in  $\hilb_k$. As an alternative, the kernel embedding may be defined via the use of the Bochner integral
    $\mu_k (\nu) = \int k(\cdot,z) d\nu(z)$.}
\end{definition}
Embeddings allow the definition of distance measures between probability distributions. 
%Hypothesis tests can then be defined by using these  distance measures  \citep{smola2007hilbert borgwardt2006integrating}. 
%In particular, we use two notions
%of distance (see for instance~\citep{sejdinovic2013equivalence}): %removed citation smola2007hilbert borgwardt2006integrating
%Maximum Mean Discrepancy (MMD) and Hilbert-Schmidt Independence Criterion (HSIC).
%MMD
% \begin{definition}
%  Let $k$ be a kernel on $\pspace$. The squared distance between the kernel embeddings of two probability measures $P$ and $Q$ in the RKHS, defined as $\text{MMD}_k(P,Q)=\|\mu_k(P)-\mu_k(Q)\|_{\mathcal H_k}^2$, is called Maximum Mean Discrepancy between $P$ and $Q$ with respect to $k$.
% \end{definition}
%
% When the corresponding kernels are \emph{characteristic}, it has be shown~\citep{fukumizu2008kernel} that the embedding is injective and
% MMD is a metric on probability measures. Estimators of MMD can then be used as statistics in non-parametric two-sample
% testing~\citep{gretton2012kernel}, that is, testing if two samples are drawn from the same probability distribution. 
For any kernels $k_{\mathcal X}$ and $k_{\mathcal Y}$ on the  domains $\mathcal X$ and $\mathcal Y$, we have that $k=k_{\mathcal X}\otimes k_{\mathcal Y}$ given by 
%\begin{equation}
$k\left(\left(x,y\right),\left(x',y'\right)\right)=k_{\mathcal X}(x,x')k_{\mathcal Y}(y,y')$
%\end{equation}
is a valid kernel on the product domain $\mathcal X \times \mathcal Y$.
The RKHS of $k=k_{\mathcal X}\otimes k_{\mathcal Y}$ is
isometric to $\mathcal H_{k_\mathcal X}\otimes \mathcal H_{k_\mathcal
  Y}$, which can be viewed as the space of Hilbert-Schmidt operators between $\mathcal H_{k_\mathcal Y}$ and $\mathcal H_{k_\mathcal X}$. % (Lemma 4.6 of~\citep{Steinwart08}).
This allows us to define a RKHS-based measure of dependence between variables $X$ and $Y$. 

\begin{definition}
 Let $X$ and $Y$ be random variables on domains $\mathcal X$ and
 $\mathcal Y$ (non-empty topological spaces). Let $k_{\mathcal X}$ and
 $k_{\mathcal Y}$ be kernels on $\mathcal X$ and $\mathcal Y$,
 respectively. The Hilbert-Schmidt Independence Criterion (HSIC) $\Xi_{k_\mathcal X,k_\mathcal Y}(X,Y)$ of $X$
 and $Y$ is the maximum mean discrepancy between the joint measure $P_{XY}$ and the product of marginals $P_{X}P_{Y}$, computed with the  kernel $k=k_{\mathcal X}\otimes k_{\mathcal Y}$:
 \begin{equation}
 \label{eq:hsic_def}
 \begin{aligned}
  &\Xi_{k_\mathcal X,k_\mathcal Y}(X,Y)=\Big\| E_{XY}[k_\mathcal{X}(.,X) \otimes k_\mathcal{Y}(.,Y)] \\
  &- E_X [k_\mathcal{X}(.,X)] \otimes E_Y [k_\mathcal{Y}(.,Y)]
  \Big\|^2_{\hilb_{k_\mathcal{X}\otimes k_\mathcal{Y}}}\, .
 \end{aligned}
 \end{equation}
\end{definition}
The argument of the norm in \eqref{eq:hsic_def} can be identified as a cross-covariance operator $\mathcal{C}_{XY}$ and, therefore, HSIC is  the squared Hilbert-Schmidt norm, $\| \mathcal{C}_{XY}\|_{HS}^2$, of $\mathcal{C}_{XY}$.
This norm can be normalised 
$$\dCor=\frac{\| \mathcal{C}_{XY}\|_{HS}^2}{\| \mathcal{C}_{XX}\|_{HS}\| \mathcal{C}_{YY}\|_{HS}}$$
to obtain the \textit{distance correlation} \citep{sejdinovic2013equivalence,szekely2009brownian,szekely2007measuring}, that can be thought as a kernelised version of the correlation coefficient (its range is $[0,1]$, where $0$ means independence).
% 
%  HSIC is well defined whenever $P_X \in \M^1_{k_{\mathcal X}}(\mathcal{X})$ and $P_Y \in \M^1_{k_{\mathcal Y}}(\mathcal{Y})$ as this implies $P_{XY} \in \M^{1/2}_{k_{\mathcal X}\otimes k_{\mathcal Y}}(\mathcal{X} \times \mathcal{Y})$ \cite{SejSriGreFuku13}. The name of HSIC comes from the operator view of the RKHS $\hilb_{k_\mathcal{X}\otimes k_\mathcal{Y}}$. Namely, the difference between embeddings $\mathbb{E}_{XY}[k_\mathcal{X}(.,X) \otimes k_\mathcal{Y}(.,Y)] - \mathbb{E}_X k_\mathcal{X}(.,X) \otimes \mathbb{E}_Y k_\mathcal{Y}(.,Y)$ can be identified with the cross-covariance operator $C_{XY}:\mathcal H_{k_\mathcal Y}\to\mathcal H_{k_\mathcal X}$ for which $\langle f,C_{XY}g\rangle_{\mathcal H_{k_\mathcal X}}=\text{Cov}\left[f(X)g(Y)\right]$, $\forall f\in\mathcal H_{k_\mathcal X},g\in\mathcal H_{k_\mathcal Y}$ \cite{gretbousmol2005,greker08}. HSIC is then simply the squared Hilbert-Schmidt norm $\Vert C_{XY}\Vert_{HS}^2$ of this operator, while distance correlation (dCor) of \cite{Cordis2007,SR2009} can be cast as $\Vert C_{XY}\Vert_{HS}^2 / \Vert C_{XX}\Vert_{HS}\Vert C_{YY}\Vert_{HS}$ \cite[Appendix A]{SejSriGreFuku13}. In the sequel, we will suppress dependence on kernels $k_\mathcal X$ and $k_\mathcal Y$ in notation $\Xi_{k_\mathcal X,k_\mathcal Y}(X,Y)$ where there is no ambiguity. \\

% \section{Kernel independence test}
The HSIC  can be decomposed as follows \citep{gretton2008kernel}.
\begin{proposition}
\label{prop:HSIC_as_expectations}
The HSIC of $X$ and $Y$ can be written as:
\begin{align}
\label{eq:HSIC_as_expectations}
\text{HSIC}(X,Y)&= E_{XX'YY'} \big(k_\mathcal{X}(X,X') k_\mathcal{Y}(Y,Y')\big) \\\nonumber
& + E_{XX'}\big(k_\mathcal{X}(X,X')\big) E_{YY'} \big(k_\mathcal{Y}(Y,Y')\big)\\\nonumber
& - 2 E_{XY}\big(E_{X'}( k_\mathcal{X}(X,X')) E_{Y'}(k_\mathcal{Y}(Y,Y'))\big). 
\end{align}
\end{proposition}
Estimators of HSIC have been used as statistics in non-parametric independence testing~\citep{gretton2008kernel}.
An empirical  estimate of the HSIC statistics from i.i.d.\ samples  $(x_1,y_1),\dots,(x_n,y_n)$ on $\mathcal{X}\times\mathcal{Y}$ can be obtained by replacing expectation operators with the sample mean \cite{gretton2008kernel}. 
% is given by
% \begin{align}
% \nonumber
% \text{HSIC}_{obs}(X,Y)&= \frac{1}{n^2} \sum_{i,j=1}^n k_\mathcal{X}(X_i,X_j)k_\mathcal{Y}(Y_i,Y_j)\\
% \nonumber
% &+\frac{1}{n^4} \sum_{i,j,q,r=1}^n k_\mathcal{X}(X_i,X_j)k_\mathcal{Y}(Y_q,Y_r)\\
% \label{eq:HSIC_as_expectationsemp}
% &-\frac{2}{n^3} \sum_{i,j,q=1}^n k_\mathcal{X}(X_i,X_j)k_\mathcal{Y}(Y_i,Y_q).
% \end{align}
Hence, \cite{gretton2008kernel} proposed a non-parametric NHST to determine if two variables $X$ and $Y$ are dependent, that is, $X \nindep Y$ ($H_1$), or independent,
that is, $X \indep Y$ ($H_0$).
$H_0$ is the so-called  null hypothesis and $H_1$ is the alternative hypothesis. This test can be carried out by computing the probability that any HSIC generated from the null hypothesis has magnitude greater than or equal to that of the observed $\text{HSIC}_{obs}$, that is, $\text{p-value}=p(\text{any } \text{HSIC}_{null}  \geq \text{HSIC}_{obs})$.
A very similar approach follows to perform NHST using $\dCor$.
Note that the space of all possible $\text{HSIC}_{null}$ that might have been observed depends on how the data were intended
to be sampled.\footnote{In case the data were intended to be collected until
a threshold sample size was achieved, the space of all
possible $\text{HSIC}_{null}$ is the set of all HSIC values with that exact
sample size. This is the conventional assumption and also the
assumption used in~\citep{gretton2008kernel} to compute p-values,
though seldom made explicit.}
If the intention is to collect data for a certain duration of time or if the intention of the practitioner is to perform multiple tests, a different p-value must be computed. For multiple comparisons, this comes from Bonferroni-like corrections to control the FWER.
%There are many different intentions for generating the space of possible $\text{HSIC}_{null}$ values, and hence many different p-values for a single set of data.
%
Moreover, such kernel NHST independence test does not provide any measure of evidence for the null hypothesis.
% Decisions are simply made by setting
% the significance level to $0.01$ or $0.05$, without considering the probability of Type II errors.
%Summarizing,
While elegant, they are unfortunately affected by the drawbacks which characterise all NHSTs.
%To overcome such issues, we address the problem from a Bayesian perspective.

\section{Dirichlet Process}
Let us consider again  $\M^+(\pspace)$, the space of probability measures on  $(\pspace,\mathcal{B}_{\pspace})$, equipped with the weak topology and the
corresponding Borel $\sigma$-field $\mathcal{B}_{M^+}$. Let $\mathbb{M}^+$ 
be the set of all probability measures on $(\M^+,\mathcal{B}_{\M^+})$. We call any element $\upsilon \in \mathbb{M}^+$ a  non-parametric prior.
An element of  $ \mathbb{M}^+$ is called a Dirichlet Process (DP) measure $\mathcal{D}(\nu)$ with base measure $\nu$ if for every finite measurable
partition $B_1,\dots,B_m$ of $\pspace$, the vector $(P(B_1),\dots,P(B_m))$ has a Dirichlet distribution with parameters $(\nu(B_1),\dots,\nu(B_m))$, where $\nu(\cdot)$ is a finite positive Borel measure on $\pspace$  \citep{Ferguson1973}.
As an example, consider the partition $B_1=A$ and $B_2=A^c=\pspace\backslash A$ for some measurable set $A \in \pspace$; then if $P \sim \mathcal{D}(\nu)$, let $s = \nu(\pspace)$ stand for the total mass of $\nu(\cdot)$; from the definition of the DP we have that
$(P(A),P(A^c))\sim Dir(\nu(A),s-\nu(A) )$, which is a Beta distribution.
By computing the moments of the Beta distribution, we derive:
%for instance, that $\mathcal{E}[P(A)]=\dfrac{\nu(A)}{\nu(\pspace)}$,
\begin{equation}
\label{eq:priormom}
\resizebox{.9\hsize}{!}{$
\begin{array} {rl}
\mathcal{E}[P(A)]=\dfrac{\nu(A)}{s},  &\mathcal{V}[P(A)]=\dfrac{\nu(A)(s-\nu(A))}{s^2(s+1)},
\end{array}$}
\end{equation}
where we have used the calligraphic letters $\mathcal{E}$ and
$\mathcal{V}$ to denote  expectation and variance with respect to the DP.
This highlights that the normalized measure $\nu^*(\cdot) = \nu(\cdot)/s$ of the DP reflects the prior expectation of $P$, while the scaling parameter $s$ controls how much $P$ deviates from its mean.
If $P \sim \mathcal{D}(\nu)$, we shall also describe this by saying $P \sim Dp(s,\nu^*)$.
Let $P\sim Dp(s,\nu^*)$ and $f$ be a real-valued bounded function
defined on $(\pspace,\mathcal{B})$. Then the expectation (with respect
to the DP) of $E[f]$ is
\begin{equation}
\label{eq:expf}
\mathcal{E}\big[E(f)\big]=\mathcal{E}\left[\int f dP\right]=\int f d\mathcal{E}[P] = \int f d\nu^*.
\end{equation}
One of the properties of DP priors is that the posterior distribution of $P$ is again a DP.
Let $Z_1,\dots,Z_n$ be an independent and identically distributed
sample from $P$ and $P \sim Dp(s,\nu^*)$. The posterior distribution of $P$ given the observations is 
\begin{equation}
\label{eq:DPPost}
%\resizebox{.9\hsize}{!}{$
P|Z_1,\dots,Z_n \sim Dp\left(s+n, \frac{s\nu^* + \sum_{i=1}^n \delta_{Z_i}}{s+n}\right),
%$}
\end{equation}
where $\delta_{Z_i}$ is an atomic probability measure centered at $Z_i$. The DP is therefore a conjugate prior, since the posterior for $P$ is again a DP with updated unnormalized base measure $\nu+ \sum_{i=1}^n \delta_{Z_i}$.
From Eqs. \eqref{eq:priormom}, \eqref{eq:expf} and
\eqref{eq:DPPost}, we can derive the posterior mean and
variance of $P(A)$ for an event $A$, and the posterior expectation of $f$.
% This means that the Dirichlet process satisfies a property of conjugacy, in the sense that the posterior for $P$ is again a Dirichlet process with updated base measure.
A useful property of the DP (\cite[Ch.3]{ghosh2003bayesian}) is the following:
 Let $P$ have distribution $Dp(s+n, \tfrac{s \nu^*+ \sum_{i=1}^n \delta_{Z_i}}{s+n})$. We can write
\begin{equation} \label{eq:mixing}
P=w_0 P_0+ \sum_{i=1}^n w_i \delta_{Z_i},
\end{equation}
with  $(w_0,w_1,\dots,w_n)\sim Dir(s,1,\dots,1)$,  $P_0 \sim Dp(s,\nu^*)$.

Since the DP is a  measure on probability distribution functions, it can be used to model prior information on probability measures:
\begin{equation}
 \label{eq:embDP}
 \mu_k (P) = E(k(\cdot,z))\!=\int\! k(\cdot,z) d P(z),~ P \sim Dp(s,\nu^*),
\end{equation}
and, therefore, can be used to compute a posterior
distribution over HSIC.
% DPs, in combination with other processes, can be used to model priors on signed measures~\citep{liang2007nonparametric}.

In some applications, the underlying probability measure is required to satisfy specific constraints (e.g., symmetries). This leads to \textit{invariant DPs}.
In particular, we consider invariance to exchangeability (that is, permutation invariance). The exchangeable measures on $\mathcal{Z}^n$ are invariant measures under the action of the permutation group $\Pi_n$ of order $n$ on the order of the coordinates of the vectors. We can  build a  DP that is invariant under the action of the permutation group.
Let $IDp(s,\nu^*,\Pi_n)$ denotes such invariant DP prior. Then the posterior distribution of $P$ given the observations is 
\begin{equation}
\label{eq:IDPPost}
%\resizebox{.9\hsize}{!}{$
P|Z_1,\dots,Z_n \sim Dp\left(s+n, \dfrac{s\nu^* +\frac{1}{n!}\sum\limits_{\pi \in \Pi_n} \sum\limits_{i=1}^n \delta_{Z_{i\pi}} }{s+n}\right).
%$}
\end{equation}
To understand the above formula, assume that we focus on the components $X,Y$ of the vectors $Z_i$, that is, $Z_i=[X_i^T,Y_i^T]^T$ for $i=1,\dots,n$. Then we can define $Z_{i\pi}$ as $([X_i^T,Y_{\pi(1)}^T]^T,[X_i^T,Y_{\pi(2)}^T]^T,\dots,[X_i^T,Y_{\pi(n)}^T]^T)$, where 
$[\pi(1),\dots,\pi(n)]$ denotes a permutation of  $[1,\dots,n]$.
% The number \eqref{eq:IDPPost}

\section{Bayesian Kernel Independence Test}
To carry out Bayesian analysis, we start by specifying a prior distribution on the unknown quantity
of interest, that is, $P_{XY}$. In particular, we consider a non-parametric prior by  assuming that  the joint distribution $P_{XY}$ is DP distributed: $P_{XY}\sim Dp(s,\nu^*)$.
By introducing the joint variable $Z=[X,Y]^T$, from Equation~\eqref{eq:DPPost} we have that the posterior distribution of $P(Z):=P_{XY}(Z)$ is:
$P=w_0 P_0 + \sum_{i=1}^n w_i \delta_{Z_i}$,  
 with $(w_0,w_1,\dots,w_n)\sim Dir(s,1,\dots,1)$ and $P_0\sim Dp(s,\nu^*)$.
 Hence, based on  Proposition \ref{prop:HSIC_as_expectations},  we derive the following result (proofs are in Supp. Material).
% compute the posterior mean 
% of the HSIC distance w.r.t.\  $P$.

\begin{theorem}
 \label{th:HSIC}
Let  $Z,Z'\sim P$ and $P$ have distribution $w_0 P_0 + \sum_{i=1}^n w_i \delta_{Z_i}$, with $W=(w_0,w_1,\dots,w_n)\sim Dir(s,1,\dots,1)$ and $P_0\sim
Dp(s,\nu^*)$. Then
 \begin{align}
 \nonumber
\widehat{HSIC}(X,Y)&=W\mathbb{K}^{XX'YY'}W^T +W\mathbb{K}^{XX}W^TW\mathbb{K}^{YY}W^T\\
\label{eq:whsicxy}
 &-2\, W\left(\mathbb{K}^{XX}W^T\circ\mathbb{K}^{YY}W^T\right),
 \end{align}
   where $\circ$ denotes the Schur product, $\mathbb{K}^{XX}$ is a symmetric $(n+1)\times (n+1)$ matrix such that $\mathbb{K}^{XX}_{00}=\int K_{\mathcal{X}}(x,x') d(P_0(z)P_0(z'))$,   $\mathbb{K}^{XX}_{0i}=\int K_{\mathcal{X}}(x,X_i) d P_0(z)$ for $i>0$ and $\mathbb{K}^{XX}_{ij}=K(X_i,X_j) $ for $i=1,\dots,n$ and $j\geq i$;
   $\mathbb{K}^{YY}$ is similarly defined; $\mathbb{K}^{XX'YY'}$ is a
   symmetric $(n+1)\times (n+1)$ matrix such that
   $\mathbb{K}^{XX'YY'}_{00}=\int
   K_{\mathcal{X}}(x,x')K_{\mathcal{Y}}(y,y')  d(P_0(z)P_0(z'))$,
   $\mathbb{K}^{XX'YY'}_{0i}=\int
   K_{\mathcal{X}}(x,X_i)K_{\mathcal{Y}}(y,Y_i) d P_0(z)$ for $i>0$
   and
   $\mathbb{K}^{XX'YY'}_{ij}=K_{\mathcal{X}}(X_i,X_j)K_{\mathcal{Y}}(Y_i,Y_j)
   $ for $i=1,\dots,n$, $j\geq i$.
\end{theorem}
$\widehat{HSIC}$ is the   $HSIC$ computed with respect to the posterior distribution $P$.

\begin{theorem}
 \label{th:postmean}
 The posterior mean of $\widehat{HSIC}(X,Y)$ (given in the Supp. Material) tends to the asymptotic limit of the sample estimate of $HSIC(X,Y)$ (Eq.~\eqref{prop:HSIC_as_expectations}) for $n \rightarrow \infty$.
\end{theorem}
This shows the connection between Expressions~\eqref{eq:HSIC_as_expectations} and $\widehat{HSIC}(X,Y)$
and allows us to consequently derive its asymptotic consistency  (see Supp. Material). % (posterior concentrates%Expression~\eqref{eq:HSIC_as_expectationsemp} tends to the asymptotic limit of Expression~\eqref{eq:postmean} for $n \rightarrow \infty$. 
% In other words, the posterior mean converges to the same limit of the statistic of the NHST test.

The advantage of the Bayesian test is that we  can compute the posterior distribution of $\widehat{HSIC}$ by Monte Carlo sampling the weight vector $W$ from the Dirichlet distribution
$Dir(s,1,\dots,1)$ and $P_0$ from the prior $Dp(s,\alpha^*)$ (for instance, via stick-breaking). % as depicted in Figure \ref{fig:}.
%In Bayesian analysis, the experiment is summarized by the posterior distribution.
This posterior distribution represents the credibility across possible HSIC values that takes into account both prior knowledge and collected data.
By querying the posterior distribution, we can evaluate the probability of the hypotheses.
First, we define
\begin{equation}
\label{eq:Krho}
   \overline{HSIC}(X,Y) =\frac{\widehat{HSIC}(X,Y)}{\sqrt{\widehat{HSIC}(X,X)\widehat{HSIC}(Y,Y)}}.
\end{equation} 
This is the Bayesian equivalent of  \textit{distance correlation}. Provided that the three
terms  $\widehat{HSIC}(X,Y)$, $\widehat{HSIC}(X,X)$, and $\widehat{HSIC}(Y,Y)$ %with respect to 
are computed
w.r.t. the same joint distribution $P_{XY}$, then we have $0\leq \overline{HSIC}(X,Y)  \leq 1$.

In our case, the hypotheses are $X \nindep Y$ versus $X \indep Y$, and hypothesis testing can be performed  by  computing 
the posterior probability that $\overline{HSIC}(X,Y) $ exceeds a suitable threshold $\tau(X,Y)$:
\begin{equation}
\label{eq:bayesdec}
\mathcal{P}\left[\overline{HSIC}(X,Y)-\tau(X,Y) >0\right],\end{equation} 
where $\mathcal{P}$ denotes the probability computed  with respect to $W\sim Dir(s,1,\dots,1)$ and $P_0\sim Dp(s,\alpha^*)$.
This gives us the posterior probability of the two variables being dependent, while $\mathcal{P}\left[\overline{HSIC}(X,Y)-\tau(X,Y) \leq 0\right]=1-\mathcal{P}\left[\overline{HSIC}(X,Y)-\tau(X,Y) >0\right]$
gives us the posterior probability of them being independent.
In order to define a meaningful threshold $\tau(X,Y)$, we simply shift $\overline{HSIC}$ to have zero mean whenever $X,Y$ are independent. To do so, we must compute the posterior mean
of $\overline{HSIC}$ under independence, which
is obtained by observing that, under independence, 
the sequence of observations $(X_1,Y_{\pi(1)}),\dots, (X_n,Y_{\pi(n)})$ is exchangeable, for any permutation
$[\pi(1),\dots,\pi(n)]$  of  $[1,\dots,n]$.
The posterior DP is given by Eq.\eqref{eq:IDPPost} and, therefore,  $\tau(X,Y)=\frac{1}{n!}\sum_{\pi \in \Pi_n} \widehat{HSIC}(X,Y_{\pi})/\sqrt{\widehat{HSIC}(X,X)\widehat{HSIC}(Y,Y)}$.
 The exact expression of the posterior mean requires the computation of $\widehat{HSIC}(X,Y_{\pi})$ for all possible permutations, which is infeasible for large $n$.
 In spite of that, we approximate this sum by Monte Carlo sampling (see Supp. Material). Hence, we define a Bayesian corrected $\dCor$ as follows: $\BdCor(X,Y)=$
\begin{equation}
\label{eq:Kdcor}
\begin{aligned}
   &=\left(\tfrac{\widehat{HSIC}(X,Y)}{\sqrt{\widehat{HSIC}(X,X)\widehat{HSIC}(Y,Y)}}-\mathcal{E}(\tau(X,Y))\right)\tfrac{1}{1-\mathcal{E}(\tau(X,Y))},\\
   &W=(w_0,\dots,w_n)\sim Dir(s,1,\dots,1), ~P_0\sim Dp(s,\nu^*).
\end{aligned}
\end{equation} 
This ensures  $\BdCor(X,Y)$ is centred at the origin whenever $X,Y$ are independent.
Similarly to the approach proposed in \citep{SZEKELY2013193}, the Eq.\eqref{eq:Kdcor} allows us to remove the bias of $\overline{HSIC}(X,Y)$ towards dependence  that increases with the dimension of the variables being compared.\footnote{Note that
$\BdCor(X,Y)$ can also assume values that are less than zero, the same happens for the correction proposed in \citep{SZEKELY2013193}.}
The normalization $1-\mathcal{E}(\tau(X,Y))$ guarantees
that $\BdCor(X,Y)$ can still span the whole interval $[0,1]$, which allows us to introduce a \textit{Region Of Practical Independence} (ROPI).
ROPI indicates a small range of $\BdCor(X,Y)$ values that are considered to be practically equivalent to the null value (the origin, representing
independence) for purposes of the particular application. 
%We borrow this idea from the region of practical equivalence (ROPE), used to indicate an
%interval of small effect sizes that are practically equivalent to the null case of no effect~\citep{kruschke2010bayesian}. 

%,kruschke2017bayesian}.
% \begin{figure*}[htp!]
%     \centering
%         \centering
%         \includegraphics[height=1.3in]{indep.pdf} \includegraphics[height=1.3in]{none.pdf} \includegraphics[height=1.3in]{dep.pdf}
%         \caption{Three possible different outputs of the Bayesian test.}
%         \label{fig:1}
% \end{figure*}

For example, in feature selection, it is common practice to  decrease the significance level $\alpha$ of the independence test (for instance to $0.001$ or $0.0001$) in order to (indirectly) induce  sparsity and so to improve accuracy in high dimensional data. This is necessary because NHSTs do not allow one to accept the null hypothesis or to assess whether two variables are only weakly dependent. %#, and obviously this does not allow to reject independence for  some features with p-values just below the 0.05$).   
Instead, the proposed test gives us the flexibility of accepting the null hypothesis and assessing weak dependence through a suitable choice of the ROPI and, therefore, it naturally induces sparsity while retaining dependent features. 
Once a ROPI is chosen, we take a decision to reject a null value according to the rule: Independence is declared to be not credible, or rejected, if the posterior probability that $\BdCor>ROPI$ is greater than a suitable threshold.
%(i.e., $0.85,0.9,0.95,...$).

The use of a ROPI around the null value implies that if the null value really is true, we will eventually ``accept'' the null value as the sample size becomes larger.
In fact, as the sample size increases, as the posterior tends to get concentrated and closer to the
true value. When the sample size becomes suitably large, the high-density interval of the posterior is almost certain to be narrow enough and close enough to the true value, and hence to fall (almost) completely within the ROPI. 
Hence, a natural way to define a ROPI is to declare independence when $\BdCor $ is small enough, that is, $\BdCor\leq ROPI$ for some $ROPI \in [0,1]$.
% For that purpose, we define $\gamma=\tau(X,Y) + \rho_{th}\sqrt{||\Sigma_{XX}||_{HS}^2||\Sigma_{YY}||_{HS}^2}$ in Expression \eqref{eq:bayesdec} and we call 
% $HSIC_{s}=HSIC-\gamma$.  %Under independence, $HSIC_{s}$ has negative mean due to $ \rho_{th}\sqrt{||\Sigma_{XX}||_{HS}^2||\Sigma_{YY}||_{HS}^2}$ 
% Note that $||\Sigma_{XX}||_{HS}^2:=\widehat{HSIC}(X,X)$ and $||\Sigma_{YY}||_{HS}^2:=\widehat{HSIC}(Y,Y)$ can be computed 
% as in Theorem \ref{th:HSIC} by simply imposing $Y=X$ and $X=Y$, respectively.
The pseudo-code to perform the Bayesian non-parametric kernel test of independence is reported in the Supp. Material.

% Figure \ref{fig:1} shows the posterior distribution of $HSIC_{s}$ in three scenarios. 
% The number of samples $N_{mc}$ is equal to $1000$ in the next examples and figures.
%  In the left case, 
% ``all''  posterior mass is to the left of the origin (we can declare independence with high probability). In the center, the posterior mass is around the origin (either independent or dependent with almost equal probability). 
% In the right case,  ``all'' the posterior mass is to the right of the origin (we can declare dependence with high probability).

\subsection{Priors, fast computation and large-scale approximation}
\label{sec:addon}
In order to completely specify the DP  prior on $P_{XY}(Z)$, we need to choose the prior probability measure $\nu^*(Z)$ and the scale parameter $s$.
Here we consider the limiting DP obtained for $s \rightarrow 0$~\citep{Ferguson1973}. %,Rubin1981}.
In this way, the DP posterior does not depend on $\nu^*(Z)$. Thus, we can more efficiently compute $\widehat{HSIC}$ as follows:
\begin{corollary}
 \label{co:fast}
  $\widehat{HSIC}=Tr(K^{XX}RK^{YY}R)$ with $R=\text{diag}(W)-W^TW$ and $K_{ij}^{XX}=K(X_i,X_j), K_{ij}^{YY}=K(Y_i,Y_j)$.
\end{corollary}
Note that each matrix in the above expression is $n \times n$ and, therefore, the computational cost and memory storage to determine a sample of  $\widehat{HSIC}$ becomes prohibitive for large $n$. A way to address this issue is to use the so called low-rank approximations of the kernel matrix. The most common approaches are the Nystr\"{o}m method and the random features. We have extended the approach in \cite{zhang2018large} to derive a  \textbf{low-rank approximation} of the posterior distribution of $\BdCor$ based on the Nystr\"{o}m method.
\begin{theorem}
 \label{th:low-rank}
$\widehat{HSIC}(X,Y)$ can be approximated as
$$
\widehat{HSIC}(X,Y) = Tr(K^{XX}RK^{YY}R) \approx ||\tilde{\phi}_{X}^T R \tilde{\phi}_{Y}||_F^2,\\
$$
where $||\cdot||_F$ is the Frobenius norm and  $\tilde{\phi}_X, \tilde{\phi}_Y\in \mathbb{R}^m$ are the feature  representation of the kernel 
(Nystr\"{o}m method) for the chosen approximating dimension $m$.
\end{theorem}
The computational complexity is now $O(nm^2)$.
%The details and derivations are in Supplementary.
$\BdCor$ is a pairwise method, however in many applications
we must assess the association of more than one pair of variables at a time and then return joint statements of dependence/independence.
For this, we follow the approach described in \cite{benavoli2015b} that starts from the pair of variables  having the highest posterior probability of being dependent/independent and accepts as many statements as possible stopping when  the joint 
posterior probability of all being true is less than $\gamma$ (e.g. using $\gamma=0.85$ or $0.9$).
The multiple comparison proceeds as follows.
\vspace{-0.3cm}
\begin{itemize}
\itemsep0em 
\item For each pair $i,j$, perform a BKR test and derive the posterior probability $\mathcal{P}(\BdCor_{ij}>ROPI)$  and $\mathcal{P}(\BdCor_{ij}\leq ROPI)$ . Select the direction with highest posterior probability for each pair $i,j$.
\item Sort the posterior probabilities obtained in the previous step for the various pairwise comparisons in decreasing order. Let $\mathcal{P}_1,\dots,\mathcal{P}_k$ be the sorted posterior probabilities and $S_1,\dots,S_k$ the corresponding statements $\BdCor_{ij}\lesseqgtr ROPI$.  
% \item 
% Otherwise, we make the statement $S_1$ corresponding to the hypotheses with the highest posterior probability.
%Otherwise, we accept all the first $\tilde{m}$ hypotheses for which the joint posterior probability of all them being true is less than $(1-\gamma)\%$, that is $P(H_1\wedge)$
\item Accept all the statements $S_i$ with $i \leq  \ell$, where $\ell$ is the greatest integer s.t.
$\mathcal{P}(S_1 \wedge S_2 \wedge \dots \wedge S_\ell)>\gamma$.
% \item If $\mathcal{P}(H_1)>1-\gamma$ we check if the joint posterior probability of the two hypotheses being true, that is $\mathcal{P}(H_1\wedge H_2)$, is larger than $1-\gamma$. If so, make also the statement $S_2$.
% \item Accept all statements $S_i$, $i=1,\dots,K$ such that $\mathcal{P}(P(S_i)>0.5)>1-\gamma$ and 
% \begin{equation} \label{eq:condition}
% \mathcal{P}( H_1\wedge H_2\wedge \dots \wedge  H_K)>1-\gamma
% \end{equation}
% and stop when the joint probability of all the hypotheses being true falls below $(1-\gamma)\%$.
\end{itemize}
\vspace{-0.3cm}
If none of the hypotheses has at least $\gamma$ posterior probability of being true, then we reach no statements. 
The joint posterior probability of multiple statements $\mathcal{P}(S_1 \wedge S_2 \wedge \dots \wedge S_\ell)$ can be computed numerically by Monte Carlo sampling the vector $W$ from the same Dirichlet distribution as for each comparison $S_i$.  This way we ensure that the posterior probability $1-\mathcal{P}(S_1 \wedge  S_2 \wedge \dots \wedge  S_\ell)$ that there is an error in the list of accepted statements is lower than $1-\gamma$. 
Hence, this Bayesian approach does not assume independence between the different hypotheses, such as the NHST does; instead, it considers their joint distribution at once.
% Notice that, according to the 
\vspace{-0.2cm}
\section{Numerical experiments}
\label{numexp}
\vspace{-0.2cm}
\subsection{Synthetic data}
\label{ss}

Two synthetic datasets named D1 and D2 were created in order to assess the \textit{decision accuracy}, that is, the fraction of independent and dependent relationships that were recovered at a given credible level.
D1 includes $6$ variables generated as follows:

\begin{small}
\begin{align*}
 X &\sim N(0,1),\quad Y \sim \text{B}^{-1}(\Phi(T)),\quad T\sim \mathcal{N}(0,1),\\
 C_X^{\rho}&=\rho X + (1-\rho^2)^{0.5} W_1, ~W_1\sim \mathcal{N}(0,1),\\
 D_X^{\rho}&=\text{B}^{-1}(\Phi(Z)),~  Z=\rho X\! + (1\!-\rho^2)^{0.5} W_2, ~W_2\sim \mathcal{N}(0,1),\\
  D_Y^{\rho}&=\text{B}^{-1}(\Phi(Z)),~  Z=\rho T\! + (1\!-\rho^2)^{0.5} W_3, ~W_3\sim \mathcal{N}(0,1),\\
\mathbb{C}_X^{\rho}&=U^{-1}(\Phi(Z)), ~Z=\rho X\! + (1\!-\rho^2)^{0.5} \mathbb{W}, ~\mathbb{W}\sim \mathcal{N}(0,I_{1024}),
\end{align*}
\end{small}\noindent 
where $\text{B}^{-1},U^{-1},\Phi$, are the inverse cumulative distribution function (CDF) of the Bernoulli distribution (with probability $0.5$), the inverse CDF of the Uniform distribution in $[0,1]$, and the CDF of the Normal distribution, respectively.
The variables $X,C_X^{\rho},\mathbb{C}_X^{\rho}$ are continuous, while $Y,D_X^{\rho},D_Y^{\rho}$ are binary.
$\mathbb{C}_X^{\rho}$ has dimension 1024 from a $32\times32$ image. This synthetic data
resembles the \textit{House} dataset discussed in Section~\ref{sec:intro}. 
By increasing $|\rho| \in [0,1]$ we increase the dependence among variables and, therefore, we evaluate the effectiveness of BKR.\footnote{Note that in practice we are simulating dependence using a Gaussian Copula.}
From the above relationships, we immediately see
that $X$ and $Y,D_Y^{\rho}$, as well as $Y$ and $C_X^{\rho},D_X^{\rho},\mathbb{C}_X^{\rho}$, are  pairwise independent. Therefore, for $\rho>0$, 
there are 7/15 independent variable pairs 
and 8/15 dependent variable pairs.
Table \ref{tab:1} in the Supp. Material shows an instance of D1 including $n=5$ observations generated with $\rho=0.9999$.

We compare BKR with the NHST %independence test 
based on HSIC \citep{gretton2008kernel} on 100 different Monte Carlo generated instances of D1.
For BKR, we make a decision of dependence when the posterior probability of $\dCor>ROPI$ is greater than $0.85$ and a decision
of independence when the posterior probability of $\dCor<ROPI$ is greater than $0.85$. ROPI has been fixed to $0.025$ in the simulations. Note that joint 
decisions for the 15 pairwise comparisons are made
according to the methodology described in Section~\ref{sec:addon}. For the HSIC test, we reject the independence hypothesis when the p-value is less than $\alpha=0.05/15$ (Bonferroni correction). However, the HSIC test cannot ever make a decision of independence. For BKR and HSIC, we use a square exponential kernel with length-scale set to the median distance
between points in input space (for continuous variables) and the indicator function (for binary variables).
%\footnote{It is well-know that the square
%exponential kernel is universal and, as a consequence, characteristic~\citep{gretton2012kernel}. The indicator kernel is universal, thus
%characteristic as well. Indeed, in a discrete space $\mathcal{V}=\{v_1,\ldots,v_r\}$, the universality of a kernel is a consequence of the strictly
%positive definiteness of the associated Gram matrix $K$, that is, $[K]_{ij}=k(v_i,v_j)$, $\forall v_i, v_j \in \mathcal{V}$~\citep{borgwardt}. In the case of the indicator kernel, the Gram matrix corresponds to the identity matrix, which is strictly positive definite.} 
With this choice, HSIC/BKR can detect any  dependence in D1.

\begin{figure}[htp]
\centering
\vspace{-9pt}
 \begin{tabular}{r}
\includegraphics[height=1.4in]{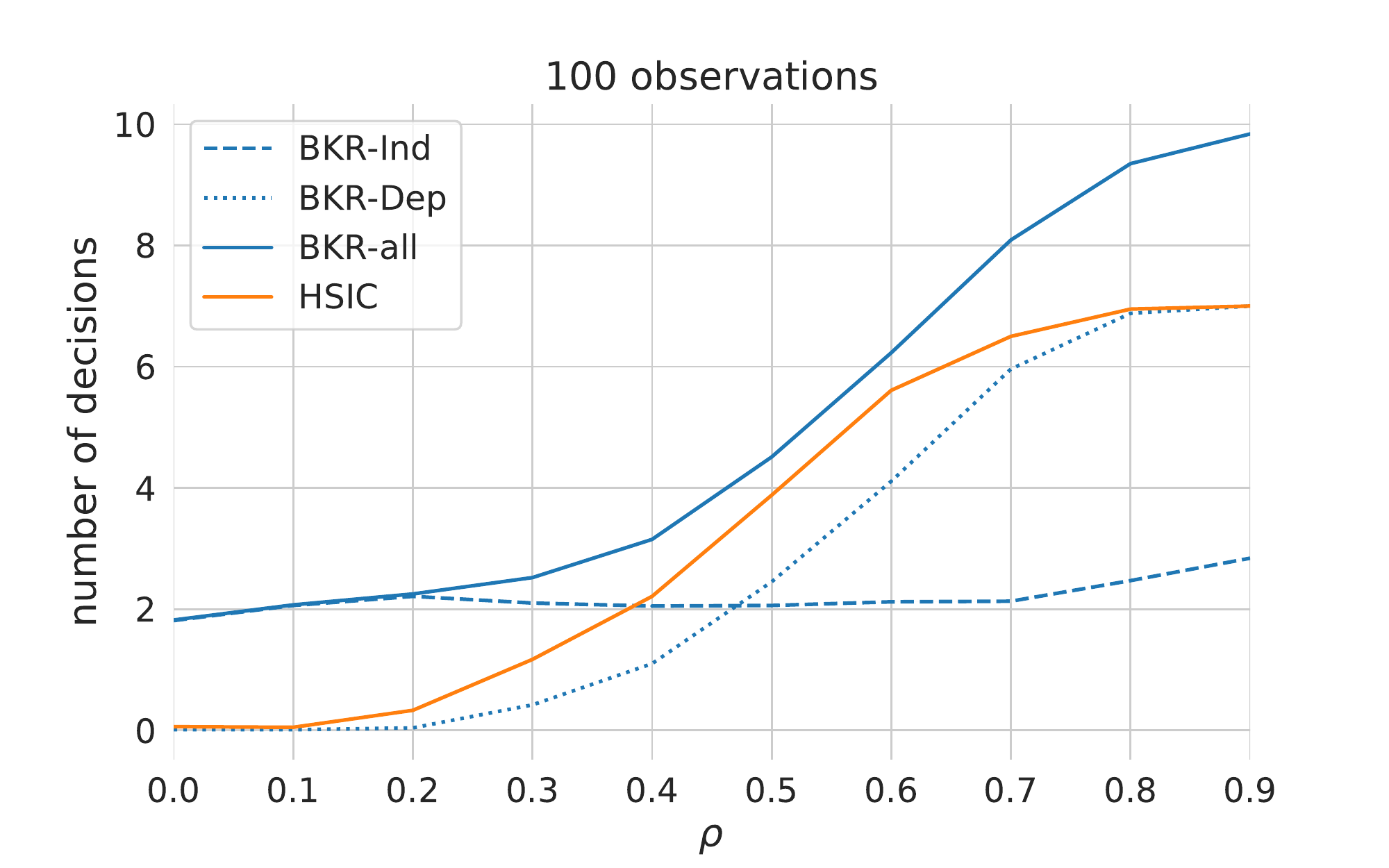}\\ %\hspace{-5.4mm} 
\includegraphics[height=1.4in]{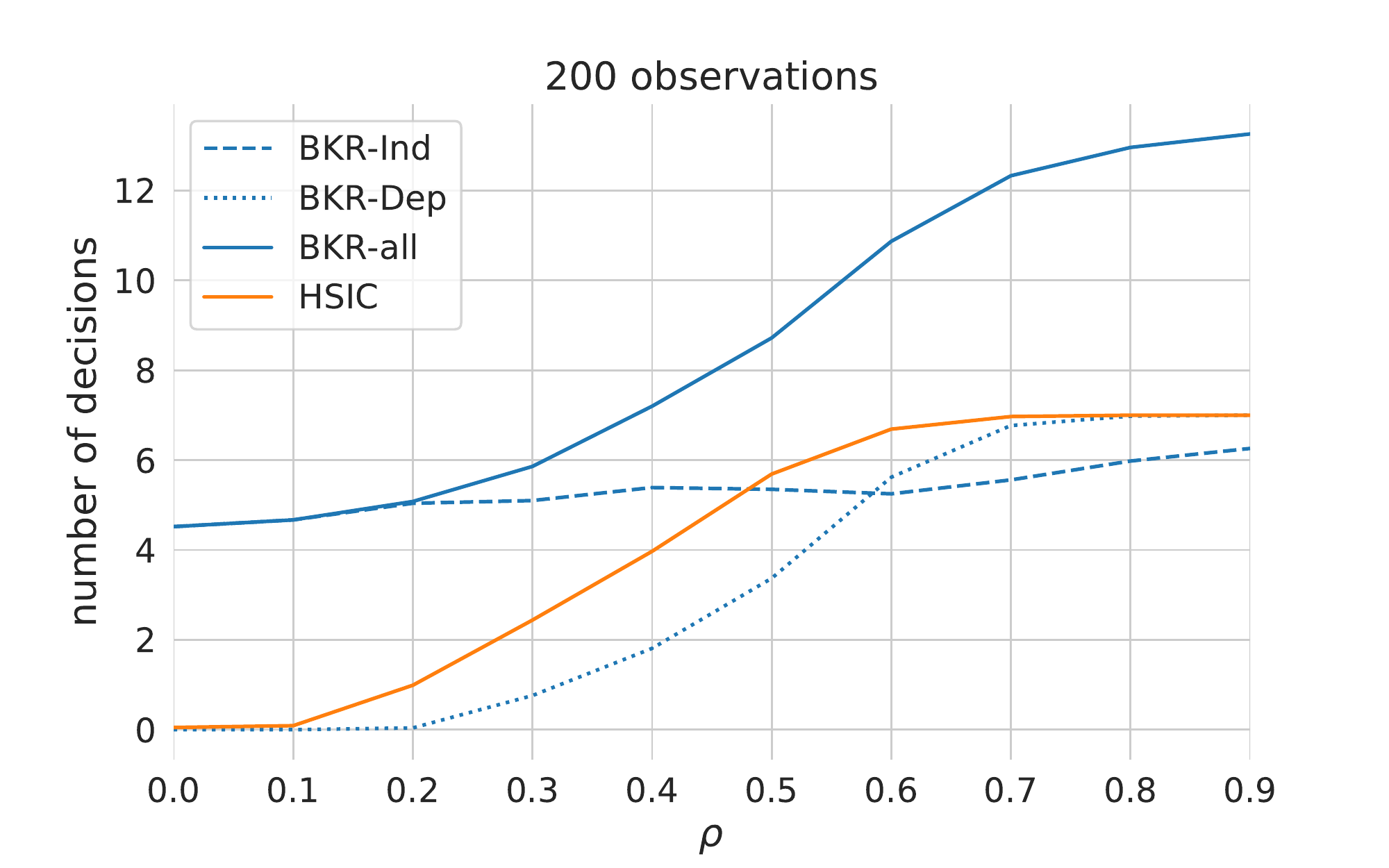} %\hspace{-5.4mm} 
\end{tabular}
\vspace{-5pt}
\caption{Synthetic dataset D1. Overall, BKR always makes more decisions than HSIC.}
\label{fig:2}
\vspace{-10pt}
\end{figure}

Figure~\ref{fig:2}(top) shows results of the comparison
for a dataset D1 of $100$ observations.
We have reported the number of decisions of (i)
independence for BKR (BKR-Ind); (ii) dependence for BKR (BKR-Dep); (iii)  overall dependence/independence  for BKR (BKR-all); (iv) dependence for HSIC.
BKR makes overall more decisions than HSIC for all values of the simulated correlation $\rho$. For instance for $\rho=0.9$, BKR makes $\sim 10$ decisions and HSIC only $7$. The reason is that HSIC can never declare independence. It can also be observed that BKR is calibrated under the null hypothesis (number of dependence decisions is zero when $\rho=0$) and moreover it can detect two independences. Dataset D2 is generated in a similar way but simulating a nonlinear dependence using a Clayton Copula (instead of a Gaussian Copula). Similar results were obtained for the Clayton Copula which are reported in the Supp. Material.

\vspace{-0.2cm}
\paragraph{MIC.}
We compare BKR against the \emph{Maximal Information Coefficient} (MIC)~\citep{reshef2011detecting}. The MIC
statistic measures (in a non-parametric way) the strength of the association between two continuous variables; in principle it can detect any functional
dependence between two variables. A MIC-based NHST of independence is obtained as described in \citep{reshef2011detecting}.
For comparison, we consider the performance statistics from the 2008 Major
League Baseball (MLB) season~\cite{reshef2011detecting}
(131 variables).
\setlength{\columnsep}{0pt}
\begin{figure}
\begin{footnotesize}
%\begin{wrapfigure}{r}{0pt}
\centering
%\vspace{-85pt}
 \setlength\tabcolsep{0.0pt}
  \begin{tabular}{c  c   }
    \includegraphics[width=.4\linewidth,trim={3cm 0.0cm 2.3cm 0},clip]{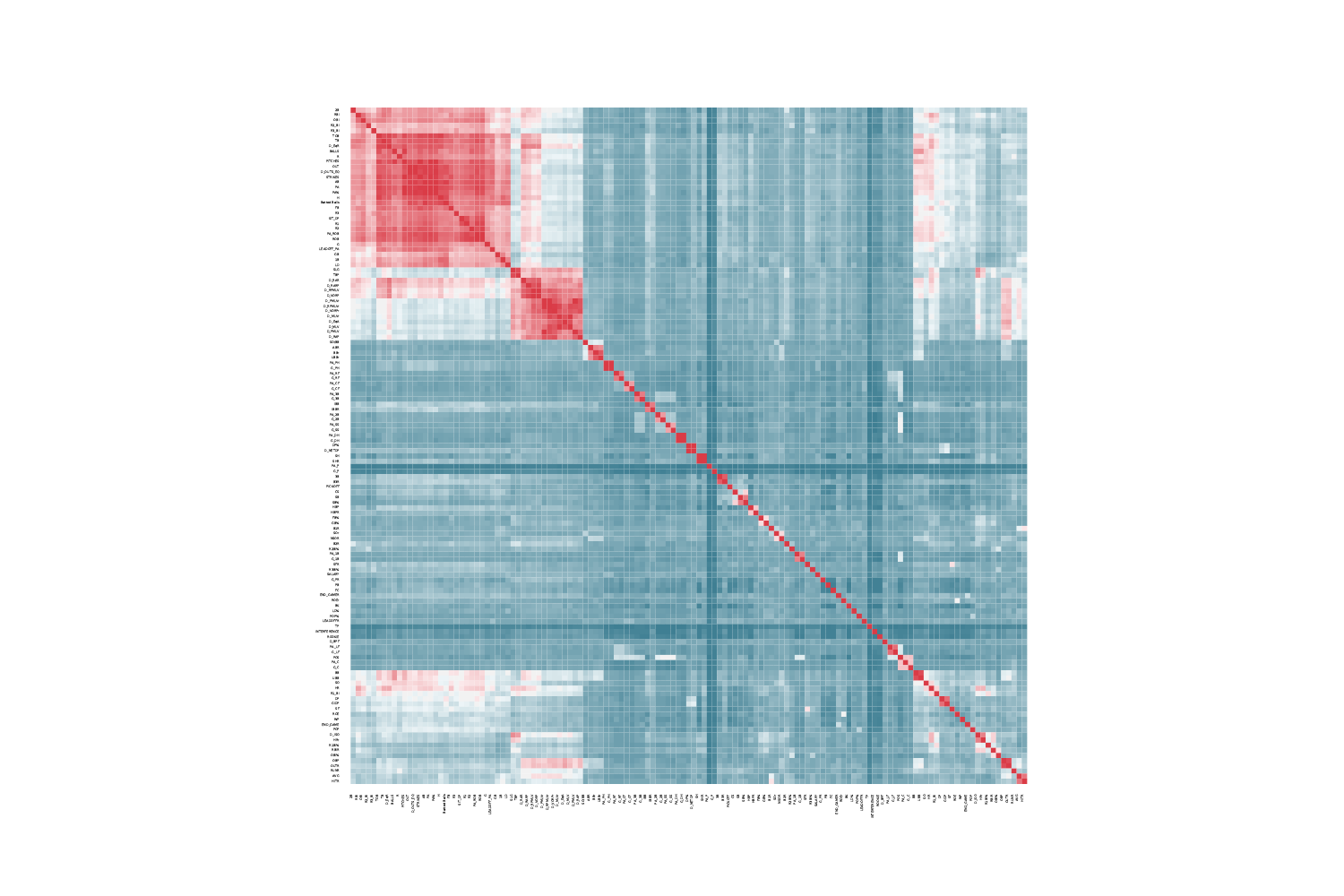} &
    \includegraphics[width=.4\linewidth,trim={3cm 0.0cm 2.3cm 0},clip]{{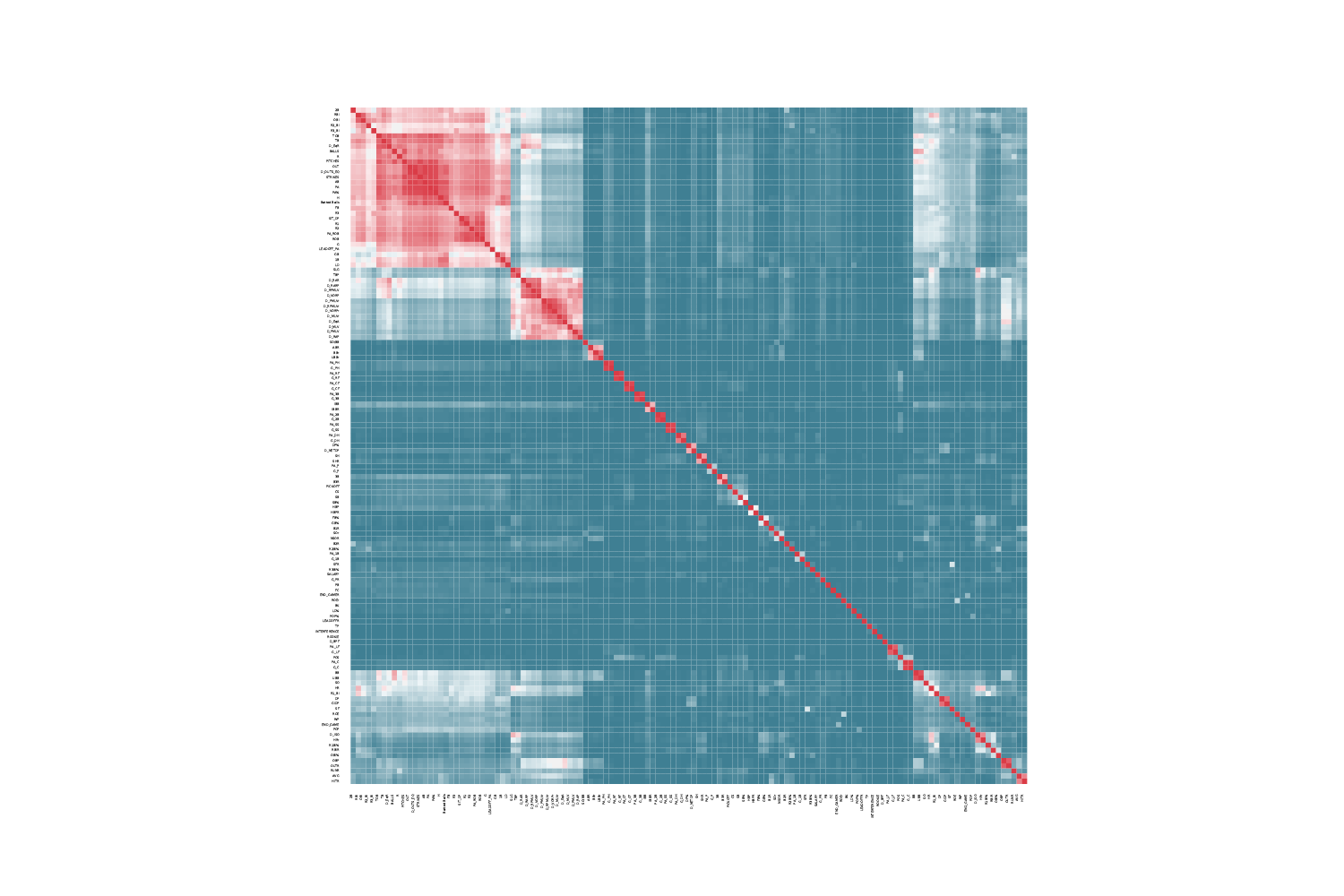}} \\
%    \small MIC stats & \small $\mathcal{E}(\BdCor)$ \vspace{-8pt}\\
  \end{tabular}
\caption{MLB dataset overview heatmap: red shows dependences and blue shows independences. In the left, MIC statistics. In the right, $\mathcal{E}(\BdCor)$. (Details of the graphs are not visible -- they are intended to provide an overview picture only.)}
%\vspace{-15pt}
\label{Fig:mic1}
\end{footnotesize}
\end{figure}
%\end{wrapfigure}
% We start with the MLB data set (131 variables). 

Figure \ref{Fig:mic1} shows the values
 of the MIC statistics and the posterior mean
 $\mathcal{E}(\BdCor)$ as a heatmap for all the $8515$ pairwise comparisons. In both cases, red denotes large values
 of the statistics (a sign of dependence) and blue small values (a sign of independence). The two measures provide a similar evaluation of the strength of the pairwise associations: this shows that our proposed $\BdCor$ measure is commensurable.

\begin{figure*}
\centering
\begin{minipage}[t]{0.45\textwidth}
 \setlength\tabcolsep{0.0pt}
\begin{tabular}{c  c  c }
    \includegraphics[width=.30\linewidth,trim={3.8cm 0.0cm 2cm 0},clip]{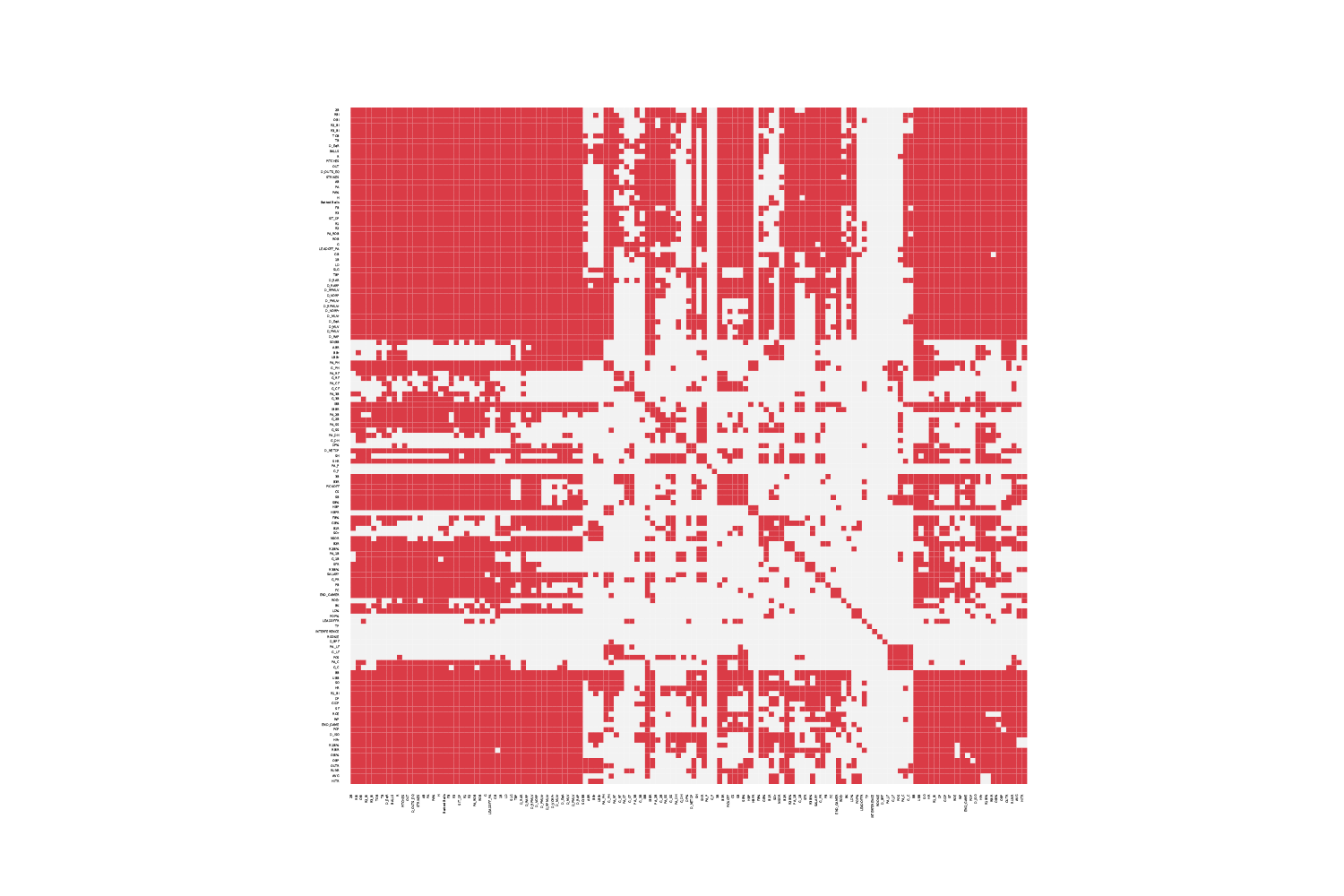} &
    \includegraphics[width=.30\linewidth,trim={3.8cm 0.0cm 2cm 0},clip]{{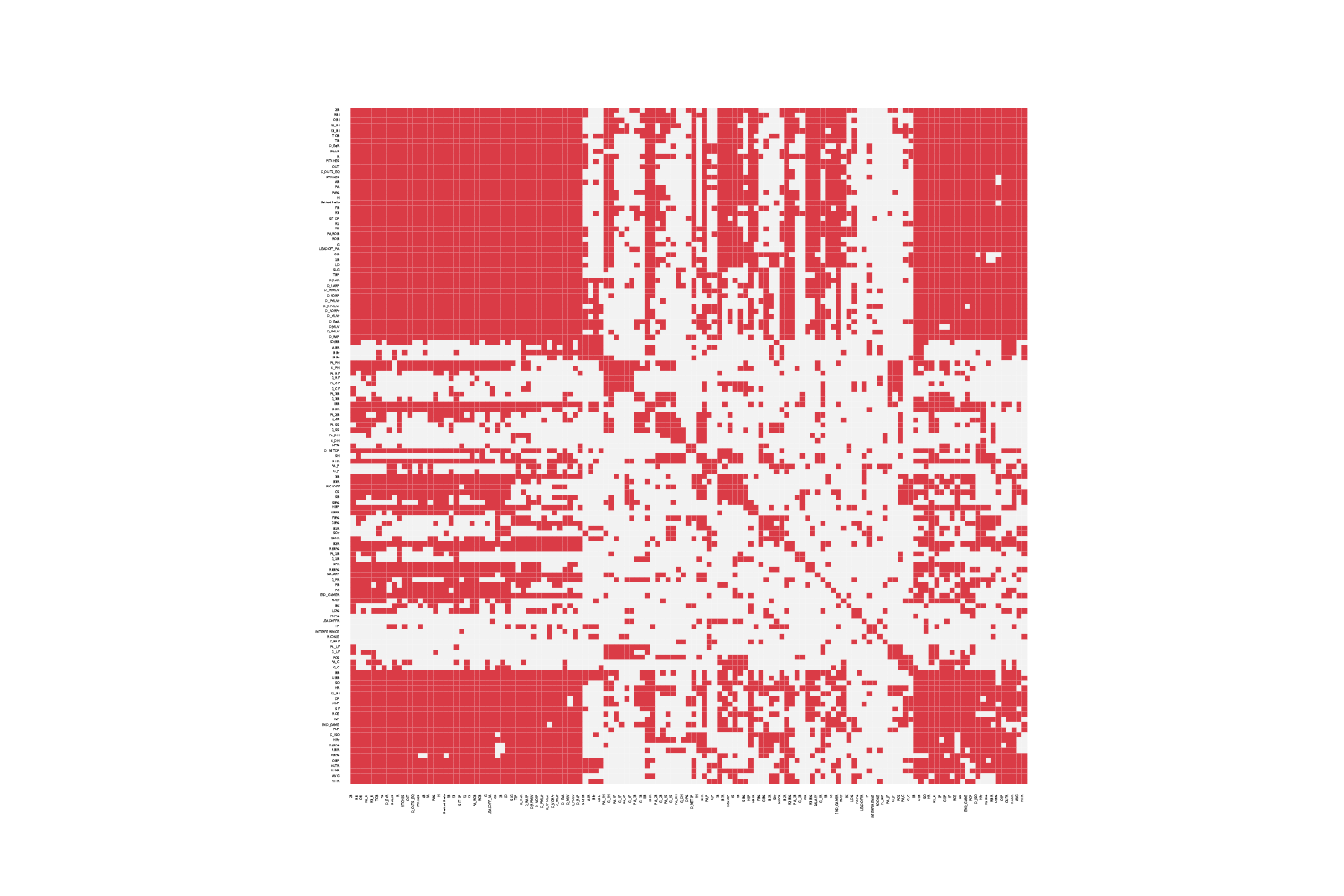}} &
    \includegraphics[width=.30\linewidth,trim={3.8cm 0.0cm 2cm 0},clip]{{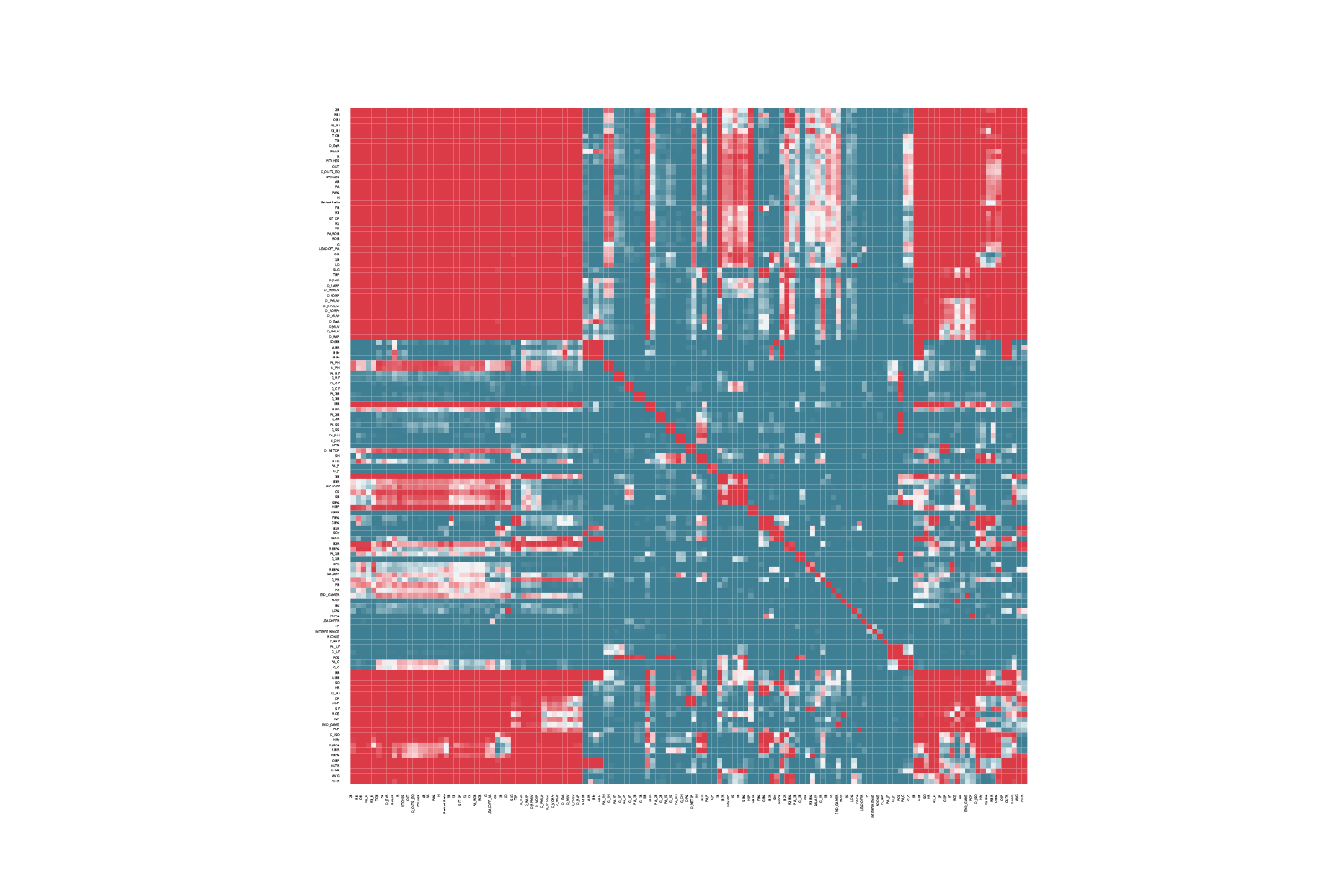}}\\
    \small HSIC & \small MIC & \small BKR\\
    \end{tabular}
\end{minipage}\hspace{-0.65cm}
\begin{minipage}[t]{0.57\textwidth}
 \setlength\tabcolsep{0.0pt}
 \begin{tabular}{c  c  c }
     \includegraphics[width=.35\linewidth,trim={0.3cm 0.0cm 1cm 0},clip]{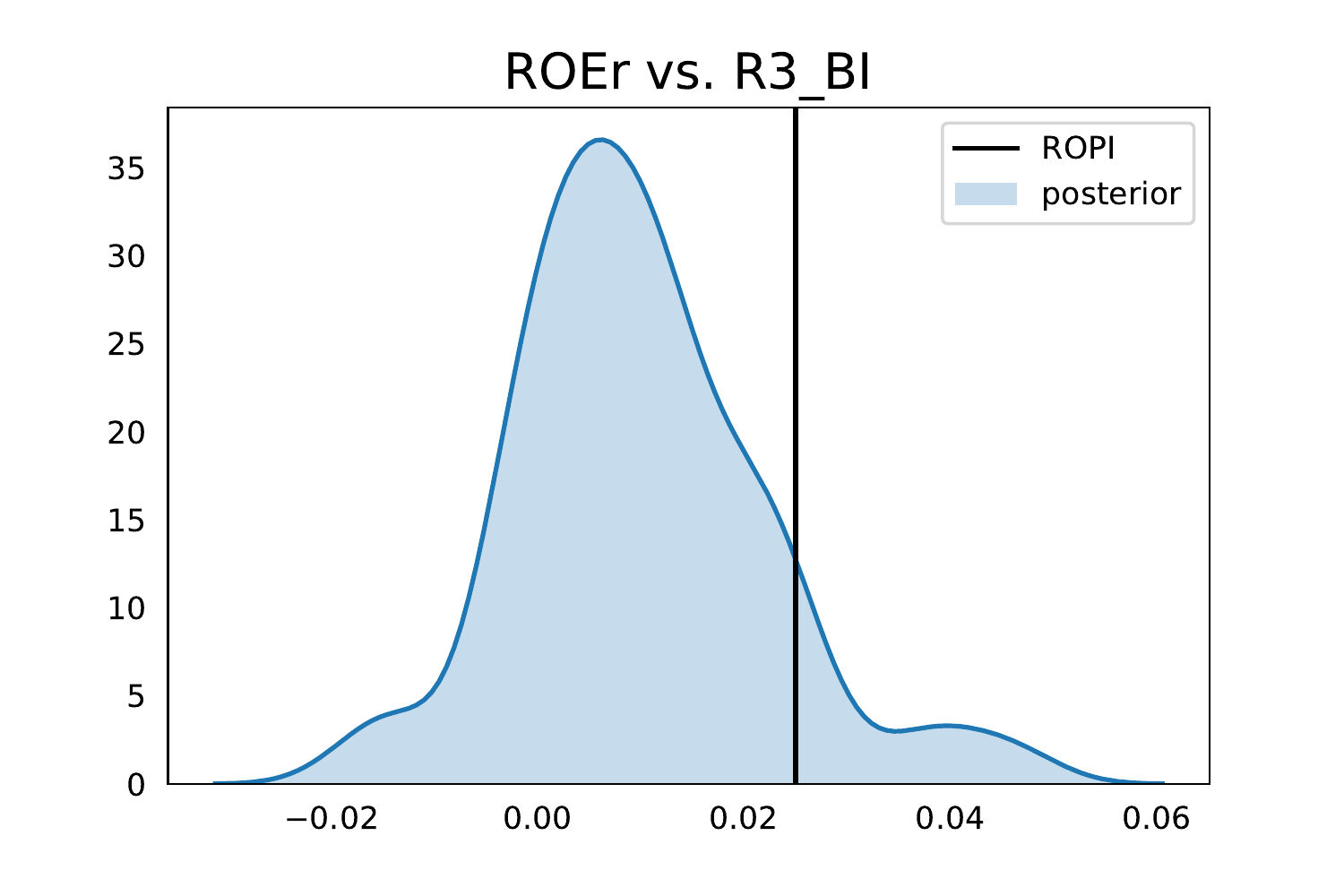} &
    \includegraphics[width=.35\linewidth,trim={0.3cm 0.0cm 1cm 0},clip]{{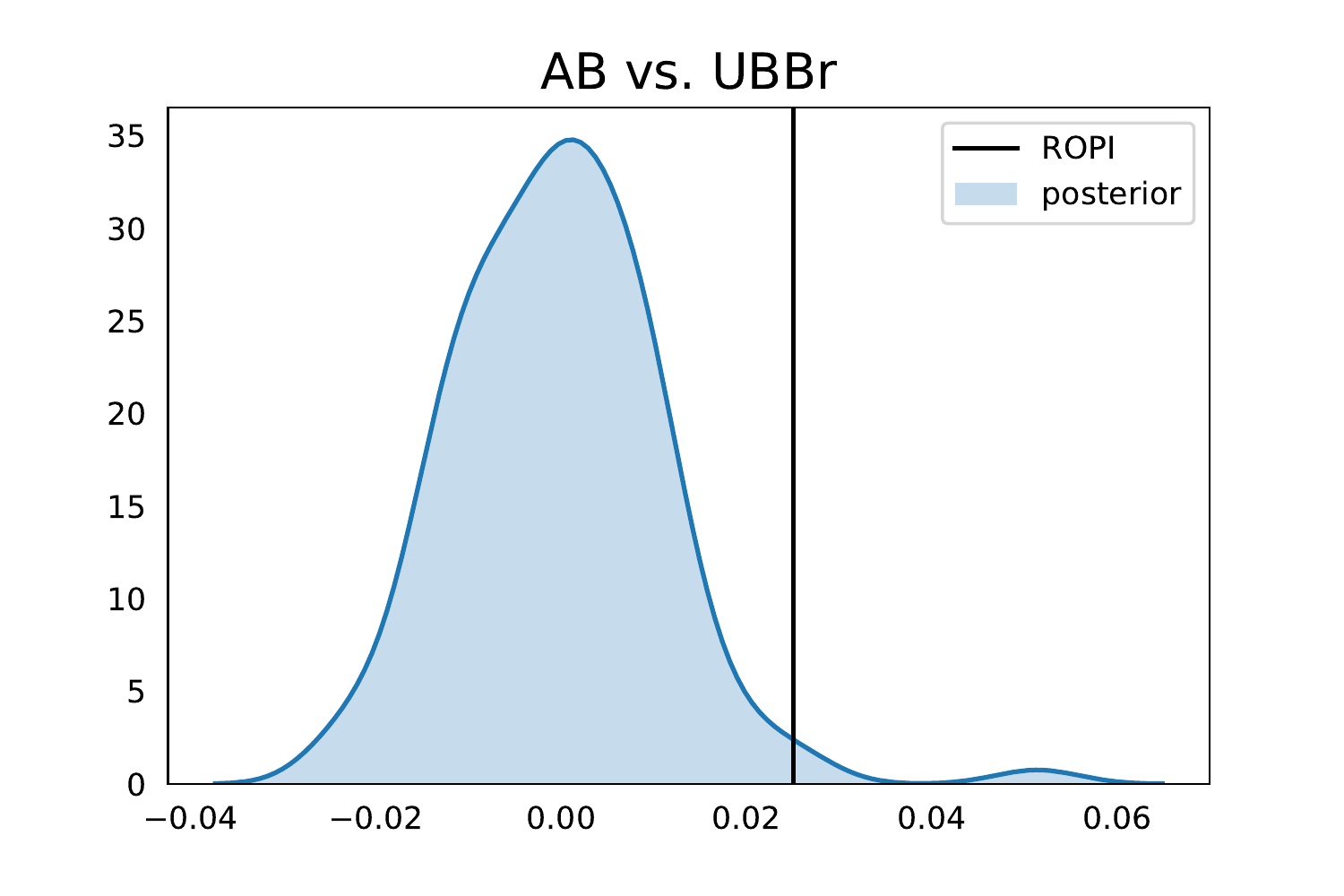}} &
    \includegraphics[width=.35\linewidth,trim={0.3cm 0.0cm 1cm 0},clip]{{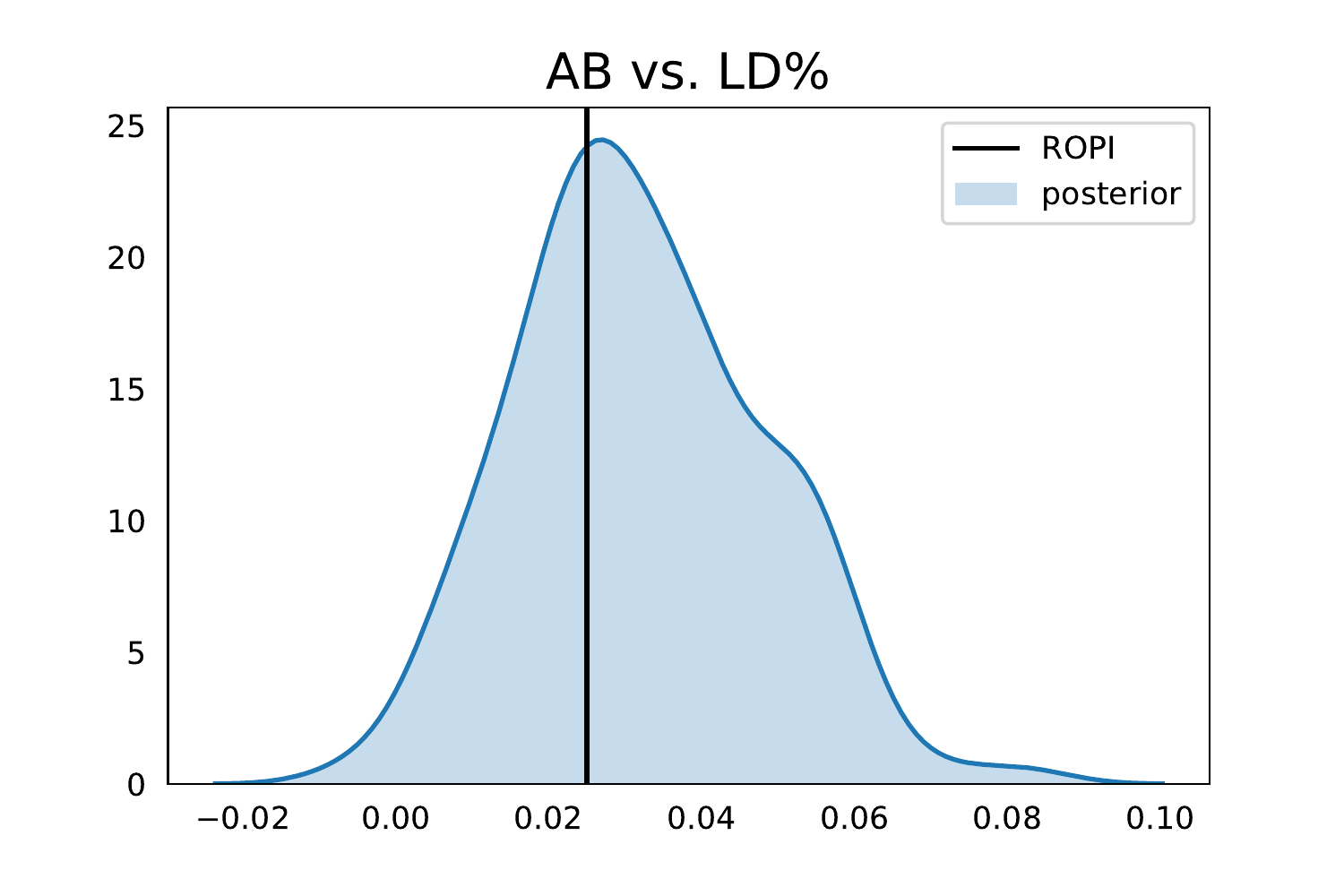}}\\
    \includegraphics[width=.35\linewidth,trim={0.1cm 0.0cm 0.7cm 0.0cm},clip]{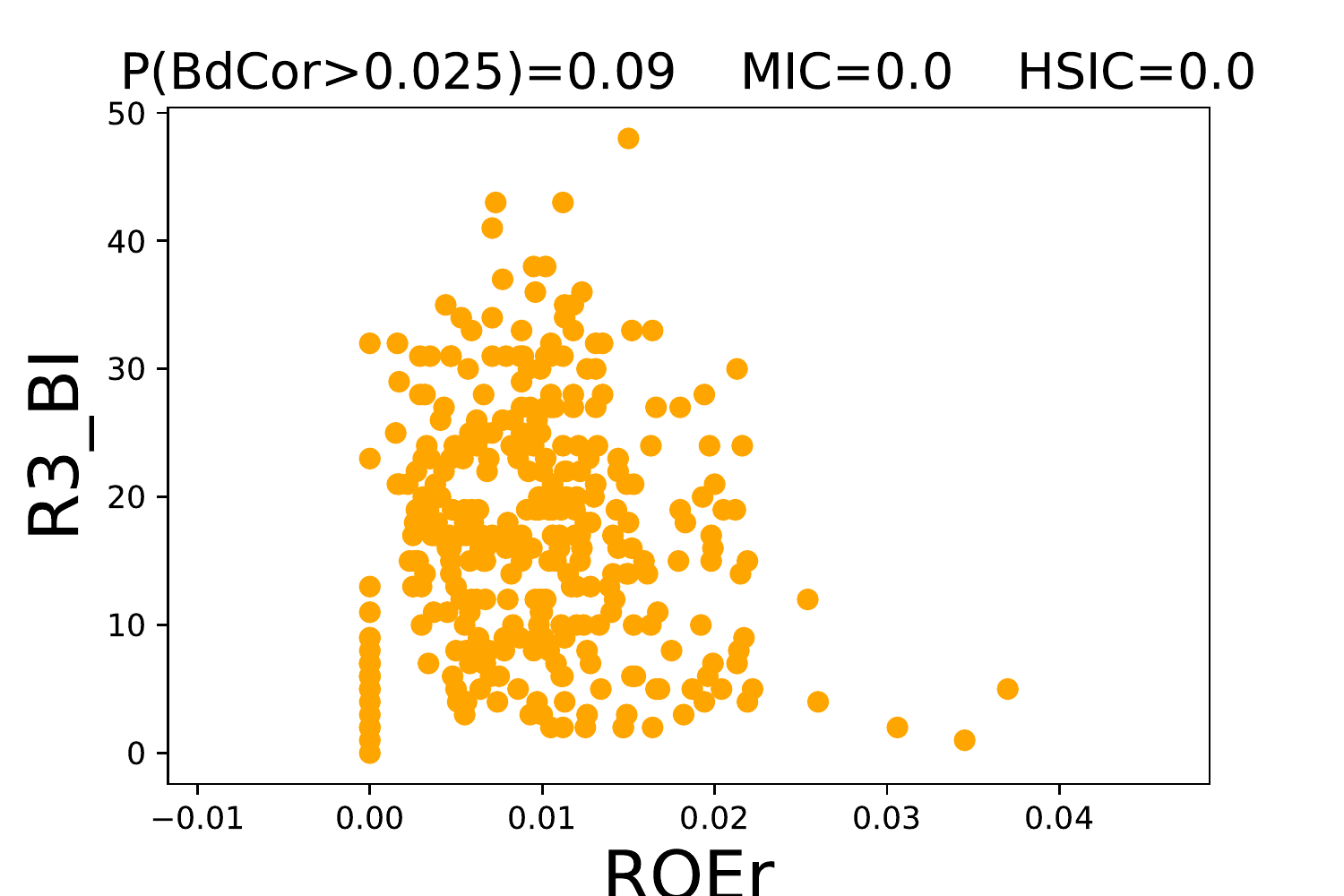} &
    \includegraphics[width=.35\linewidth,trim={0.0cm 0.0cm 0.7cm 0.0cm},clip]{{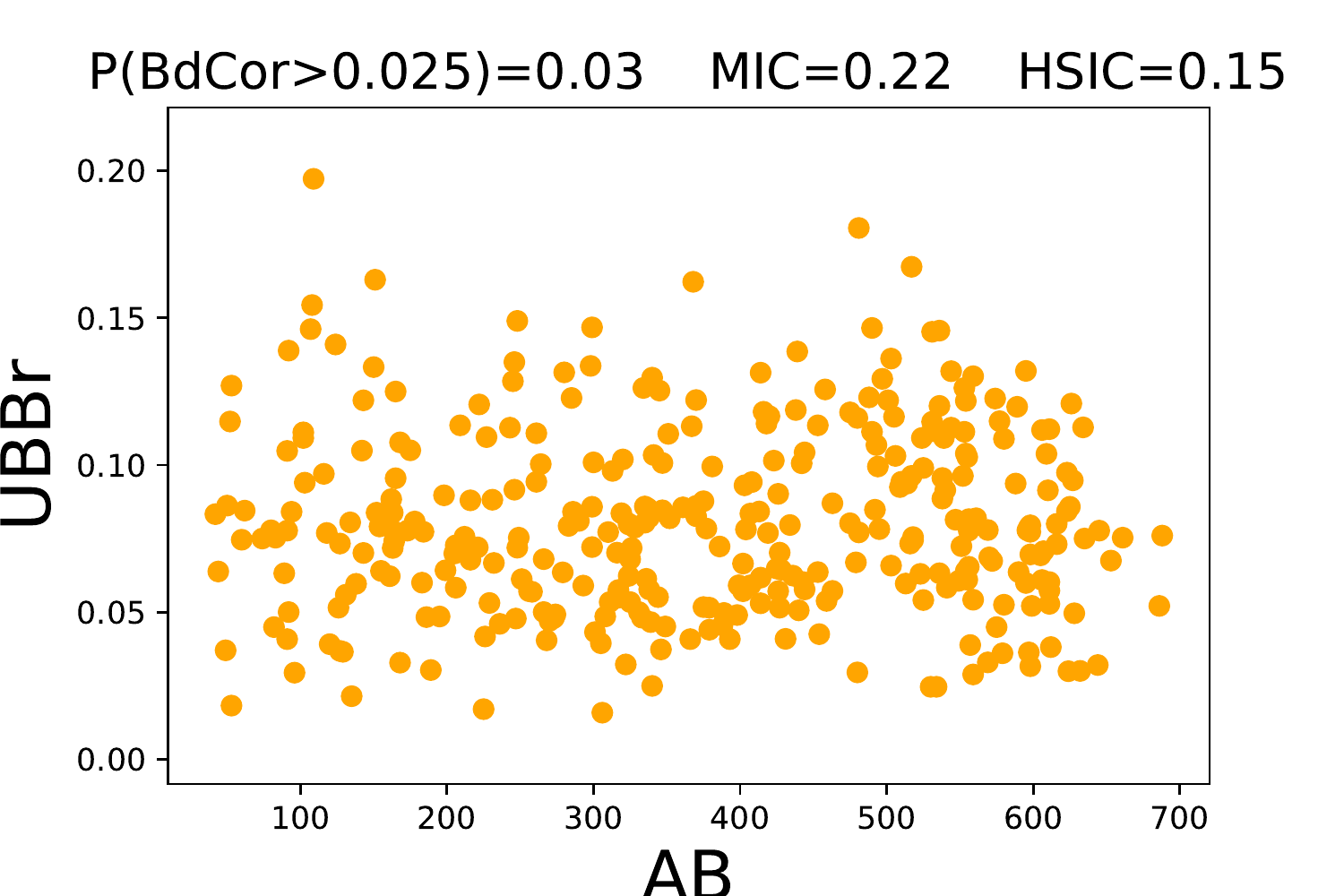}} &
    \includegraphics[width=.35\linewidth,trim={0.1cm 0.0cm 0.7cm 0.0cm},clip]{{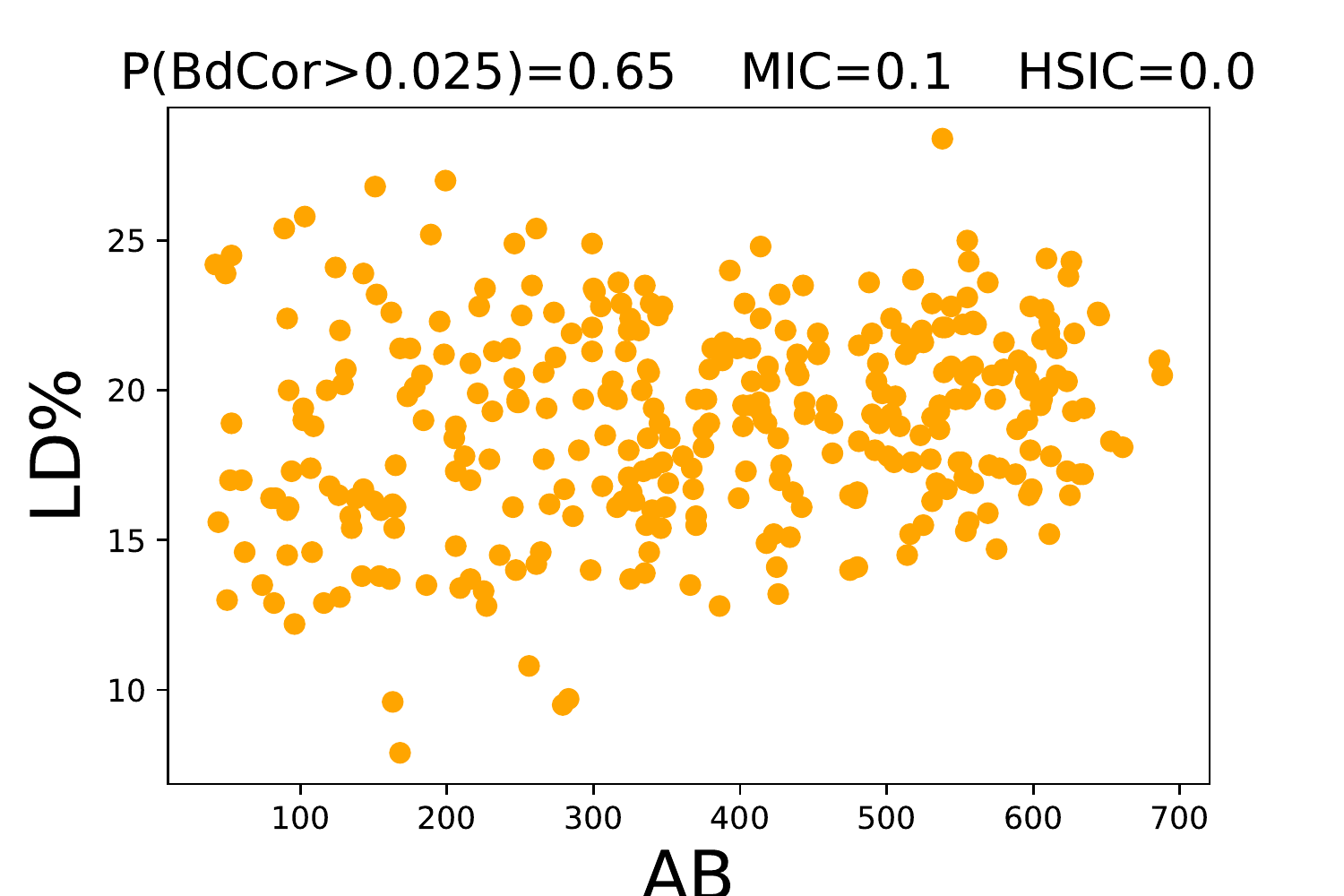}}\\
    \end{tabular}
\end{minipage}
\caption{{ Left: MLB heatmap overviews showing cases where HSIC and MIC cannot discriminate statistical and practical significance. Right: Scatter and density plots for three illustrative comparisons: (i) for ROEr vs. R3BI: BKR correctly declares practical independence, while both MIC and HSIC reject independence (from the scatter plot it is evident that the dependence is overall weak); (ii) for AB vs.\ UBBr, MIC and HSIC cannot make any decision, while BKR declares practical independence; (iii)
AB vs.\ LD\% MIC and BKR are undecided, while HSIC rejects dependence (scatter plot shows that variables are only weakly dependent).}}
\label{Fig:mic2}
\end{figure*}

The NHST based on MIC cannot convert association strength analysis in assessments of independence/weak dependence. Therefore, it is not able to distinguish between statistical and practical significance. This is illustrated in Figure~\ref{Fig:mic2}'s heatmaps where we compare the decisions of three methods.
NHSTs based on HSIC and MIC detect a large set of dependencies (red), but some of them are spurious. The white colour corresponds to statistically insignificant relationships ($\alpha = 0.05/8515$ with Bonferroni correction), which include relationships of independence that a NHST cannot detect. 
BKR is able to assess both dependence (red) and independence (blue). The scatter and density plots show the limits of NHSTs in  three comparisons.

\vspace{-0.2cm}
\paragraph{Gapminder.}
The Gapminder dataset includes 319 global developmental indicators (such as education, health, trade, poverty, population growth, and mortality rates) for 200 countries spanning over 5 centuries \cite{rosling2010gapminder}. We only consider the year 2002 and we compare three dependence detection techniques: BKR, HSIC (with Bonferroni correction), and another Bayesian approach called MI-Crosscat (the  mutual information computed using Crosscat).
Cross-Categorization (CrossCat) is a general Bayesian non-parametric method for learning the joint distribution over all variables in a mixed-type,  high-dimensional population. In particular, the learned joint distribution
can be used to perform a pairwise dependence/independence
test based on the mutual information 
\cite{pmlr-v54-saad17a}. The CrossCat prior induces sparsity over the dependencies; in BKR we can induce sparsity by increasing the width of the ROPI.

Figure \ref{Fig:gapmind0} shows the heatmap of the posterior mean $\mathcal{E}(\BdCor)$ computed using BKR.
\setlength{\columnsep}{5pt}
\begin{wrapfigure}{r}{0pt}
\centering
%\vspace{-85pt}
 \setlength\tabcolsep{0.0pt}
\centering
  \begin{tabular}{c}
    \includegraphics[width=.39\linewidth,trim={3cm 0.0cm 2.5cm 0},clip]{{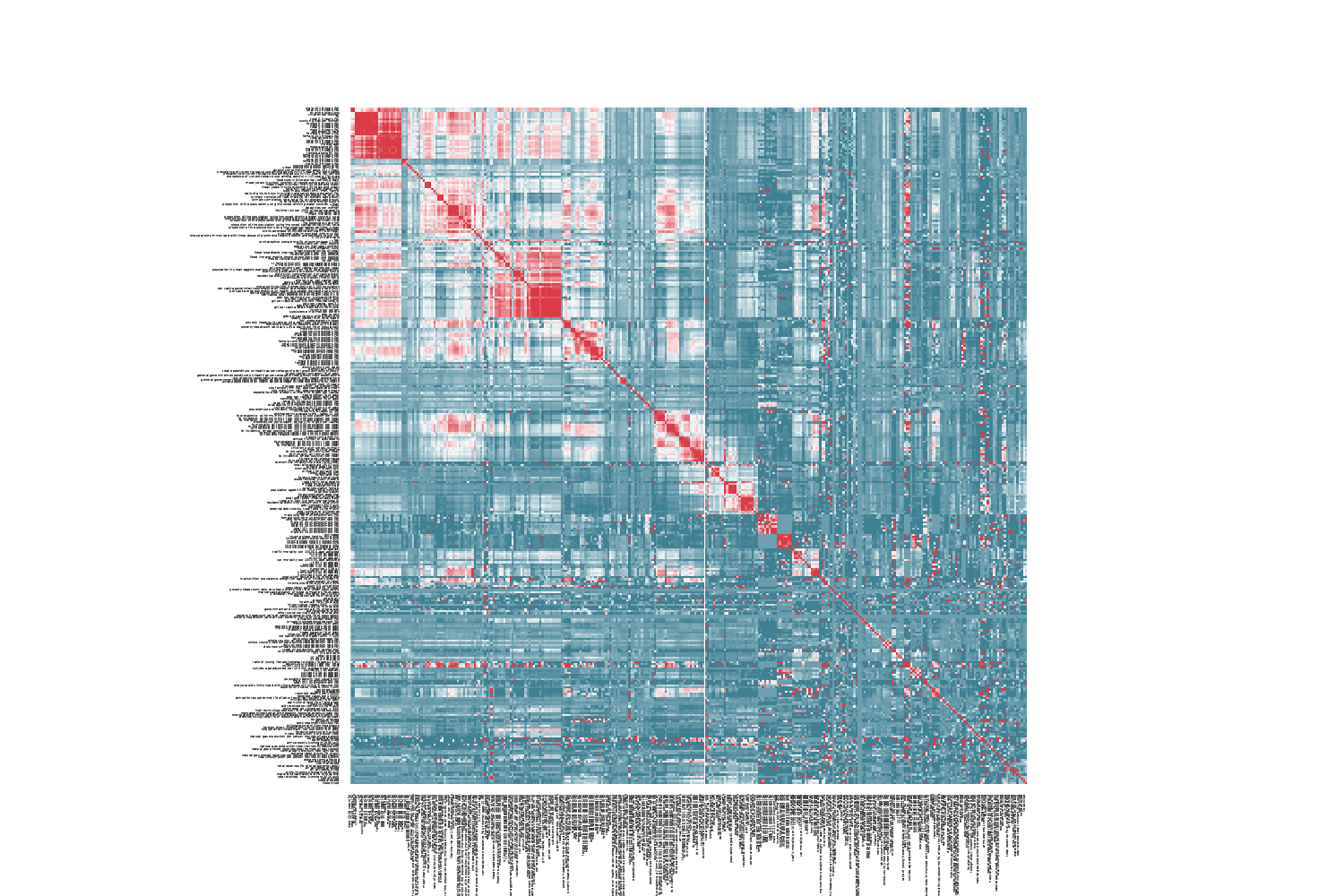}} \vspace{-0.4cm}
  \end{tabular}
  \caption{$\mathcal{E}(\BdCor)$.}
\label{Fig:gapmind0}
\vspace{-0.4cm}
\end{wrapfigure}
This allows us to immediately capture the strength of pairwise association (red means strong, blue means weak)
and to select a value for ROPI to induce sparsity. We select ROPI $0.15$ and plot the heatmap for the posterior probability of dependence (red) and independence (blue)
in Figure \ref{Fig:gapmind1}(top-row). 
\begin{figure}[htp]
% \begin{wrapfigure}{r}{0pt}
% \centering
% %\vspace{-85pt}
%  \setlength\tabcolsep{0.0pt}
  \begin{tabular}{ccc}
    \includegraphics[width=.31\linewidth,trim={3cm 0.0cm 2cm 0},clip]{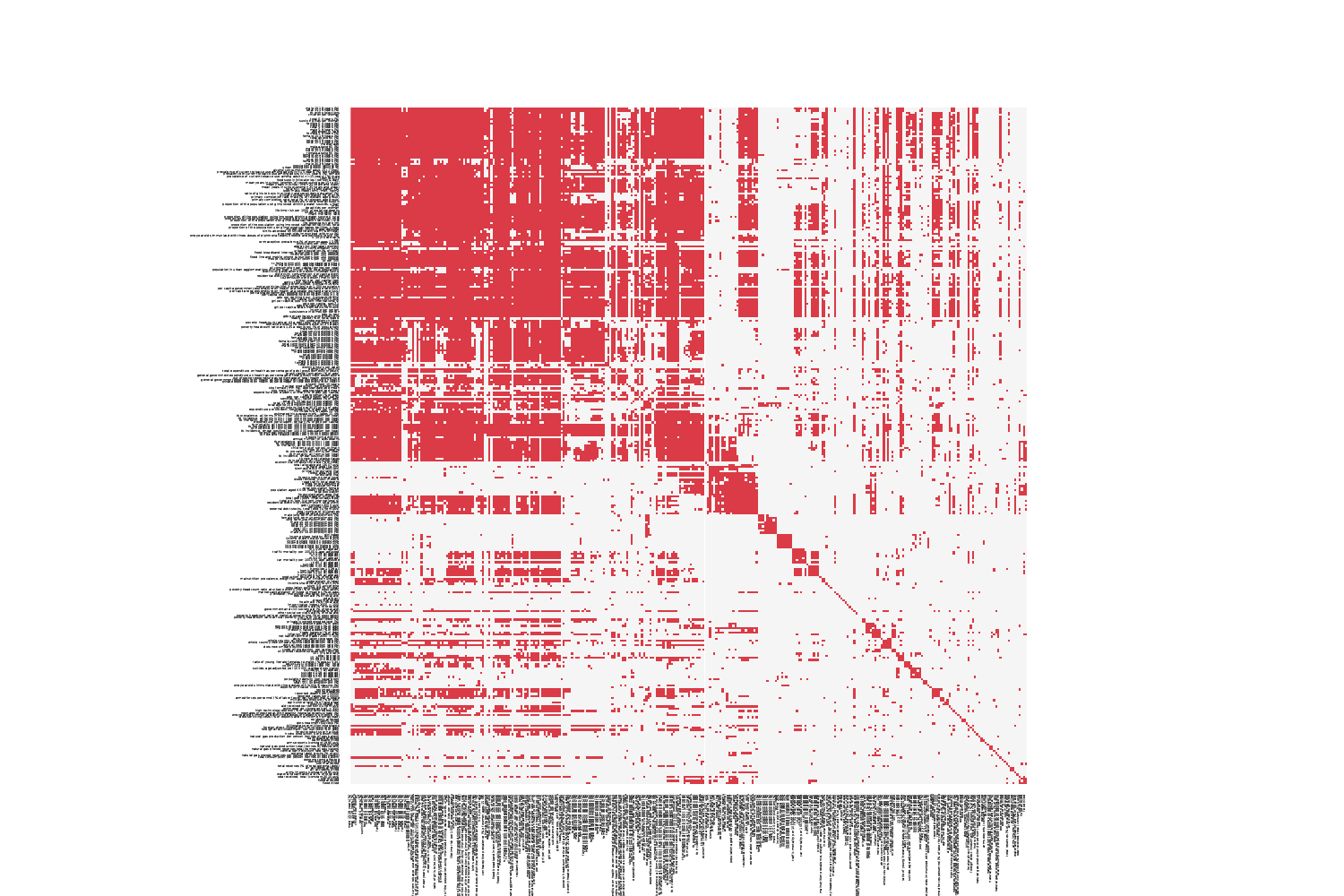} &
   \includegraphics[width=.31\linewidth,trim={3cm 0.0cm 2cm 0},clip]{{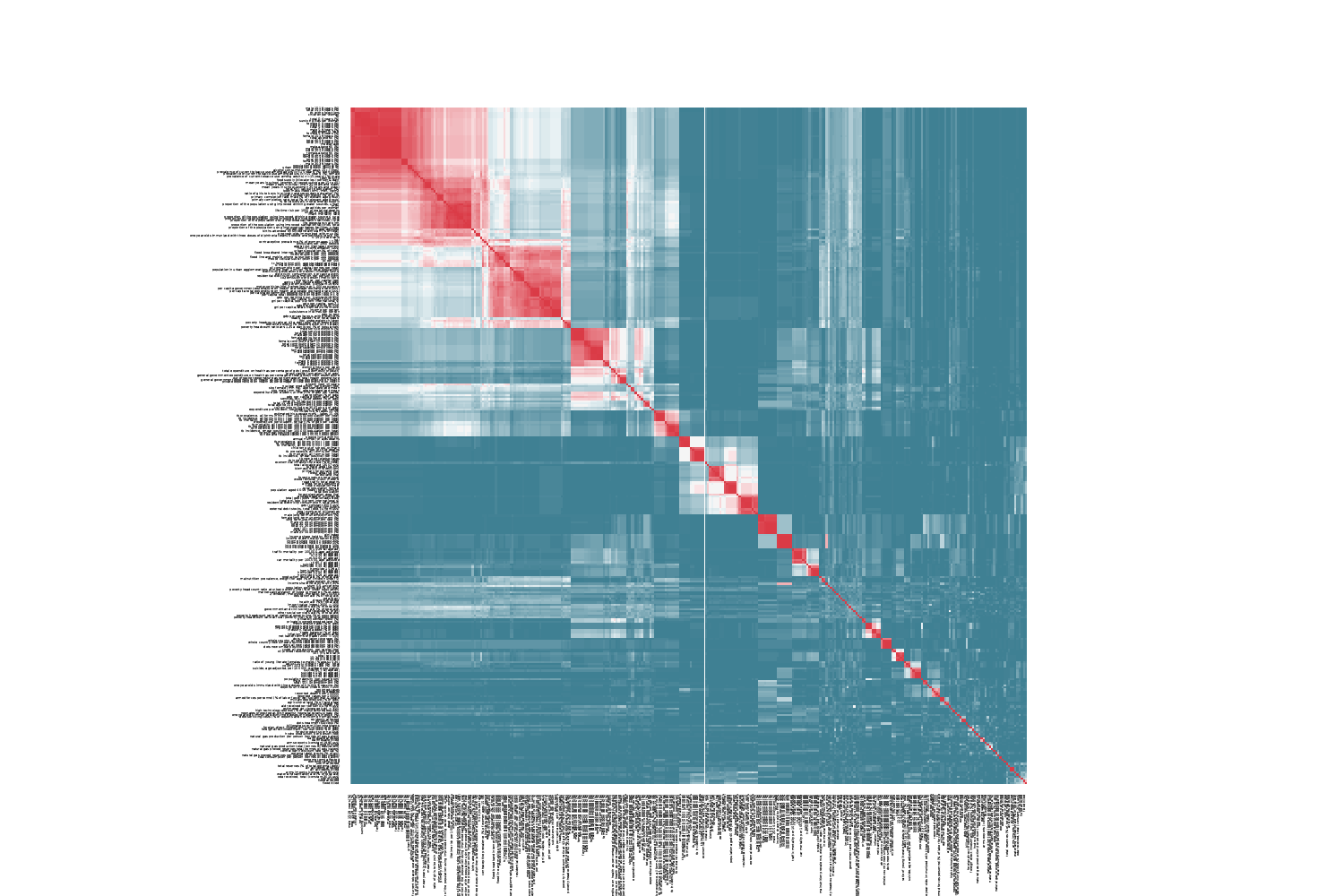}} &
    \includegraphics[width=.31\linewidth,trim={3cm 0.0cm 2cm 0},clip]{{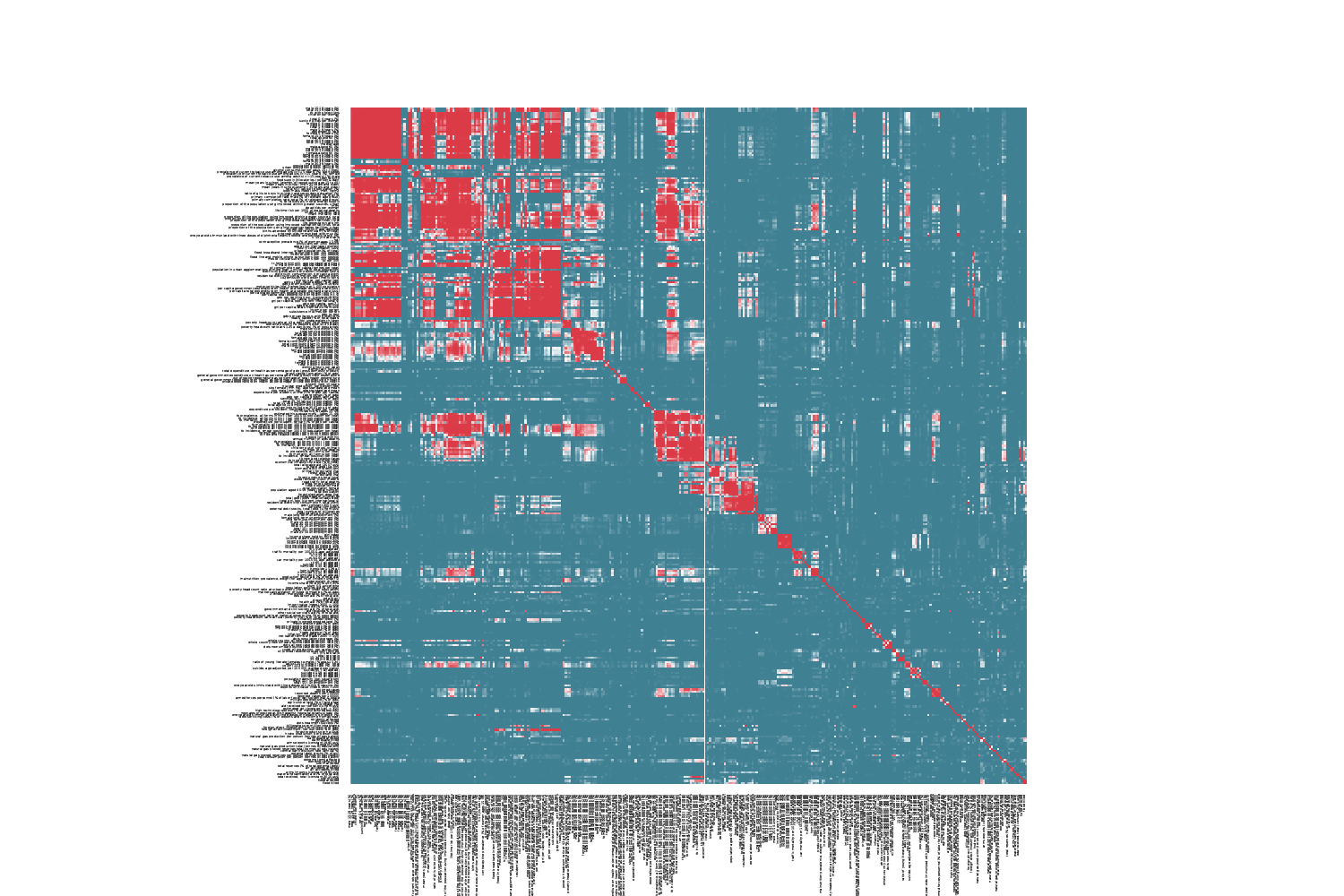}}\\
         \small HSIC & MI-Crosscat & \small BKR  
                 \end{tabular}  
          \begin{tabular}{ll}
         \includegraphics[width=.48\linewidth,trim={0.6cm 0.0cm 1.1cm 1.0cm},clip]{{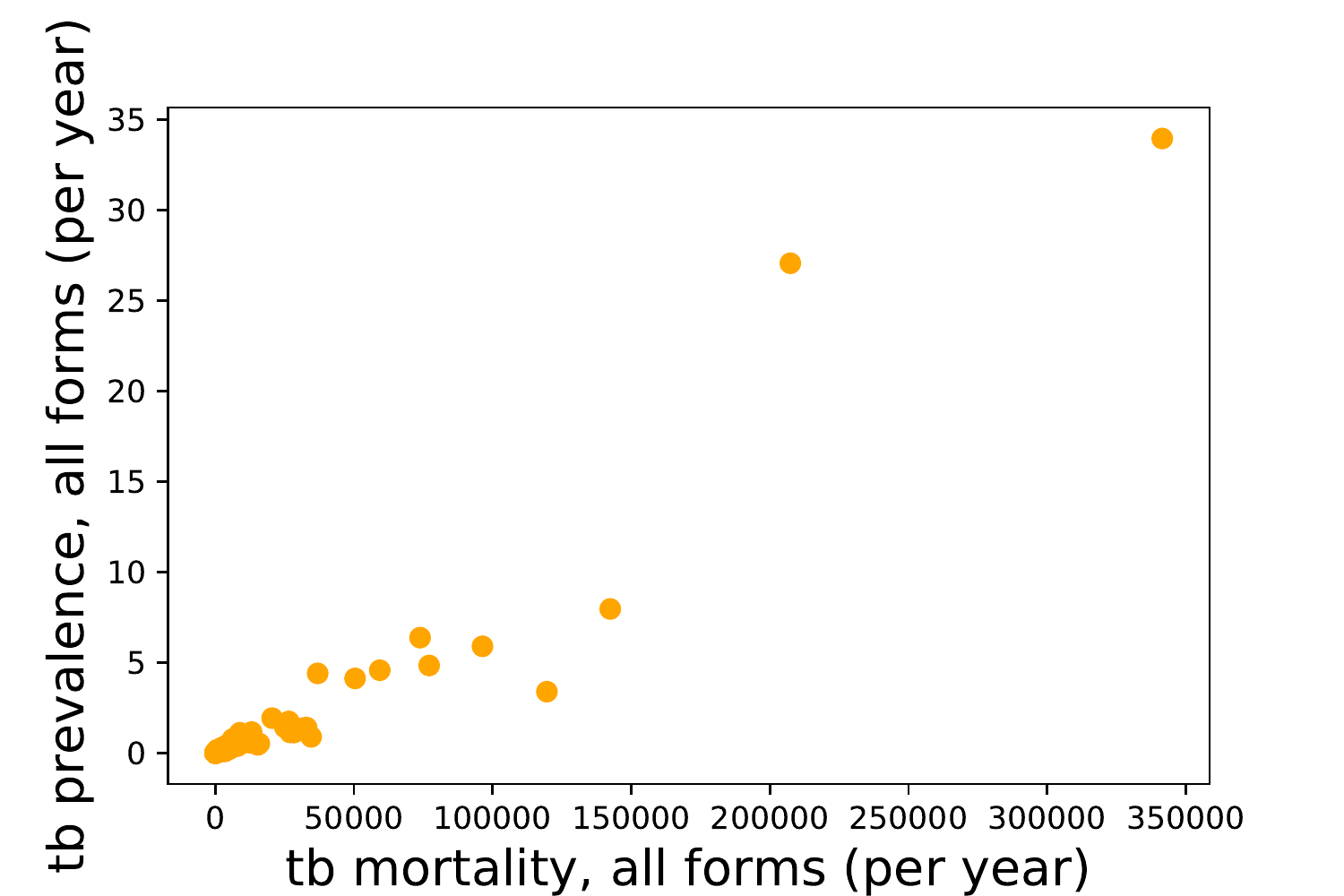}} &
    \includegraphics[width=.48\linewidth,trim={0.6cm 0.0cm 0.7cm 1.0cm},clip]{{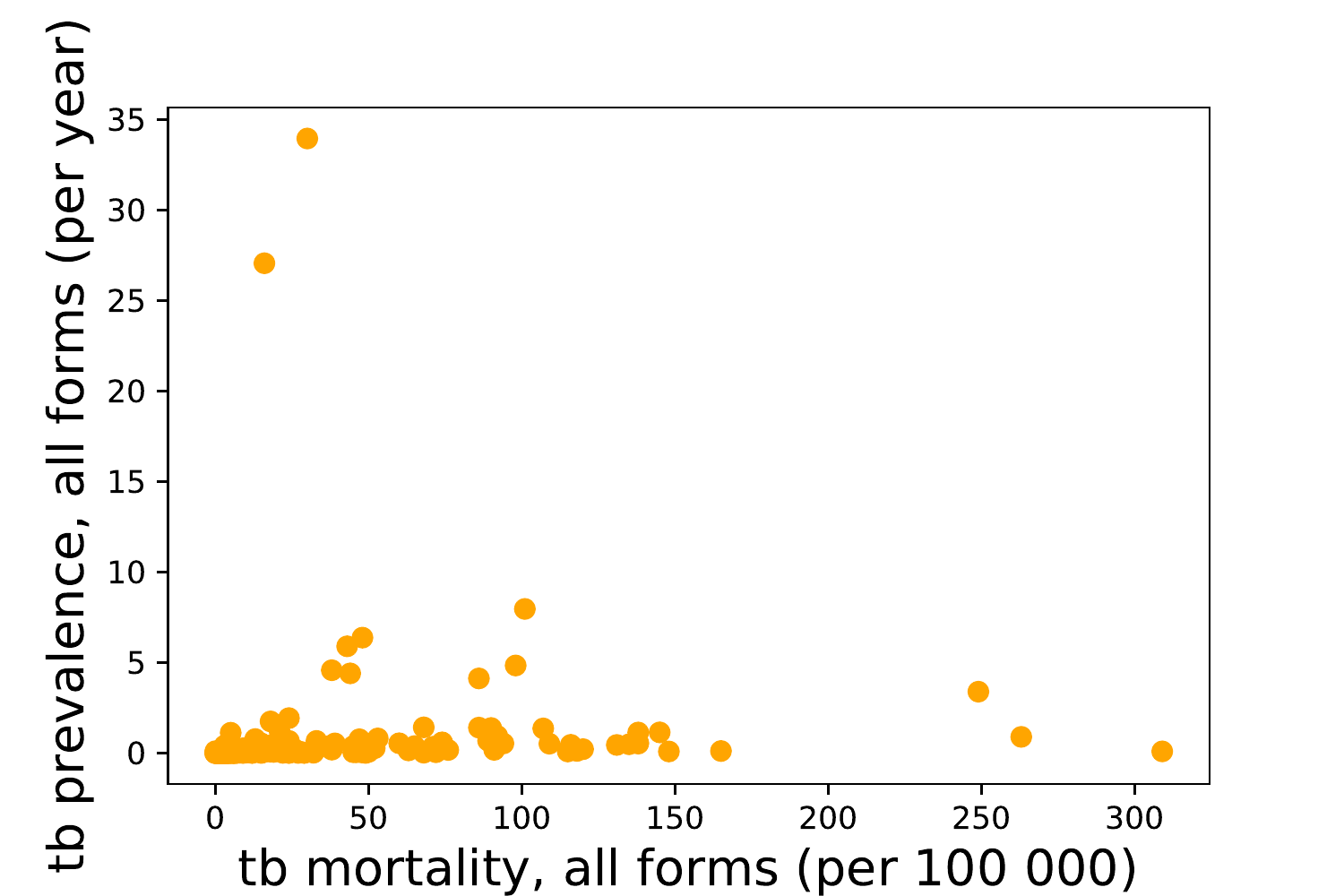}}
  \end{tabular}            
\caption{Gapminder dataset results.}
\label{Fig:gapmind1}
\end{figure}
Figure \ref{Fig:gapmind1}(top-row) shows the pairwise probability that the mutual information (MI), computed using Crosscat, exceeds zero, that is, $P[MI(X_i,X_j ) > 0]$. Note that the metric $P[MI(X_i,X_j ) > 0]$ only indicates the existence of a predictive relationship between $X_i$ and $X_j$, but it does not quantify the strength of such relationship \cite{pmlr-v54-saad17a}.
By comparing MI-Crosscat and BKR it is evident that both Bayesian methods provide a closer insight while HSIC can only reject independence. There are also evident differences: Crosscat is not able to detect some dependences due to outliers. Consider for instance the variables A=\textit{'tb mortality, all forms (per 100 000 population)'}, B=\textit{'tb prevalence, all forms (per year)'} and C=\textit{'tb mortality, all forms (per year)'}: the scatter plots for (A,B) and (B,C) are showed in Figure \ref{Fig:gapmind1}.
Variables are pairwise dependent and
BKR assigns probability $1$ to both the comparisons, while Crosscat assigns probability $1$ for (B,C) and probability $0.04$ for (A,B) due to outliers. More examples are reported in the Supp.\ Material. Crosscat is a more complex and so more computationally costly model than the one used to derive BKR, but BKR offers an effective and fast way to compute dependence/independence for mixed-type variables.

% 
% 
% Figure \ref{Fig:gapmind}(left) shows the pairwise NHST of independence using
% HSIC, which detects a dense set of dependencies (red)
% including many spurious relationships (weak dependence).
% The white colour corresponds to statistically insignificant relationships ($\alpha = 0.05/50721$ with Bonferroni correction for multiple comparisons), which  include relationships of independence that a NHST cannot assess. 
% 
% 
% \begin{figure*}
% \centering
%   \begin{tabular}{c  c  c }
%     \includegraphics[width=.25\linewidth,trim={3cm 0.0cm 2cm 0},clip]{HSIC_decisions.pdf} &
%     \includegraphics[width=.25\linewidth,trim={3cm 0.0cm 2cm 0},clip]{{crosscat.pdf}} &
%     \includegraphics[width=.25\linewidth,trim={3cm 0.0cm 2cm 0},clip]{{BKR_gapminder_prob_sparse2.pdf}}\\
%     \small HSIC & \small MI-Crosscat & \small BKR\\
%   \end{tabular}
% \caption{Gapminder dataset}
% \label{Fig:gapmind}
% \end{figure*}
\vspace{-0.2cm}
\paragraph{Predicting classifier performance.}
We have shown that $\BdCor$
is a commensurable measure of association by comparing it with MIC and MI-Crosscat. In machine learning, dependence measures are often used to assess what features are relevant for a certain prediction task. Here we use $\BdCor$
to quantify the association between all features of a dataset together ($X$) and the class variable $Y$ in 84 datasets (details in the Supp. Material) from the
Penn Machine Learning Benchmarks~\citep{Olson2017PMLB}. We have used the Nystr\"{o}m low-rank approximation with a RBF kernel for $X$ and a indicator kernel for $Y$.
For each dataset, we have also trained a logistic classifier and computed the  averaged 10-fold cross-validation log-loss. Figure \ref{fig:classif} 
shows the regression plot of $\BdCor$ vs.\ averaged log-loss:  larger values for $\mathcal{E}(\BdCor)$ correspond to lower values for log-loss. This again confirms that  $\mathcal{E}(\BdCor)$ is an effective measure of association.
\begin{figure}
\centering
 \includegraphics[width=5.0cm,trim={0.cm 0.0cm 0.2cm 0.0cm},clip]{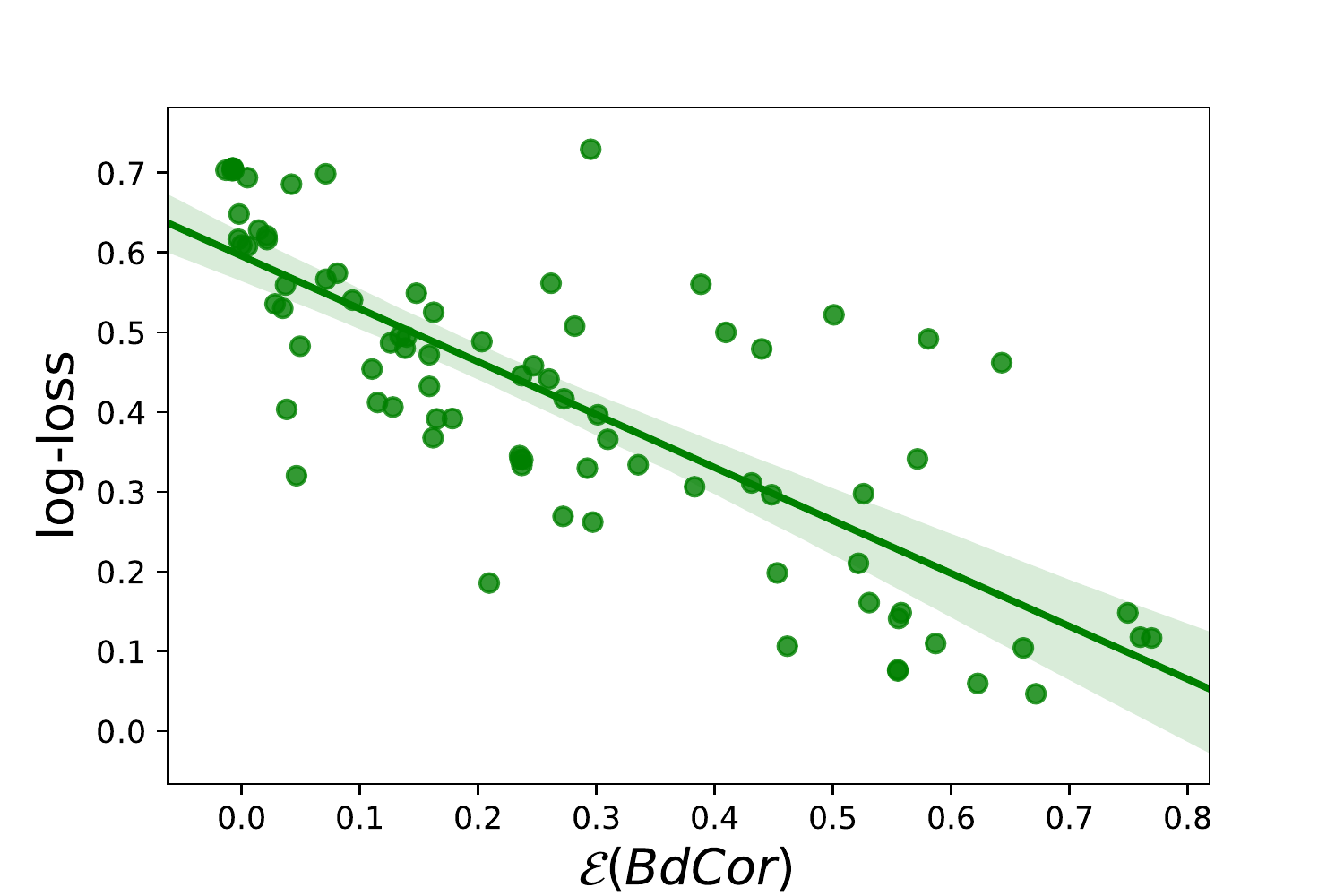}
 \caption{Predicting classifiers performance}
 \label{fig:classif}
\end{figure}

\vspace{-0.2cm}
\paragraph{DGT-Translation Memory.}
DGT-TM is a translation memory (sentences and their manually produced translations) in 24 languages. It contains segments of all the treaties, regulations, and directives adopted by the European Union.
 The documents in different languages are aligned in accordance with well-defined segmentation rules. 
Figure \ref{Fig:EU} shows an example of the English and Italian versions of Document 52015BP0930; the rectangles bound the segments.
\begin{figure}[htp]
% \begin{wrapfigure}{r}{0pt}
 \centering
% %\vspace{-85pt}
%  \setlength\tabcolsep{0.0pt}
  \begin{tabular}{cc}
    \includegraphics[width=.4\linewidth,trim={2.cm 0.0cm 1.8cm 3cm},clip]{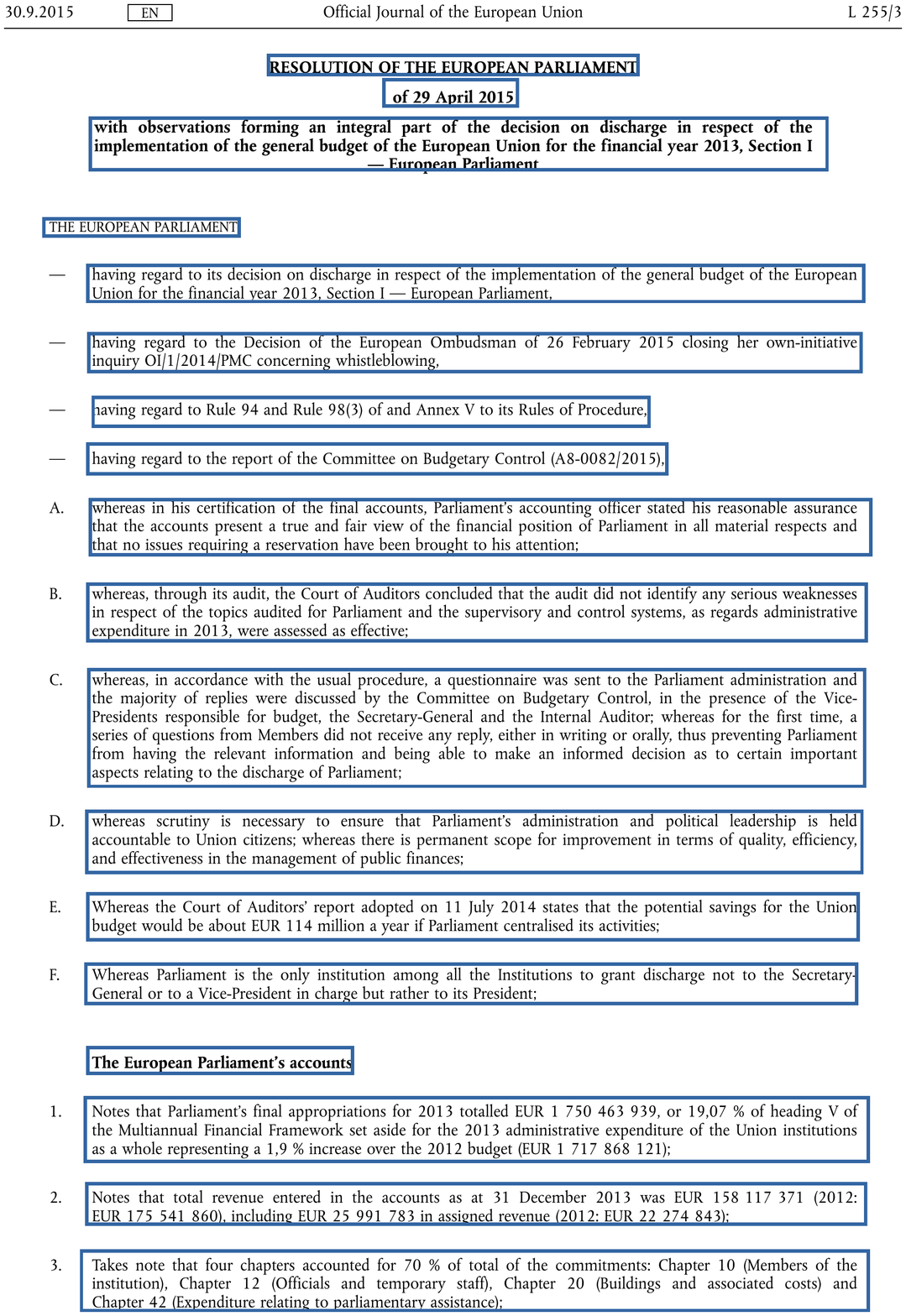} &
   \includegraphics[width=.4\linewidth,trim={2.cm 0.0cm 1.8cm 3cm},clip]{{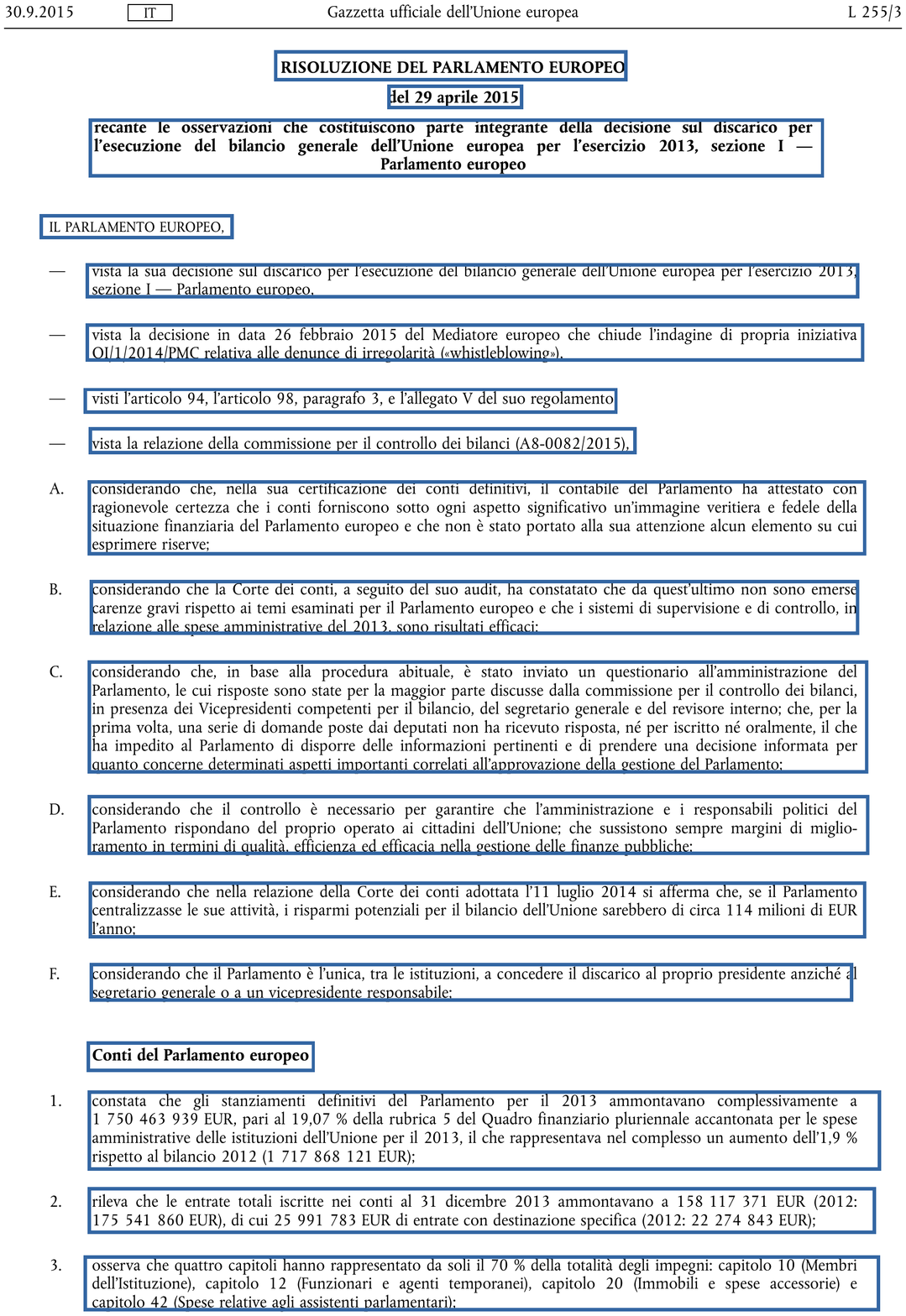}} \\
         \small EN & \small IT 
                 \end{tabular}  
          \begin{tabular}{c}
         \includegraphics[width=.93\linewidth]{{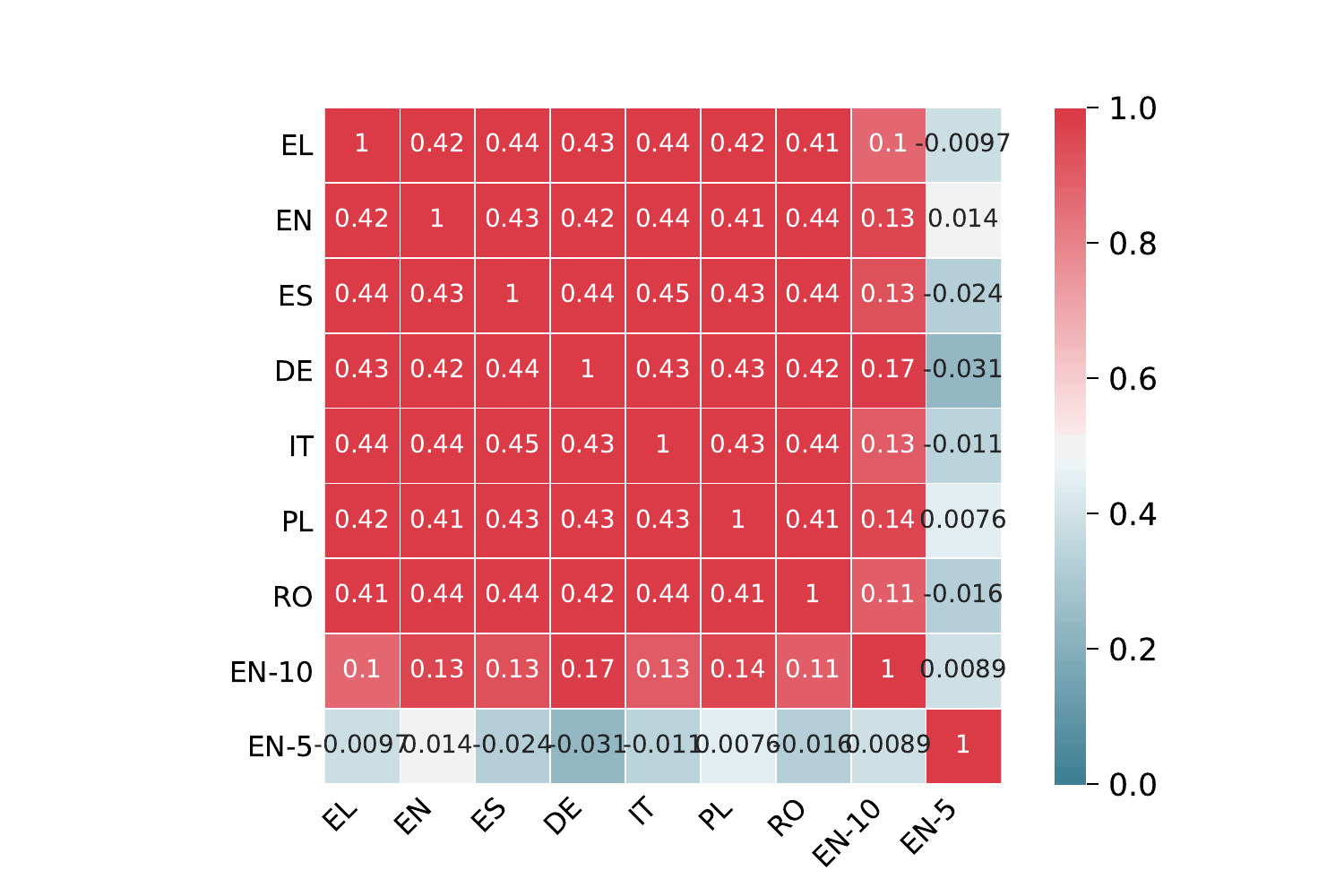}} 
\vspace{-0.2cm}
  \end{tabular}        
\caption{DGT-Translation Memory results. Two document examples (on top) and heatmap among languages.}
\label{Fig:EU}
\end{figure}
We apply BKR to test the association between 7 different translations ('EL','EN','ES','DE','IT','PL','RO') of the above document by considering the 18 segments in Figure~\ref{Fig:EU}. We have also generated two additional documents 
in English obtained by randomly selecting 10 ('EN-10')
and 5 ('EN-5') words in each segment.
We have used the Edit-distance-based RBF kernel.\footnote{Edit-distance is a way of quantifying how dissimilar two strings are to one another by counting the minimum number of operations required to transform one string into the other. We set the kernel length-scale  to the median distance.} Figure \ref{Fig:EU}(bottom)
reports a summary of the results for all pairwise comparisons performed using BKR; each number in the heatmap cell is the posterior mean of $\BdCor$, and the color of the cell is the posterior probability of practical dependence $p(\BdCor > 0.025|Data)$.
BKR is able to detect the strong association among the documents. Note that Spanish ('ES') and Italian ('IT') have the largest value of $\BdCor$ (as expected, because of their similarity). 
BKR can detect dependence even using only 10 words per segment ('EN-10'). With 5 words per segment, BKR (understandably) declares weak dependence with higher probability.

\section{Conclusions}
We proposed a novel Bayesian kernelised association measure by using a kernel embedding of probability measures
based on a Dirichlet Process prior. 
By using this setting, we  derived a new Bayesian kernel test of independence/dependence.
This test overcomes the limitations of null hypothesis significance testing (NHST) and, in particular, allows the data analyst to ``accept'' the null hypothesis and not only to reject it.
% We  provided an algorithm to compute the posterior distribution and the posterior probabilities
% of dependence and independence. 
We illustrated the effectiveness of our approach by first comparing  this new test with
the kernel independence test developed by~\cite{gretton2008kernel}, with  a test based on the maximal information coefficient, and with a Bayesian test of independence based on the Crosscat model.
% The goal was to discover independences/dependences in simulated data
% as well as in two high-dimensional data sets including mixed variables (continuous and categorical).
We showed that our test is able to declare dependence and independence
and that, for this reason, it takes more decisions than NHSTs do.
An important advantage of kernel tests is that they can be applied on any (un)structured data, being for instance able to detect the dependence
of passages of text and their translation or the dependence between images and text. Such flexibility opens doors for many possible usages.
% 
% We also used this test to learn sum-product networks using real data
% sets from the UCI Machine Learning repository. Results show
% superiority of the new approach with respect to other independence
% tests, and further investigations are required to consolidated these
% empirical evaluations.
% 
% During the design of the procedures in this paper, we used the sample mean to learn the hyper-parameters of the kernels.
% As future work, we plan to develop a fully Bayesian approach by also
% learning from data those hyper-parameters through the use
% of a Dirichlet process mixture model. We also plan to develop the Bayesian version of the two-sample kernel test and the kernel
% conditional independence test.
% 
% To the best of our knowledge, the only other attempt to define a Bayesian formulation of kernel independence
% tests is given by~\cite{Flaxman17}. There the authors assume a Gaussian Process prior over the RKHS containing the mean embedding and combine it with
% a conjugate likelihood function, with the goal of estimating the hyper-parameters of the kernel functions. Instead, as we will see, we directly place
% a prior on the unknown quantity of interest: the probability measure embedded in the RKHS.

\bibliography{biblio}

\bibliographystyle{plain}

\newpage
\appendix

\section{Supplementary Material}

\subsection{Experimental method for the \textit{House} dataset}
In this section we outline the methodology used to
produce the pairwise heatmaps shown
in Figure \ref{Fig:1}.  
The location counts for the \textit{House} dataset are quite unbalanced (some locations have only 1 or 2 houses), so we removed all locations with less than 10 houses per location. For the remaining 462 entries, we  resized
the four images (frontal, kitchen, bathroom, bedroom)
to $32\times 32$ pixels, converted the figures in gray-scale and tiled the four images together so that the first
image goes in the top-right corner, the second image in the
top-left corner, the third image in the bottom-right corner, and the final image in the bottom-left corner
(we used an alphabetical order based on the filename).
The final image is $64\times 64$. 
For BKR, we used a square exponential kernel with kernel length-scale set to the median distance
between points in input space for the variables
number of bedrooms, number of bathrooms, area, image and image-size, and we used the indicator function for the categorical variable location.
We selected ROPI=0.025 and we sampled 1000 samples
from the posterior in order to compute the quantities of interest and to plot the posterior densities.
\subsection{Proofs}~\\
We briefly sketch the proofs of our results.
We start by proving this lemma.
\begin{lemma}
\label{lem:lem0}
Let $X,X'\sim P$ and $P$  have distribution $w_0 P_0 + \sum_{i=1}^n w_i \delta_{X_i}$ with $W=[w_0,w_1,\dots,w_n]^T\sim Dir(s,1,\dots,1)$ and $P_0\sim Dp(s,\nu^*)$, then 
 \begin{align}
   \nonumber
  &E_{X,X'}(K(X,X'))= w_0^2\int K(x,x') d\left(P_0(x)P_0(x')\right) \\
  \label{eq:exx}
      &+ 2\sum_{i=1}^n w_0 w_i \int K(x,X_i) d P_0(x)+      \sum_{i=1}^n\sum_{j=1}^n w_iw_j K(X_i,X_j)\\
  \label{eq:exx1}
&=W \mathbb{K}^{XX} W^T
 \end{align}
  where $\mathbb{K}^{XX}$ is a symmetric $(n+1)\times (n+1)$ matrix such that $\mathbb{K}^{XX}_{00}=\int K(x,x') d(P_0(x)P_0(x'))$,
  $\mathbb{K}^{XX}_{0i}=\int K(x,X_i) d P_0(x)$ for $i>0$ and $\mathbb{K}^{XX}_{ij}=K(X_i,X_j) $ for $i=1,\dots,n$ and $j\geq i$.
 \end{lemma}
 \begin{proof}
  Since we have assumed that $X,X'\sim P$, then we have that
  $$
  E_{X,X'}(K(X,X'))= \int K(x,x') d\left(P(x)P(x')\right).
  $$
  By exploiting  $P=w_0 P_0 + \sum_{i=1}^n w_i \delta_{X_i}$,  we obtain \eqref{eq:exx}.
  The last equality \eqref{eq:exx1}  follows from vector algebra.  
 \end{proof}
% Theorem \ref{th:HSIC} then follows  straightforwardly from Lemma \ref{lem:lem0}
% and Proposition \ref{prop:HSIC_as_expectations}.

\paragraph{Proof of Theorem \ref{th:HSIC}}~\\
Theorem \ref{th:HSIC}  can be obtained by Proposition \ref{prop:HSIC_as_expectations}  and Lemma  \ref{lem:lem0}, by
noticing that, in this case,  the observations are paired $Z=[X,Y]$.
Under the hypothesis $X\nindep Y$, we assume that $X,Y$ are jointly distributed with posterior distribution
$w_0 P_0 + \sum_{i=1}^n w_i \delta_{Z_i}$,  with $W=(w_0,w_1,\dots,w_n)\sim Dir(s,1,\dots,1)$ and $P_0\sim Dp(s,\nu^*)$ \qed.

An empirical  estimate of the HSIC statistics from i.i.d.\ samples  $(x_1,y_1),\dots,(x_n,y_n)$ on $\mathcal{X}\times\mathcal{Y}$ can be obtained by replacing the  expectation operators in \eqref{eq:HSIC_as_expectations} with the sample mean \cite{gretton2008kernel} and it is given by
\begin{align}
\nonumber
\text{HSIC}_{obs}(X,Y)&= \frac{1}{n^2} \sum_{i,j=1}^n k_\mathcal{X}(X_i,X_j)k_\mathcal{Y}(Y_i,Y_j)\\
\nonumber
&+\frac{1}{n^4} \sum_{i,j,q,r=1}^n k_\mathcal{X}(X_i,X_j)k_\mathcal{Y}(Y_q,Y_r)\\
\label{eq:HSIC_as_expectationsemp}
&-\frac{2}{n^3} \sum_{i,j,q=1}^n k_\mathcal{X}(X_i,X_j)k_\mathcal{Y}(Y_i,Y_q).
\end{align}
 We now establish a connection of that with our Bayesian estimator for HSIC.
 
\paragraph{Proof of Theorem \ref{th:postmean}}~\\
We exploit the fact that for large $n$, we have that $P \approx  \sum_{i=1}^n w_i \delta_{Z_i}$.
The effect of the prior measure vanishes at the increase of $n$.
Hence, we have that
\begin{align}
\nonumber
\widehat{\text{HSIC}}(X,Y)&= \sum_{i,j=1}^n k_\mathcal{X}(X_i,X_j)k_\mathcal{Y}(Y_i,Y_j)w_iw_j\\
\nonumber
&+\sum_{i,j,q,r=1}^n k_\mathcal{X}(X_i,X_j)k_\mathcal{Y}(Y_q,Y_r)w_iw_jw_qw_r\\
\label{eq:HSIC_as_expectationsempx}
&-2 \sum_{i,j,q=1}^n k_\mathcal{X}(X_i,X_j)k_\mathcal{Y}(Y_i,Y_q)w_iw_jw_q.
\end{align}
The moments of the Dirichlet distributions are
\begin{align}
\nonumber
&\mathcal{E}\left(w_iw_j\right)=\dfrac{1+I_{i=j}}{n(n+1)},\\
\nonumber
&\mathcal{E}\left(w_iw_jw_q\right)=\dfrac{1+I_{i=j \vee i=q \vee j=q}+4 I_{i=j=q}}{n(n+1)(n+2)},\\
\nonumber
&\mathcal{E}\left(w_iw_jw_r w_q\right)=\dfrac{1+I_{i=j \vee i=q \vee j=r \vee j=r \vee j=q \vee r=q}}{n(n+1)(n+2)(n+3)}\\
&+\dfrac{4 I_{i=j=r \vee i=j=q \vee i=r=q \vee j=r=q }+18 I_{i=j=r=q}}{n(n+1)(n+2)(n+3)}.
\end{align}
For large $n$, \eqref{eq:HSIC_as_expectationsempx} converges to \eqref{eq:HSIC_as_expectationsemp}.
We can recover a similar result for the variance
of $\widehat{\text{HSIC}}(X,Y)$ and then use 
\cite[Th.1 and Th.2]{gretton2008kernel} to prove asymptotic consistency.\qed

\paragraph{Proof of Corollary \ref{co:fast}}~\\
We exploit the following property of the Schur product:
$$
W(K^{XX}W^T\circ K^{YY}W^T)=Tr(\text{diag}(W) K^{XX}W^T WK^{YY}).
$$
Hence, note that for $ R=\text{diag}(W)-W^TW$,
 \begin{align*}
Tr(K^{XX}RK^{YY}R)&=Tr(K^{XX}\text{diag}(W)K^{YY}\text{diag}(W))\\
&-2Tr(K^{XX}\text{diag}(W)K^{YY}W^TW)\\
&+Tr(K^{XX}W^TWK^{YY}W^TW)\\
&=W (K^{XX}\circ K^{YY})W^T\\
&-2Tr(K^{XX}\text{diag}(W)K^{YY}W^TW)\\
&+WK^{XX}W^TWK^{YY}W^T\\
&=W (K^{XX}\circ K^{YY})W^T\\
&-2W(K^{XX}W^T\circ K^{YY}W^T)\\
&+WK^{XX}W^TWK^{YY}W^T\, .
\end{align*}
\qed 
\paragraph{Proof of Theorem \ref{th:low-rank}}~\\
We provide the proof for the Nystr\"{o}m method, while the idea is similar for random Fourier features. 
The low-rank approximation approach provided by the Nystr\"{o}m method is achieved by randomly sampling $m$ data points (i.e. inducing variables) from the given $n$ samples and computing the approximate kernel matrix $\tilde{K}$:
$$
\begin{aligned}
\tilde{K}&=K_{n,m}K_{m,m}^{-1}K_{m,n}=(K_{n,m}K_{m,m}^{-0.5})(K_{n,m}K_{m,m}^{-0.5})^T \\
&=\tilde{\phi}\tilde{\phi}^T
\end{aligned}
$$
that provides a  feature  representation for $\tilde{K}$.
Hence, we have that
$$
\begin{aligned}
Tr(K^{XX}RK^{YY}R) &\approx Tr(\tilde{K}^{XX}R\tilde{K}^{YY}R) \\
&=Tr(\tilde{\phi}_{X}\tilde{\phi}_{X}^T R \tilde{\phi}_{Y}\tilde{\phi}_{Y}^TR)\\
&=||\tilde{\phi}_{X}^T R \tilde{\phi}_{Y}||_F^2.
\end{aligned}
$$

\subsection{Pseudo-code for BKR}
We provide an algorithm to perform the Bayesian non-parametric kernel test of independence ($N_{mc}$ denotes the number of Monte Carlos samples):

\begin{enumerate}
 \item Initialise the counter $\tau$ to $0$ and the array $V$ to empty;
 \item For $i=1,\dots,N_{mc}$ 
 \begin{enumerate}
  \item Sample $(\omega_0,\omega_1,\dots,\omega_n)\sim Dir(s,1,\dots,1)$ and $P_0$ from the prior Dirichlet process $Dp(s,\nu^*)$ (via stick-breaking);
  \item Compute $\widehat{HSIC}(X,Y)$, $\widehat{HSIC}(X,X)$, and $\widehat{HSIC}(Y,Y)$ as in Equation~\eqref{eq:whsicxy};
    \item Sample a permutation $\pi$ of the list $[1,2,\dots,n]$ and compute $\tau=\tau+\widehat{HSIC}(X,Y_\pi)/n$;
  \item Compute $$\widehat{HSIC}(X,Y)/\sqrt{\widehat{HSIC}(X,X)\widehat{HSIC}(Y,Y)}$$ and store the result in $V$;
  %\item If $\text{HSIC}_{S}>0$, then  $P_d=P_d+1$. %   else $P_d=P_d+0$.
%   \item return to 2.
 \end{enumerate}
 \item Compute the histogram of the elements in $\tilde{V}=(V-\tau)/(1-\tau)$ (this gives us the plot of the posterior of  $\BdCor$).
 \item Compute the posterior probability that  $\BdCor$ is greater than ROPI as $\mathcal{P}(\BdCor>ROPI)\approx \frac{\#(\tilde{V}>ROPI)}{N_{mc}}$.
\end{enumerate}

\subsection{Synthetic data}
Table \ref{tab:1} shows an instance of D1 including $n= 5$ observations  generated with $\rho=0.9999$.
\begin{table*}
\centering
\begin{tabular}{cccccc}
\toprule
      $X$ &  $Y$ &  $C_X^\rho$ &  $D_X^\rho$ &  $D_Y^\rho$ &                                $\mathbb{C}_X^\rho$ \\
\midrule
-0.99784 &  1.0 &     -0.99827 &           0.0 &           1.0 &  [0.15922, 0.15967, 0.159..] \\
  1.25424 &  0.0 &      1.25572 &           1.0 &           0.0 &  [0.89481, 0.89531, 0.895...] \\
 -0.65891 &  1.0 &     -0.65962 &           0.0 &           1.0 &  [0.25504, 0.25559, 0.255...] \\
  0.08530 &  0.0 &      0.08667 &           1.0 &           0.0 &  [0.53422, 0.53310, 0.533...] \\
  1.27031 &  0.0 &      1.26787 &           1.0 &           0.0 &  [0.89805, 0.89818, 0.898...] \\
\bottomrule
\end{tabular}
 \caption{An instance of D1 for $n=5$ and $\rho=0.9999$.}
\label{tab:1}
 \end{table*}

Figure~\ref{fig:clay} shows the results for the simulations with Gaussian Copula with ROPI=0.05 (top) and Clayton
Copula with ROPI=0.025 as discussed in Section~\ref{ss} of the paper.

\begin{figure}[htp]
\centering
\vspace{-0pt}
 \setlength\tabcolsep{0.1pt}
 \begin{tabular}{r}
\includegraphics[height=1.2in]{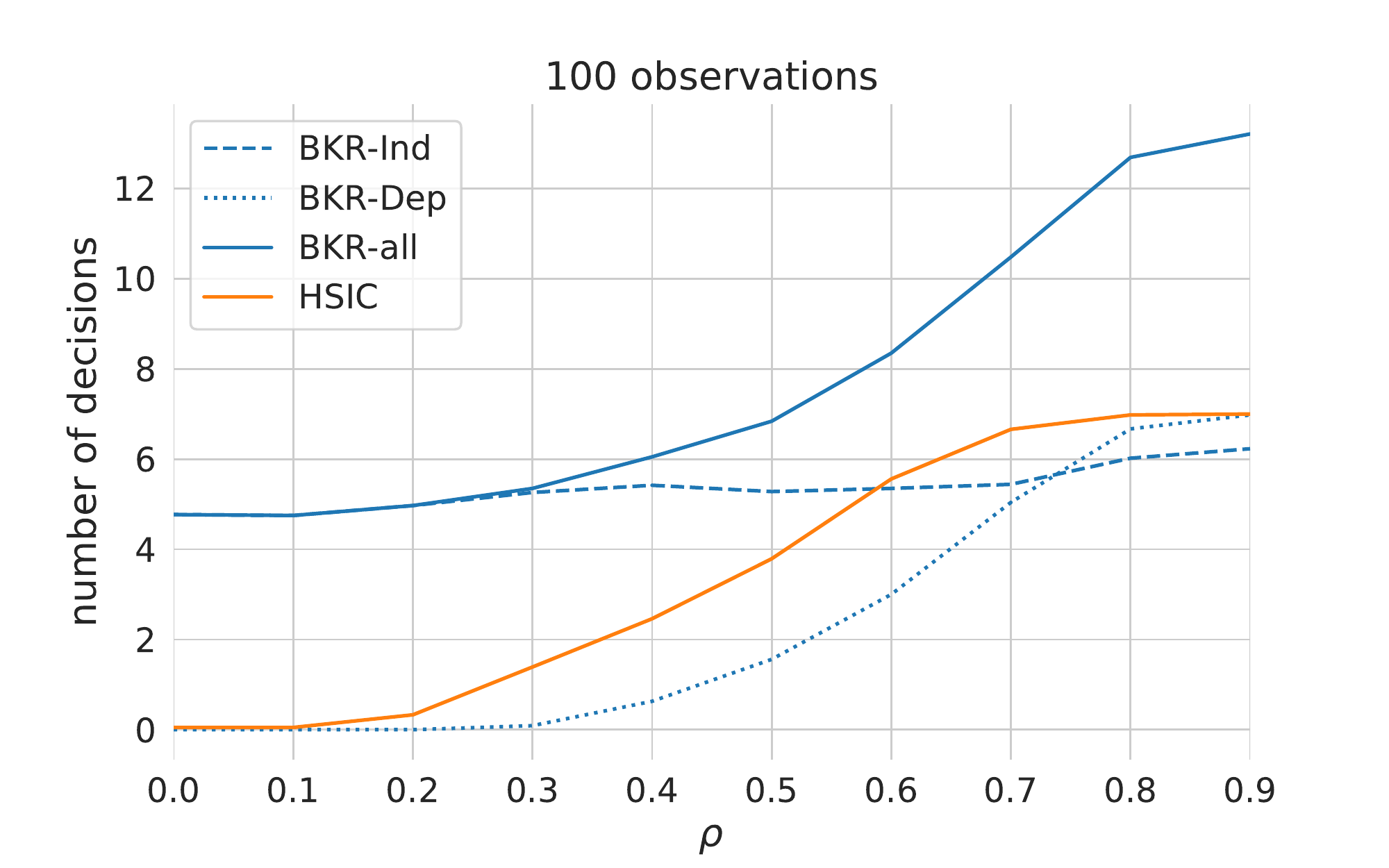}\\ %\hspace{-5.4mm} 
\includegraphics[height=1.2in]{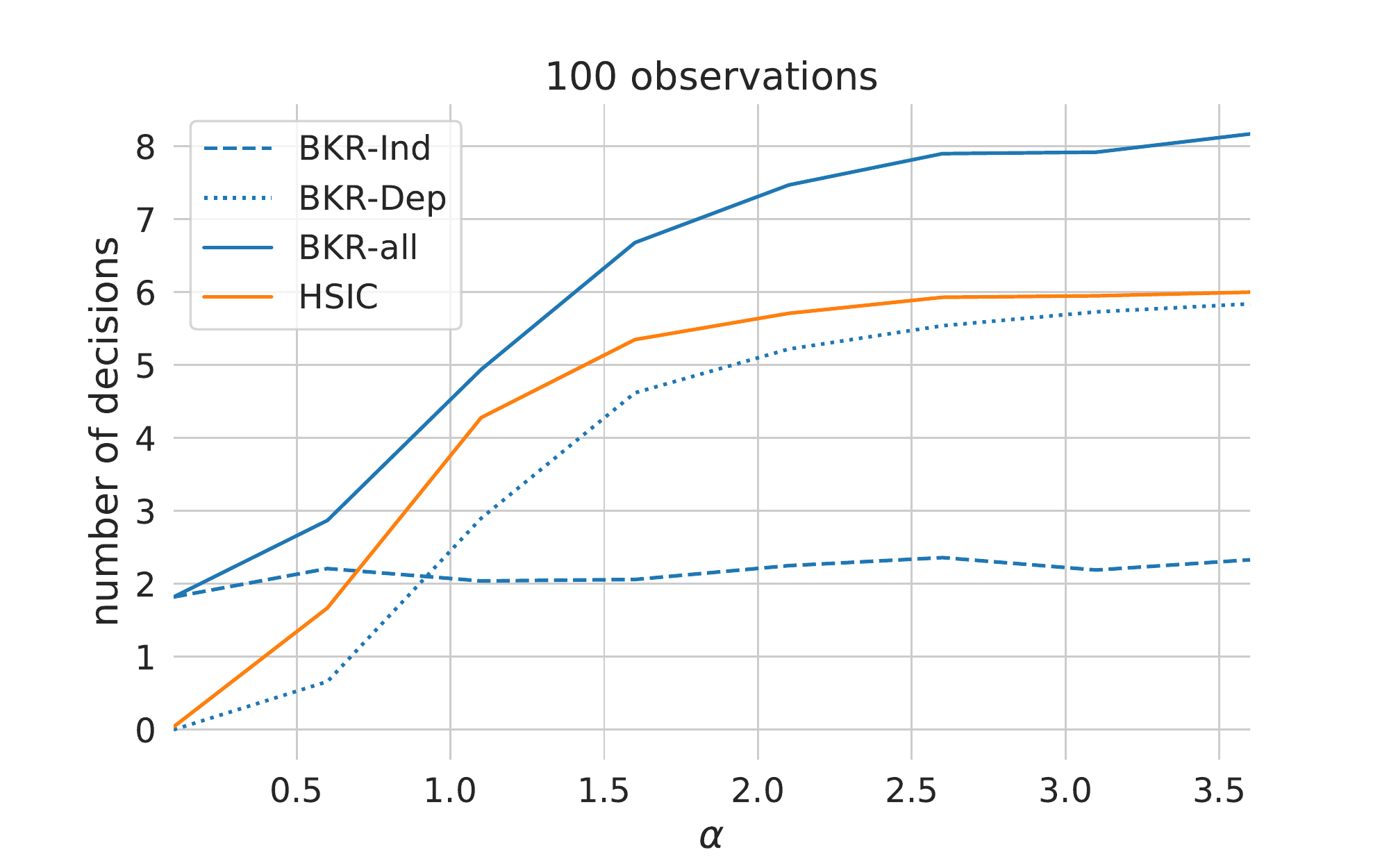} %\hspace{-5.4mm} 
\end{tabular}
\vspace{-5pt}
\caption{Synthetic dataset D1 with ROPI=0.05 and synthetic dataset D2 with ROPI=0.025}
\label{fig:clay}
\vspace{-0pt}
\end{figure}

\subsection{Experimental method for the MIC dataset}
For BKR, we used a square exponential kernel with kernel length-scale set to the median distance
between points in input space.
For MIC, we have approximated the null distribution by randomly permuting one of the variables in the comparison, repeating $500$ times, and then used the approximated null distribution to compute p-values. We have used a similar approach for HSIC.

\subsection{Experimental method for the Gapminder dataset}
In this section we outline the methodology used to
produce the pairwise heatmaps shown
in Figures \ref{Fig:gapmind0} and \ref{Fig:gapmind1}. 
For BKR and HSIC, for each pairwise comparison, all records in which at least one of the variables is missing were dropped.
If the total number of remaining observations was less than
three, the returned decision was ``undecided''
(p-value=0.5 and posterior probability equal to 0.5).

Figure \ref{Fig:gapmindSupp} shows the scatter plots
of 8 pairwise comparisons for the Gapminder datasets.
The title of each plot reports the posterior probability of dependence computed by BKR (left) and Crosscat (right).
Both tests agree for most of the comparisons, even though there are differences in the presence of outliers. For instance, BKR is able to detect the dependence in the last two plots, while the Crosscat probability is close to indicate independence.
\begin{figure*}
\centering
  \begin{tabular}{ccc}
         \includegraphics[width=.30\linewidth]{{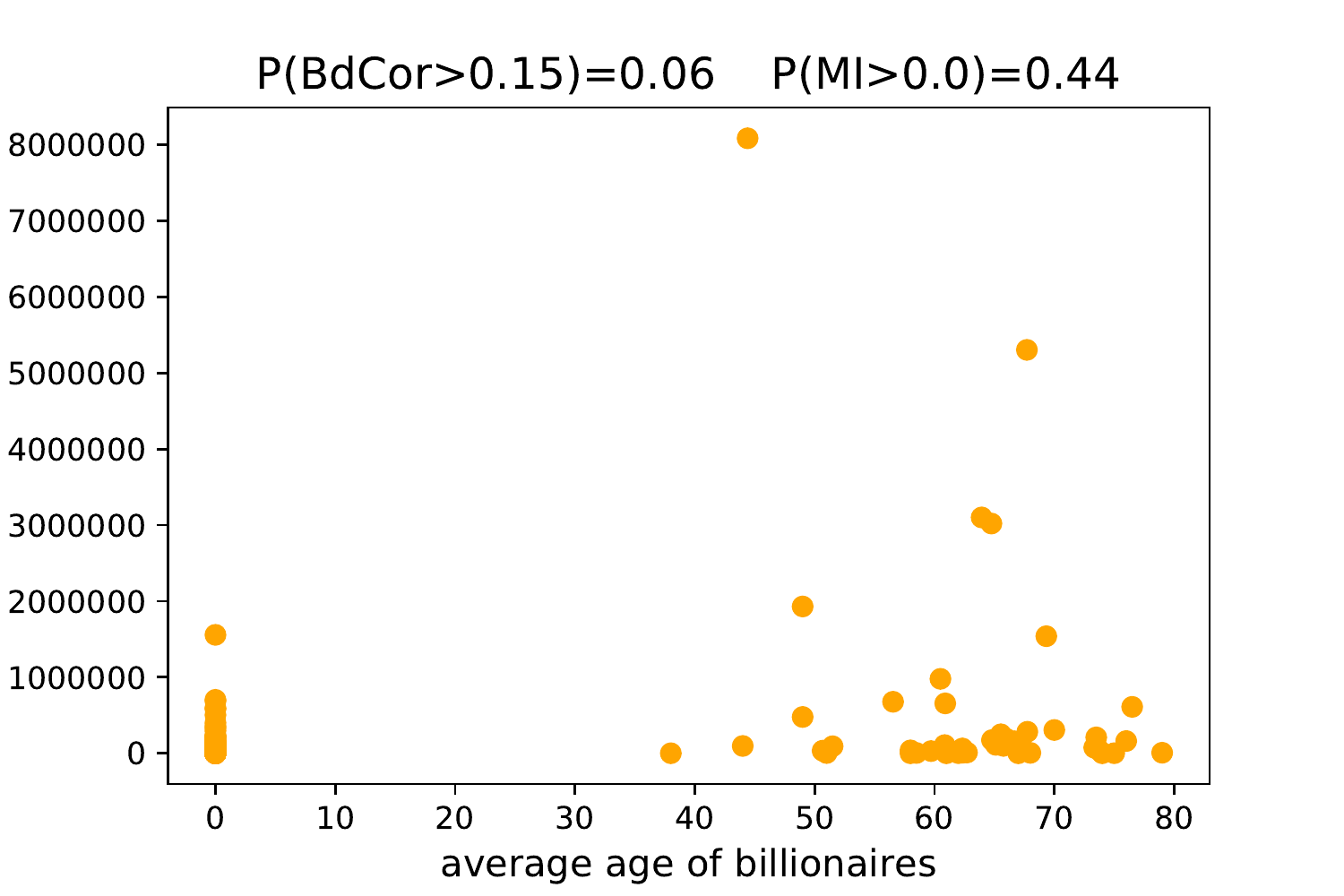}} &
    \includegraphics[width=.30\linewidth]{{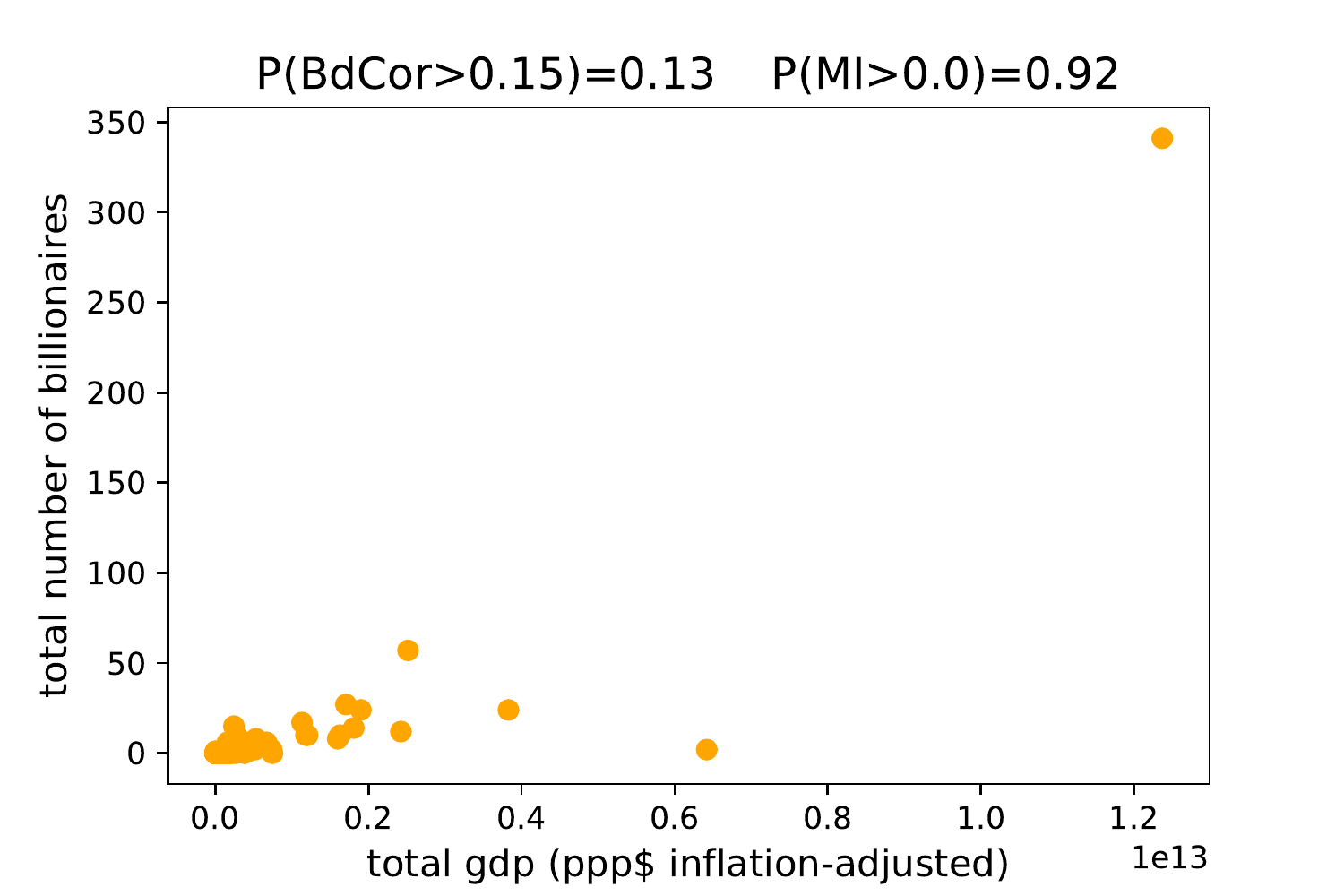}}
    &    \includegraphics[width=.30\linewidth]{{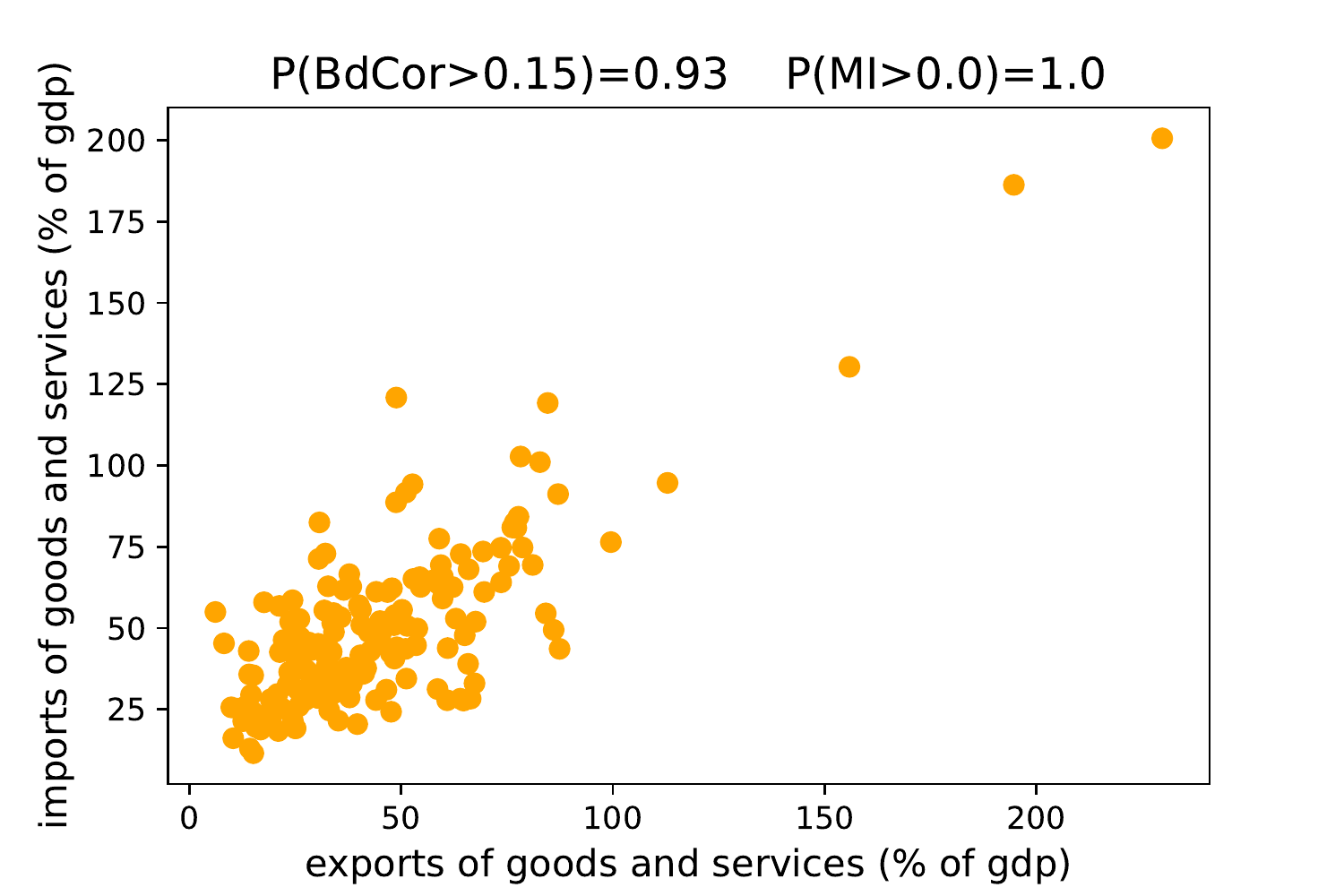}}\\
            \includegraphics[width=.30\linewidth]{{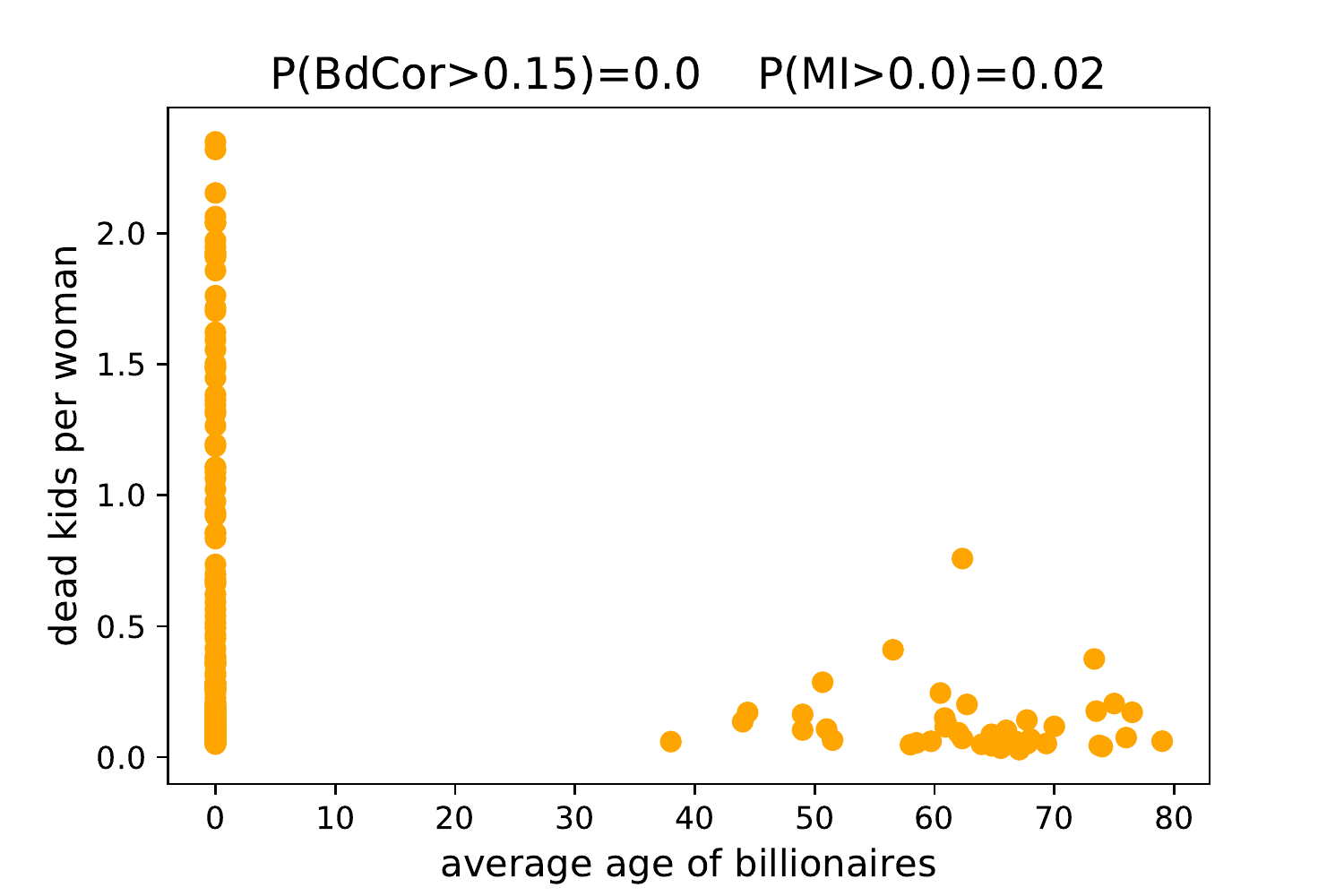}} &
    \includegraphics[width=.30\linewidth]{{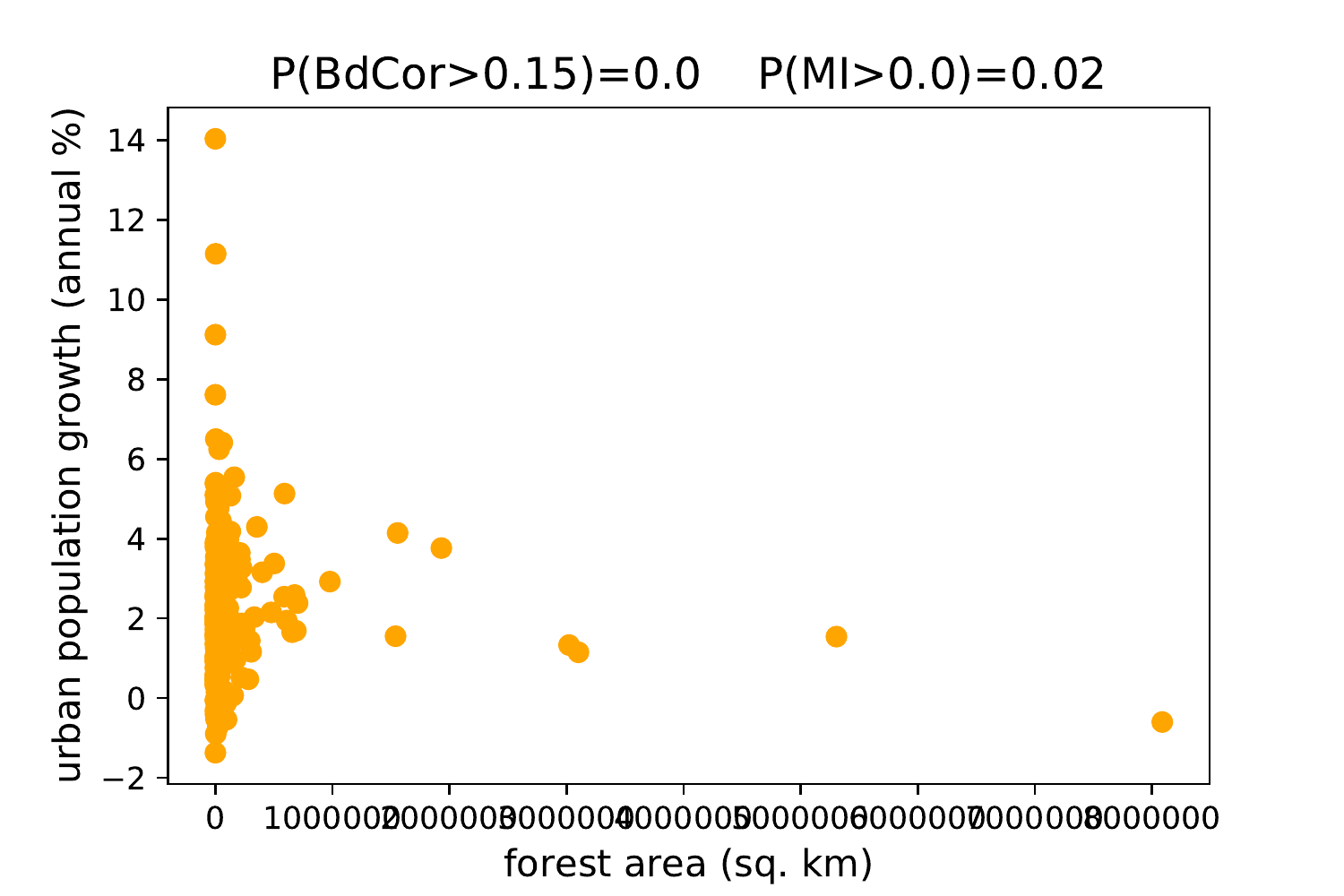}}
    &    \includegraphics[width=.30\linewidth]{{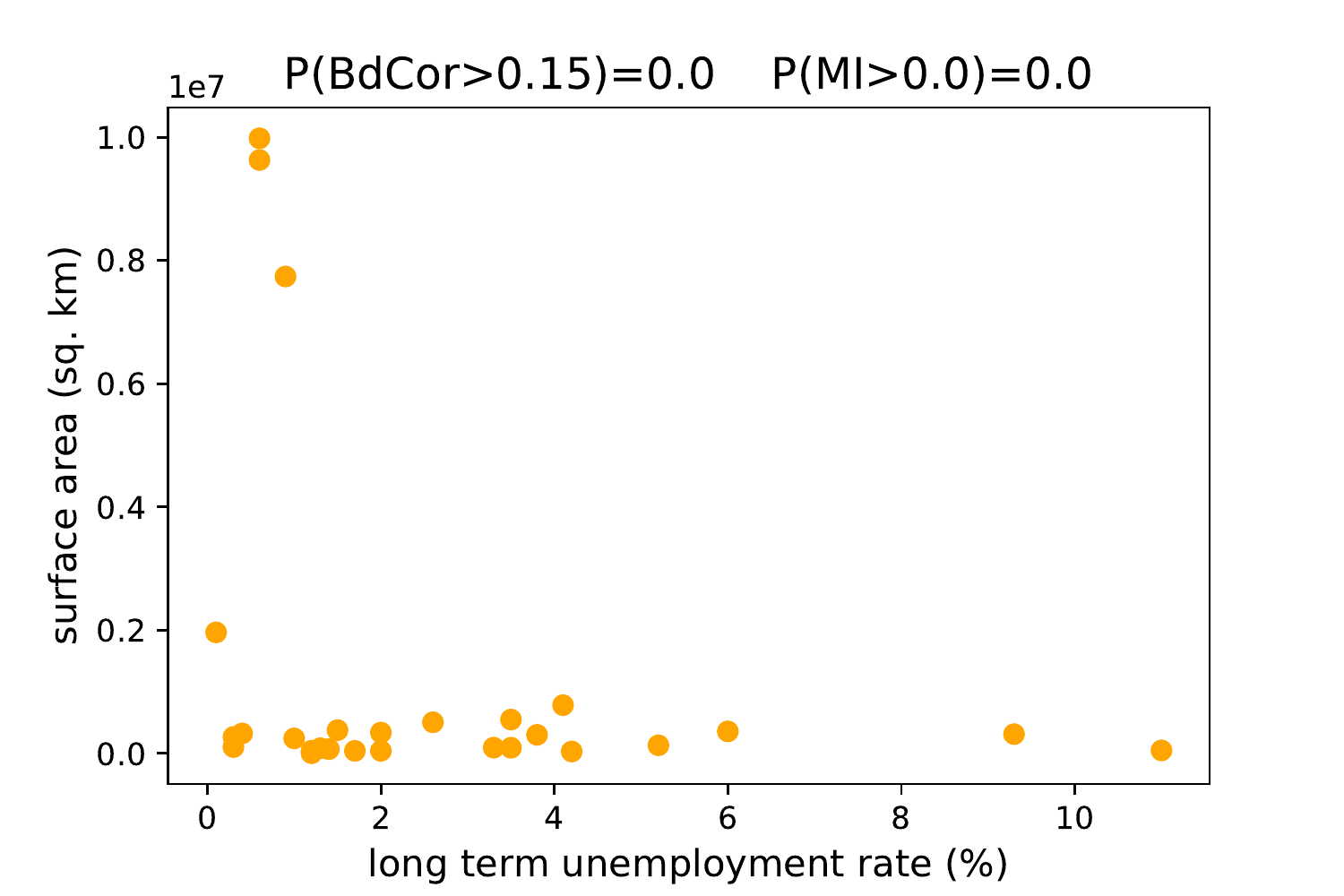}}\\
                \includegraphics[width=.30\linewidth]{{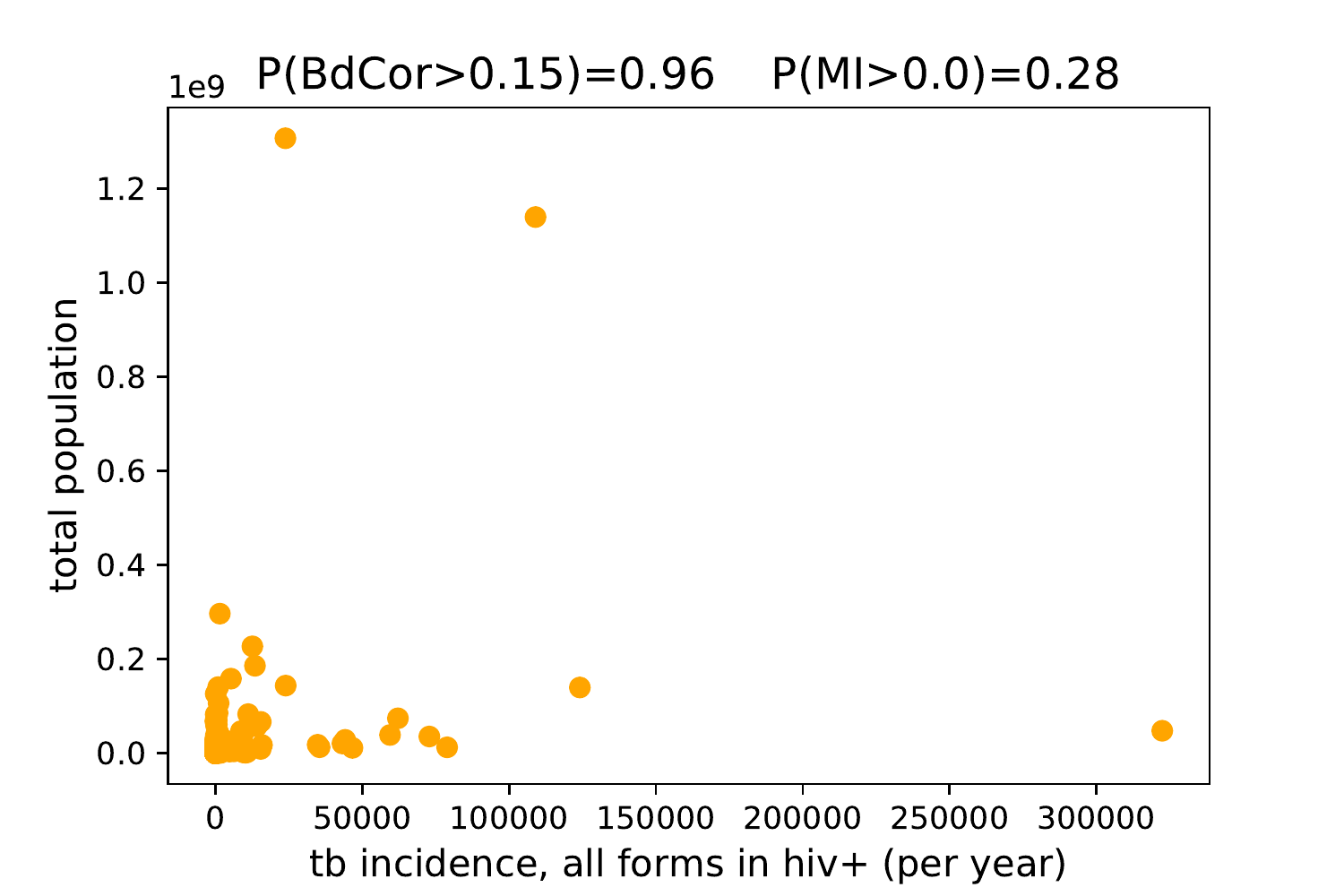}} &
    \includegraphics[width=.30\linewidth]{{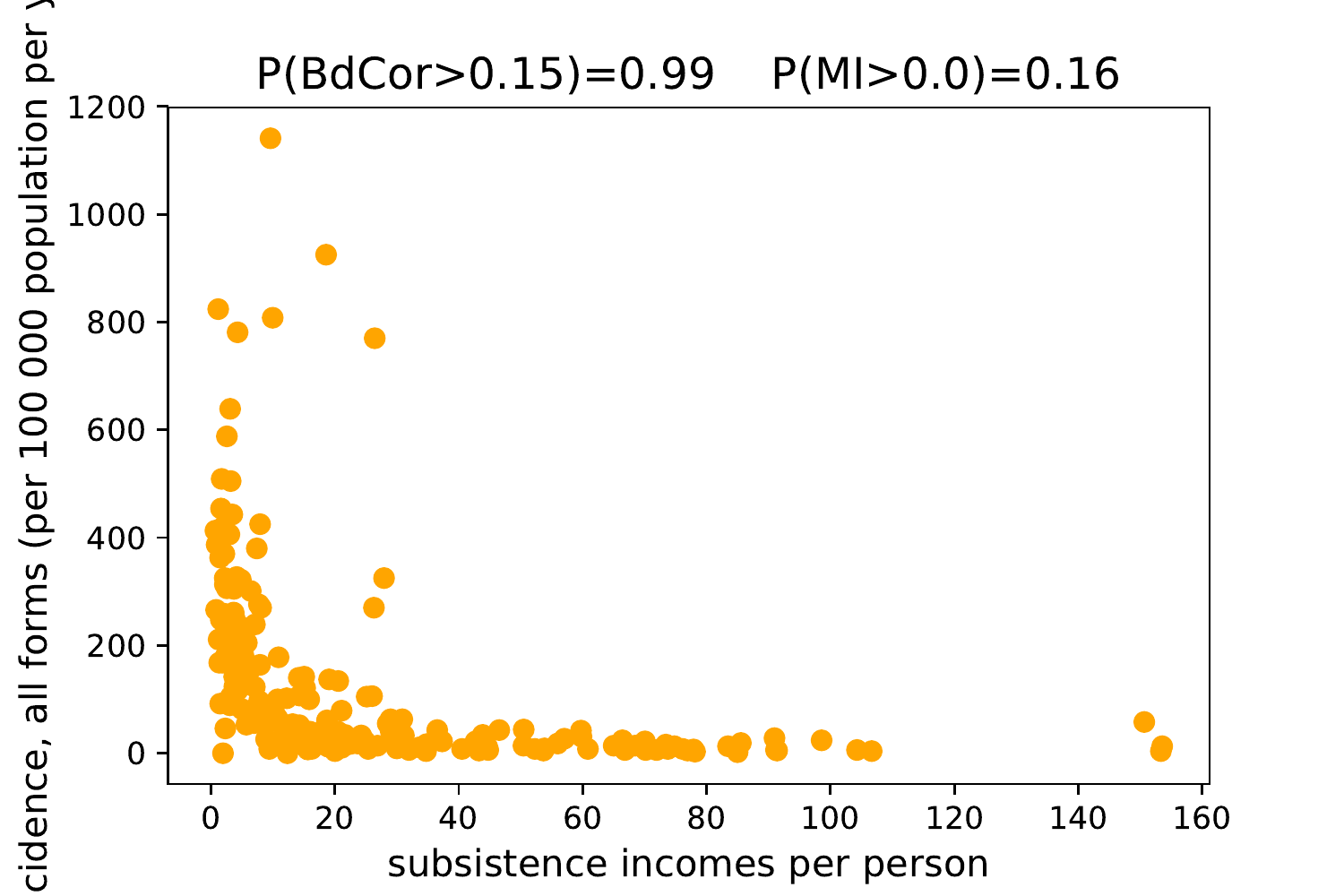}}\\
  \end{tabular}                                                                                                                                                                                                                    
\caption{Gapminder dataset}
\label{Fig:gapmindSupp}
\end{figure*}

\begin{table*}
\centering
{\tiny
\begin{tabular}{lllllrr}
\toprule
{} &                                               name & \#instances & \#features & \#classes &  $\mathcal{E}(BdCor)$ &  log-loss \\
\midrule
0  &        GAM\_Epi\_2-Way\_20atts\_0.1H\_EDM-1\_1 &        1600 &         20 &         2 &             -0.008035 &  0.702575 \\
1  &        GAM\_Epi\_2-Way\_20atts\_0.4H\_EDM-1\_1 &        1600 &         20 &         2 &             -0.007973 &  0.705753 \\
2  &        GAM\_Epi\_3-Way\_20atts\_0.2H\_EDM-1\_1 &        1600 &         20 &         2 &             -0.006651 &  0.703142 \\
3  &  GAM\_Het\_20atts\_1600\_Het\_0.4\_0.2\_... &        1600 &         20 &         2 &             -0.007961 &  0.705855 \\
4  &  GAM\_Het\_20atts\_1600\_Het\_0.4\_0.2\_... &        1600 &         20 &         2 &             -0.007268 &  0.705941 \\
5  &                             Hill\_Valley\_with\_noise &        1212 &        100 &         2 &             -0.000064 &  0.609377 \\
6  &                          Hill\_Valley\_without\_noise &        1212 &        100 &         2 &             -0.002888 &  0.616655 \\
7  &                             analcatdata\_authorship &         841 &         70 &         4 &              0.671466 &  0.046860 \\
8  &                                         australian &         690 &         14 &         2 &              0.292128 &  0.329723 \\
9  &                                           backache &         180 &         32 &         2 &              0.049330 &  0.482633 \\
10 &                                      balance-scale &         625 &          4 &         3 &              0.309459 &  0.365887 \\
11 &                                             banana &        5300 &          2 &         2 &              0.042040 &  0.685723 \\
12 &                                             biomed &         209 &          8 &         2 &              0.271644 &  0.269282 \\
13 &                                             breast &         699 &         10 &         2 &              0.759781 &  0.117858 \\
14 &                                      breast-cancer &         286 &          9 &         2 &              0.071111 &  0.566539 \\
15 &                            breast-cancer-wisconsin &         569 &         30 &         2 &              0.554668 &  0.076874 \\
16 &                                           breast-w &         699 &          9 &         2 &              0.769333 &  0.116817 \\
17 &                                           buggyCrx &         690 &         15 &         2 &              0.236791 &  0.333144 \\
18 &                                               bupa &         345 &          6 &         2 &              0.021190 &  0.620984 \\
19 &                                                car &        1728 &          6 &         4 &              0.070974 &  0.698769 \\
20 &                                               cars &         392 &          8 &         3 &              0.281463 &  0.507757 \\
21 &                                              cars1 &         392 &          7 &         3 &              0.261504 &  0.561553 \\
22 &                                              churn &        5000 &         20 &         2 &              0.046255 &  0.320298 \\
23 &                                              cleve &         303 &         13 &         2 &              0.236671 &  0.445662 \\
24 &                                              colic &         368 &         22 &         2 &              0.133850 &  0.495178 \\
25 &                                             corral &         160 &          6 &         2 &              0.296792 &  0.262120 \\
26 &                                           credit-a &         690 &         15 &         2 &              0.237597 &  0.340254 \\
27 &                                           credit-g &        1000 &         20 &         2 &              0.028159 &  0.535588 \\
28 &                                                crx &         690 &         15 &         2 &              0.235325 &  0.341087 \\
29 &                                        dermatology &         366 &         34 &         6 &              0.749185 &  0.148374 \\
30 &                                           diabetes &         768 &          8 &         2 &              0.138120 &  0.480159 \\
31 &                                                dna &        3186 &        180 &         3 &              0.209221 &  0.185771 \\
32 &                                              ecoli &         327 &          7 &         5 &              0.642567 &  0.461982 \\
33 &                                              flare &        1066 &         10 &         2 &              0.037990 &  0.403444 \\
34 &                                             german &        1000 &         20 &         2 &              0.034654 &  0.530094 \\
35 &                                             glass2 &         163 &          9 &         2 &              0.080849 &  0.574155 \\
36 &                                           haberman &         306 &          3 &         2 &              0.037010 &  0.559296 \\
37 &                                            heart-c &         303 &         13 &         2 &              0.272453 &  0.416660 \\
38 &                                            heart-h &         294 &         13 &         2 &              0.246680 &  0.458117 \\
39 &                                      heart-statlog &         270 &         13 &         2 &              0.301055 &  0.396709 \\
40 &                                          hepatitis &         155 &         19 &         2 &              0.158631 &  0.432231 \\
41 &                                        horse-colic &         368 &         22 &         2 &              0.125779 &  0.486803 \\
42 &                                     house-votes-84 &         435 &         16 &         2 &              0.555361 &  0.141234 \\
43 &                                          hungarian &         294 &         13 &         2 &              0.259716 &  0.441491 \\
44 &                                         ionosphere &         351 &         34 &         2 &              0.234897 &  0.345410 \\
45 &                                              irish &         500 &          5 &         2 &              0.203002 &  0.488340 \\
46 &                                     liver-disorder &         345 &          6 &         2 &              0.021365 &  0.616264 \\
47 &                                      mfeat-factors &        2000 &        216 &        10 &              0.557608 &  0.148499 \\
48 &                                      mfeat-fourier &        2000 &         76 &        10 &              0.388237 &  0.560296 \\
49 &                                     mfeat-karhunen &        2000 &         64 &        10 &              0.448040 &  0.296411 \\
50 &                                        mfeat-pixel &        2000 &        240 &        10 &              0.525751 &  0.297728 \\
51 &                                      mfeat-zernike &        2000 &         47 &        10 &              0.409329 &  0.499941 \\
52 &                                              monk1 &         556 &          6 &         2 &              0.004813 &  0.693910 \\
53 &                                              monk2 &         601 &          6 &         2 &             -0.002324 &  0.648324 \\
54 &                                              monk3 &         554 &          6 &         2 &              0.178117 &  0.391921 \\
55 &                                        new-thyroid &         215 &          5 &         3 &              0.521382 &  0.210466 \\
56 &                                          optdigits &        5620 &         64 &        10 &              0.530505 &  0.161148 \\
57 &                                          parity5+5 &        1124 &         10 &         2 &             -0.013396 &  0.703223 \\
58 &                                            phoneme &        5404 &          5 &         2 &              0.158538 &  0.471846 \\
59 &                                               pima &         768 &          8 &         2 &              0.139403 &  0.494266 \\
60 &                                         prnn\_synth &         250 &          2 &         2 &              0.335151 &  0.333982 \\
61 &                                              profb &         672 &          9 &         2 &              0.004817 &  0.608035 \\
62 &                                               ring &        7400 &         20 &         2 &              0.500697 &  0.521947 \\
63 &                                            saheart &         462 &          9 &         2 &              0.093666 &  0.540378 \\
64 &                                           satimage &        6435 &         36 &         6 &              0.580559 &  0.491798 \\
65 &                                       segmentation &        2310 &         19 &         7 &              0.571308 &  0.341357 \\
66 &                                      solar-flare\_2 &        1066 &         12 &         6 &              0.295000 &  0.729483 \\
67 &                                            soybean &         675 &         35 &        18 &              0.439609 &  0.479234 \\
68 &                                              spect &         267 &         22 &         2 &              0.164834 &  0.391502 \\
69 &                                             spectf &         349 &         44 &         2 &              0.161740 &  0.367843 \\
70 &                                             splice &        3188 &         60 &         3 &              0.110142 &  0.453967 \\
71 &                                            texture &        5500 &         40 &        11 &              0.461275 &  0.106539 \\
72 &                                           threeOf9 &         512 &          9 &         2 &              0.127778 &  0.406353 \\
73 &                                        tic-tac-toe &         958 &          9 &         2 &              0.014250 &  0.628244 \\
74 &                                            titanic &        2201 &          3 &         2 &              0.162253 &  0.525138 \\
75 &                                             tokyo1 &         959 &         44 &         2 &              0.452755 &  0.198396 \\
76 &                                            twonorm &        7400 &         20 &         2 &              0.622287 &  0.059936 \\
77 &                                            vehicle &         846 &         18 &         4 &              0.147601 &  0.549114 \\
78 &                                               vote &         435 &         16 &         2 &              0.586648 &  0.109903 \\
79 &                                        waveform-21 &        5000 &         21 &         3 &              0.431121 &  0.311201 \\
80 &                                        waveform-40 &        5000 &         40 &         3 &              0.382859 &  0.306393 \\
81 &                                               wdbc &         569 &         30 &         2 &              0.554726 &  0.075543 \\
82 &                                   wine-recognition &         178 &         13 &         3 &              0.660826 &  0.104399 \\
83 &                                                xd6 &         973 &          9 &         2 &              0.114803 &  0.412084 \\
\bottomrule
\end{tabular}}
\caption{Predicting classifier performance using  84 datasets from the Penn Machine Learning Benchmarks repository. Last column is the log-loss of the logistic regression classifier with $L_2$ penalty and fixed
reguralisation constant ($C=1$).}
\end{table*} 

\end{document}